\documentclass[11pt]{article}
\usepackage[letterpaper]{geometry}

\usepackage[colorlinks,linkcolor=blue,citecolor=orange]{hyperref}

\usepackage[letterpaper]{geometry}
\usepackage{amsmath,amsthm,amsfonts,amssymb,booktabs}
\usepackage{enumerate,color,xcolor}
\usepackage{graphicx}
\usepackage{subfigure}
\usepackage{url}

\usepackage{graphicx}
\usepackage{subfigure}
\usepackage{amsmath,amsthm,amsfonts,amssymb}
\usepackage{enumerate,color,xcolor}
\usepackage{url}

\usepackage{hyperref}

\usepackage{comment}

\newtheorem{theorem}{Theorem}

\newtheorem{corollary}[theorem]{Corollary}

\newtheorem{lemma}[theorem]{Lemma}

\newtheorem{proposition}[theorem]{Proposition}

\newtheorem{remark}[theorem]{Remark}
\newtheorem{assumption}[theorem]{Assumption}

\newcommand{\R}{\mathbb{R}}
\newcommand{\Sml}{\mathcal{S}_{\mu,L}}

\newcommand{\beq}{\begin{eqnarray}}
\newcommand{\beqs}{\begin{eqnarray*}}
\newcommand{\eeq}{\end{eqnarray}}
\newcommand{\eeqs}{\end{eqnarray*}}

\newcommand{\prox}{\mathcal{P}_{\mathcal{C}}}
\newcommand{\diam}{\mathcal{D}_{\mathcal{C}}}
\newcommand{\tx}{\tilde{x}}
\newcommand{\ty}{\tilde{y}}
\newcommand{\txi}{{\tilde{\xi}}}

\newcommand{\tu}{\tilde{u}}

\newcommand{\trabsq}{{{\tilde{\rho}}_{\alpha,\beta}}}

\newcommand{\tK}{{\tilde{K}}_{\alpha,\beta}}

\newcommand{\calF}{\mathcal{F}}

\newcommand{\E}{\mathbb{E}}

\begin{document}

\begin{center}
  \Large \bf Accelerated Linear Convergence of Stochastic Momentum Methods in Wasserstein Distances
\end{center}

\author{}
\begin{center}
{Bugra Can}\,\footnote{Department of Management Science
and Information Systems, Rutgers Business School, Piscataway, NJ-08854, United States of America;
    bugra.can@rutgers.edu},
    {Mert G\"{u}rb\"{u}zbalaban}\,\footnote{Department of Management Science
and Information Systems, Rutgers Business School, Piscataway, NJ-08854, United States of America;
    mg1366@rutgers.edu},
  Lingjiong Zhu\,\footnote{Department of Mathematics, Florida State University, 1017 Academic Way, Tallahassee, FL-32306, United States of America; zhu@math.fsu.edu
  }
\end{center}

\begin{center}
 \today
\end{center}
\begin{abstract}
Momentum methods such as Polyak's heavy ball (HB) method, Nesterov's accelerated gradient (AG) as well as accelerated projected gradient (APG) method have been commonly used in machine learning practice, but their performance is quite sensitive to noise in the gradients. We study these methods under 
a first-order stochastic oracle model where noisy estimates of the gradients are available. 
For strongly convex problems, we show that the distribution of the iterates of AG converges with the accelerated $O(\sqrt{\kappa}\log(1/\varepsilon))$ linear rate to a ball of radius $\varepsilon$ centered at a unique invariant distribution in the 1-Wasserstein metric where $\kappa$ is the condition number as long as the noise variance is smaller than an explicit upper bound we can provide. Our analysis also certifies linear convergence rates as a function of the stepsize, momentum parameter and the noise variance; recovering the accelerated rates in the noiseless case and quantifying the level of noise that can be tolerated to achieve a given performance. 
To the best of our knowledge, these are the first linear convergence results for stochastic momentum methods under the stochastic oracle model. We also develop finer results for the special case of quadratic objectives, extend our results to the APG method and weakly convex functions showing accelerated rates when the noise magnitude is sufficiently small. 
\end{abstract}

\vspace{-0.3in}
\section{Introduction}
\label{sec-intro}
Many key problems in machine learning can be formulated as convex optimization problems. Prominent examples in supervised learning include linear and non-linear regression problems,  support vector machines, logistic regression or more generally risk minimization problems \cite{vapnik2013nature}. Accelerated first-order optimization methods based on momentum averaging and their stochastic and proximal variants have been of significant interest in the machine learning community due to their scalability to large-scale problems and good performance in practice both in convex and non-convex settings, including deep learning (see e.g. \cite{sutskever2013importance,nitanda2014stochastic,hu2009accelerated,xiao2010dual}). 


Accelerated optimization methods for unconstrained problems based on momentum averaging techniques go back to Polyak who proposed the \emph{heavy ball} (HB) method \cite{Polyak64heavyball} and are closely related to Tschebyshev acceleration, conjugate gradient and under-relaxation methods from numerical linear algebra \cite{varga2009matrix,vavasis}. Another popular momentum-based method is the Nesterov's \emph{accelerated gradient} (AG) method \cite{nesterov2004introductory}. For deterministic strongly convex problems, with access to the gradients of the objective, there is a well-established convergence theory for momentum methods. In particular, for minimizing strongly convex smooth objectives with Lipschitz gradients AG method requires $O(\sqrt{\kappa}\log(1/\varepsilon))$ iterations to find an  $\varepsilon$-optimal solution where $\kappa$ is the condition number, this improves significantly over the $O(\kappa\log(1/\varepsilon))$ complexity of the gradient descent (GD) method. HB method also achieves a similar accelerated rate asymptotically in a local neighborhood around the global minimum. Also, for the special case of quadratic objectives, HB method can achieve the accelerated linear rate globally. In the absence of strong convexity, for convex functions, AG has an iteration complexity of $O(1/\sqrt{\varepsilon})$ in function values which accelerates the standard $O(1/\varepsilon)$ convergence rate of GD. In particular, it can be argued that AG method achieves an optimal convergence rate among all the methods that has access to only first-order information \cite{nesterov2004introductory}. For constrained problems, a variant of AG, the \emph{accelerated projected gradient} (APG) method \cite{o2015adaptive} can also achieve similar accelerated rates \cite{nesterov2004introductory,fazlyab2017dynamical}.

On the other hand, in many applications, the true gradient of the objective function $\nabla f(x)$ is not available but we have access to a noisy but unbiased estimated gradient $\hat\nabla f(x)$ of the true gradient instead. The common choice of the noise that arises frequently in (stochastic oracle) models is the centered, statistically independent noise with a finite variance where for every $x\in\mathcal{X}$,
\begin{align*}
&\mbox{\textbf{(H1)}} \qquad  \mathbb{E}\left[ \hat\nabla f(x) \big| x \right] = {\nabla} f(x), 
\\
&\mbox{\textbf{(H2)}} \qquad \mathbb{E} \left[\| \hat\nabla f(x) - \nabla f(x)\|^2 \big| x\right] \leq \sigma^2, 
\end{align*}
(see e.g.  \cite{Bubeck2014,lan2012optimal}). A standard example of this in machine learning is the familiar prediction scenario when $f(x) = \mathbb{E}_{\theta}\ell(x,\theta)$ where $\ell(x,\theta)$ is the (instantaneous) loss of the predictor $x$ on the example $\theta$ with an unknown underlying distribution where the goal is to find a predictor with the best expected loss. In this case, given $x$, the stochastic oracle draws a random sample $\theta$ from the unknown underlying distribution, and outputs $\hat \nabla  f(x) = \nabla_x \ell(x,\theta)$ which is an unbiased estimator of the gradient. In fact, linear regression, support vector machine and logistic regression problems correspond to particular choices of this loss function $\ell$ (see e.g. \cite{vapnik2013nature}). A second example is where an independent identically distributed (i.i.d.) Gaussian noise with a controlled magnitude is added to the gradients of the objective intentionally, for instance in \emph{private risk minimization} to guarantee privacy of the users' data \cite{bassily2014private}, to escape a local minimum \cite{ge2015escaping} or to steer the iterates towards a global minimum for non-convex problems \cite{GGZarxiv1,GGZarxiv2,raginsky2017non}. Such additive gradient noise arises also naturally when gradients are estimated from noisy data \cite{cohen18,birand2013measurements} or the true gradient is estimated from a subset of its components as in (mini-batch) stochastic gradient descent 
(SGD) methods and their variants.

It is well recognized that momentum-based accelerated methods are quite sensitive to gradient noise \cite{Hardt-blog,devolder2014first,flammarion2015averaging,devolder2013intermediate}, and need higher accuracy of the gradients to perform well \cite{daspremont,devolder2014first} compared to standard methods like GD. In fact, with the standard choice of their stepsize and momentum parameter, numerical experiments show that they lose their superiority over a simple method like GD in the noisy setting \cite{Hardt-blog}, yet alone they can diverge \cite{flammarion2015averaging}. On the other hand, numerical studies have also shown that carefully tuned constant stepsize and momentum parameters can lead to good practical performance for both HB and AG under noisy gradients in deep learning \cite{sutskever2013importance}. Overall, 
there has been a growing interest for obtaining convergence guarantees for \emph{stochastic momentum} methods, i.e. momentum methods subject to noise in the gradients. 

Several works provided sublinear convergence rates for stochastic momentum methods. \cite{lan2012optimal,ghadimi2012optimal-1} developed the AC-SA method which is an adaptation of the AG method to the stochastic composite convex and strongly convex optimization problems and obtained an optimal $O(1/\sqrt{k})$ for the convex case. In a follow-up paper, \cite{ghadimi-lan-shrinking} obtained an optimal $O(1/k)$ convergence bound  for the \emph{constrained} strongly convex optimization employing a domain shrinking procedure.  However, these results do not apply to stochastic HB (SHB). \cite{yang2016unified} provided a uniform analysis of SHB and accelerated stochastic gradient (ASG) showing $O(1/\sqrt{k})$ convergence rate for weakly convex stochastic optimization. \cite{gadat2018stochastic} obtained a number of sublinear convergence guarantees for SHB, showing that with decaying stepsize $\alpha_k = O(1/k^\theta)$ for some $\theta \in (0,1]$, SHB method converges with rate $O(1/k^\theta)$. Several other works focused on proper averaging for reducing the variance of the gradient error in the iterates for strongly convex linear regression problems \cite{jain2017accelerating,flammarion2015averaging,dieuleveut2017harder} and obtained a $O(1/k)$ convergence rate that achieves the minimax estimation rate. Recently, \cite{loizou2017momentum} studied the SHB algorithm for optimizing the least squares problems arising in the solution of consistent linear systems where the gradient noise comes from sampling the rows of the associated linear system and therefore the gradient errors have a multiplicative form vanishing at the optimum (see \cite[Sec 2.5]{loizou2017momentum}), in which case SGD enjoys linear rates to the optimum with constant stepsize. 
The authors show that using a constant stepsize the expected SHB iterates converge linearly to a global minimizer with the accelerated rate and provide a first linear (but not an accelerated linear) rate for the expected suboptimality in function values, however the rate provided is not better than the linear rate of SGD and does not reflect the acceleration behavior compared to SGD. We note however that the results of this paper do not apply to our setting as our noise assumptions \textbf{(H1)--(H2)} are more general. In our setting, due to the persistence of the noise, it is not possible for the iterates of stochastic momentum methods converge to a global minimum, but rather converge to a stationary distribution around the global minimum. To our knowledge, a linear convergence result for momentum-based methods has never been established under this setting. For SGD, \cite{bachGD} showed that when $f$ is strongly convex, the distribution of the SGD iterates with constant stepsize converges linearly to a unique stationary distribution $\pi_\alpha$ in the 2-Wasserstein distance requiring $O(\kappa \log(1/\varepsilon))$ iterations to be $\varepsilon$ close to the stationary distribution when $\alpha=1/L$ which is similar to the iteration complexity of (deterministic) gradient descent. A natural question is whether stochastic momentum methods admit a stationary distribution, if so whether the convergence to this distribution can happen faster compared to SGD. As the momentum methods are quite sensitive to gradient noise \cite{Hardt-blog,cohen18} in terms of performance; a precise characterization of how much noise can be tolerated to achieve accelerated convergence rates under stochastic momentum methods remains understudied. 


\textbf{Contributions:} 
We obtain a number of accelerated convergence guarantees for the SHB, ASG and accelerated stochastic projected gradient (ASPG) methods on both (weakly) convex and strongly convex smooth problems. We note that existing convergence bounds obtained for finite-sum problems that approximate stochastic optimization problems \cite{nitanda2014stochastic} do not apply to our setting as our noise is more general, allowing us to deal directly with the stochastic optimization problem itself.

First, for illustrative reasons, we focus on the special case when $f$ is a strongly convex quadratic on $\mathcal{X}=\R^d$ and the gradient noise is additive, statistically independent and i.i.d. with a finite variance $\sigma^2$. We obtain accelerated linear convergence results for the ASG method in the weighted 2-Wasserstein distances. Building on the framework of \cite{hu2017dissipativity} which simplifies the analysis of momentum-based deterministic methods, our analysis shows that all the existing convergence rates and constants can be translated from the deterministic setting to the stochastic setting. 
Building on novel non-asymptotic convergence guarantees in function values we develop for both the deterministic HB and AG methods, we show that the Markov chain corresponding to the stochastic HB
and AG iterates is geometrically ergodic and the distribution of the iterates converges to a unique equilibrium distribution (whose first two moments we can estimate) with the accelerated linear rate $O(\sqrt{\kappa}\log(1/\varepsilon))$ in the $p$-Wasserstein distance for any $p\geq 1$ with explicit constants. The convergence results hold regardless of the noise magnitude $\sigma$, although $\sigma$ scales the standard deviation of the equilibrium distribution linearly. 
We also provide improved non-asymptotic estimates for the suboptimality of the HB and AG methods both for deterministic and stochastic settings.

Second, we consider (non-quadratic) stochastic strongly convex optimization problems on $\R^d$ under the stochastic oracle model  \textbf{(H1)--(H2)}. We derive explicit bounds on the noise variance $\sigma^2$ so that ASG method converges linearly to a unique stationary distribution with the accelerated linear rate $O(\sqrt{\kappa} \log(1/\varepsilon))$ in the 1-Wasserstein distance. Our results provide convergence rates as a function of $\alpha,\beta$ and $\sigma^2$ that recovers the convergence rate of the AG algorithm as the noise level $\sigma^2$ goes to zero. Therefore, for different parameter choices, we can provide bounds on how much noise can be tolerated to maintain linear convergence.

Third, we focus on the accelerated stochastic projected gradient (ASPG) algorithm for constrained stochastic strongly convex optimization on a bounded domain. We obtain fast accelerated convergence rate to a stationary distribution in the $p$-Wasserstein distance for any $p\geq 1$. Finally, we extend our results to the weakly convex setting where we show an accelerated $O(\frac{1}{\sqrt{\varepsilon}} \log(1/\varepsilon))$ convergence rate as long as the noise level is smaller than explicit bounds we provide. To our knowledge, accelerated rates in the presence of non-zero noise was not reported in the literature before.
\section{Preliminaries}
\subsection{Notation} We use the notation $I_d$ and $0_d$ to denote the $d\times d$ identity and zero matrices. The entry at row $i$ and column $j$ of a matrix $A$ is denoted by $A(i,j)$. Kronecker product of two matrices $A$ and $B$ are denoted by $A \otimes B$. A continuously differentiable function~$f:\R^d \to \R$~is called $L$-smooth if its gradient is Lipschitz with constant $L$. A function $f:\R^d \to \R$ is \emph{$\mu$-strongly convex} if the function $x\mapsto f(x)-\frac{\mu}{2}\|x\|^2$ is convex for some $\mu>0$, 
where $\Vert\cdot\Vert$ denotes the Euclidean norm. 
Following the literature, let $\mathcal{S}_{0,L}$ denote the class of functions that are convex and $L$-smooth for some $L>0$. We use $\mathcal{S}_{\mu,L}$ to denote functions that are both $L$-smooth and $\mu$-strongly convex for $0<\mu<L$ (we exclude the trivial case $\mu=L$ in which case the Hessian of $f$ is proportional to the identity matrix where both deterministic gradient descent, HB and AG can converge in one iteration with proper choice of parameters). The ratio $\kappa:=L/\mu$ is known as the \emph{condition number}.
We denote the global minimum of $f$ on $\mathbb{R}^d$ by $f_*$ and the minimizer of $f$ on $\mathbb{R}^d$ by $x_*$, which is unique by strong convexity. 
For any $p\geq 1$, define $\mathcal{P}_{p}(\mathbb{R}^{2d})$
as the space consisting of all the Borel probability measures $\nu$
on $\mathbb{R}^{2d}$ with the finite $p$-th moment
(based on the Euclidean norm).
For any two Borel probability measures $\nu_{1},\nu_{2}\in\mathcal{P}_{p}(\mathbb{R}^{2d})$, 
we define the standard $p$-Wasserstein
metric (see e.g. \cite{villani2008optimal}):
$$\mathcal{W}_{p}(\nu_{1},\nu_{2}):=\left(\inf_{Z_{1}\sim\nu_{1},Z_{2}\sim\nu_{2}}
\mathbb{E}[\Vert Z_{1}-Z_{2}\Vert^{p}]\right)^{1/p}.$$
Let $S\in \R^{2d\times 2d}$ be a symmetric positive definite matrix. 
For any two vectors $z_1, z_2 \in \R^{2d}$, consider the following weighted $L_2$ norm: $$
    \|z_1-z_2\|_{S} := \left((z_1-z_2)^T S (z_1-z_2)\right)^{1/2}.
$$
Define $\mathcal{P}_{2,S}(\mathbb{R}^{2d})$
as the space consisting of all the Borel probability measures $\nu$
on $\mathbb{R}^{2d}$ with the finite second moment
(based on the $\Vert\cdot\Vert_{S}$ norm).
For any two Borel probability measures $\nu_1$ and $\nu_2$
in the space $\mathcal{P}_{2,S}(\mathbb{R}^{2d})$, the weighted 2-Wasserstein distance is defined as 
\begin{equation}\label{def-wass-2-norm}
\mathcal{W}_{2,S}(\nu_1,\nu_2)
:=\left(\inf_{Z_1\sim\nu_1,Z_2\sim\nu_2}
\mathbb{E}\left[\|Z_1-Z_2\|_{S}^2\right]\right)^{1/2},
\end{equation}
where the infimum is taken over all random couples $(Z_1,Z_2)$ taking values in $\R^{2d}\times \R^{2d}$ with marginals $\nu_1$ and $\nu_2$.
Equipped with the 2-Wasserstein distance \eqref{def-wass-2-norm}, 
$\mathcal{P}_{2,S}(\mathbb{R}^{2d})$ forms a complete metric space (see e.g. \cite{villani2008optimal}).

Let $\mathcal{P}_{\alpha,\beta}(z,\cdot)$ be a Markov transition kernel (with parameters $\alpha,\beta$) associated to a time-homogeneous Markov chain $\{\xi_k\}_{k\geq 0}$ on $\R^{2d}$. A Markov transition kernel is the analogue of the transition matrix for finite state spaces. In particular, if $\xi_0$ has probability law $\nu_0$ then we use the notation that $\xi_k$ has probability law $\mathcal{P}_{\alpha,\beta}^k\nu_0$. Given a Borel measurable function $\varphi:\mathbb{R}^{2d}\rightarrow[0,+\infty]$, we also define
\begin{equation*}
(\mathcal{P}_{\alpha,\beta}\varphi)(z)=\int_{\mathbb{
R}^{2d}}\varphi(y)\mathcal{P}_{\alpha,\beta}(z,dy).
\end{equation*}
Therefore, it holds that $\E[\varphi(\xi_{k+1})| \xi_k = z] = (\mathcal{P}_{\alpha,\beta}\varphi)(z)$. We refer the readers to \cite{ccinlar2011probability} for more on the basic theory of Markov chains.



\subsection{AG method}

For $f\in\Sml$, the deterministic AG method consists of the iterations
\begin{align}\label{def-AG-iters}
&x_{k+1}=y_{k}-\alpha\nabla f(y_{k}), ~ y_{k}=(1+\beta)x_{k}-\beta x_{k-1},
\end{align}
starting from the initial points $x_0, x_{-1} \in \R^d$, where $\alpha>0$ is the stepsize and $\beta>0$ is the momentum parameter \cite{nesterov2004introductory}. Since the AG iterate $x_{k+1}$ depends on both $x_k$ and $x_{k-1}$, it is standard to define the state vector
    \beq\label{def-ag-recursion}
    \xi_k := \begin{pmatrix} x_k^T & x_{k-1}^T \end{pmatrix}^T \in \R^{2d},
\eeq
and rewrite the AG iterations in terms of $\xi_k$. To simplify the presentation and the analysis, we build on the representation of optimization algorithms as a dynamical system from \cite{hu2017dissipativity} and rewrite the AG iterations as
\begin{equation*}
\xi_{k+1}=A\xi_{k}+Bw_{k},
\end{equation*}
where $A = \tilde{A} \otimes I_d$ and $B = \tilde{B} \otimes I_d$ with
\begin{align}
\tilde{A}:=\left(
\begin{array}{cc}
(1+\beta) & -\beta 
\\
1 & 0
\end{array}
\right), ~~ \tilde{B}:=\left(
\begin{array}{c}
-\alpha 
\\
0
\end{array}
\right), \label{def-AB-AG} 
\end{align}
and $w_{k}:=\nabla f\left((1+\beta)x_{k}-\beta x_{k-1}\right)$. The standard analysis of deterministic AG is based on the following Lyapunov function that combines the state vector and function values: 
\begin{equation}\label{eqn:lyapunov}
V_P(\xi_k) 
       := (\xi_k-\xi_{\ast})^T P (\xi_k-\xi_{\ast}) + f(x_k) - f_{*}, 
\end{equation} 
where $\xi_{\ast}=(x_{\ast}^T ~ x_{\ast}^T)^T$ and $P\in\R^{2d\times 2d}$ is positive semi-definite matrix to be appropriately chosen. In particular, a linear convergence $f(\xi_{k+1}) - f(\xi_*) \leq V_P(\xi_{k+1})\leq \rho V_P(\xi_k)$ with rate $\rho$ can be guaranteed if $P$ satisfies a certain matrix inequality precised as follows.
\begin{theorem}\label{thm-hu-lessard}\cite{hu2017dissipativity} Let $\rho \in [0,1)$ be given. If there exists a symmetric positive semi-definite $2\times 2$ matrix $\tilde{P}$ (that may depend
on $\rho$) such that 
\begin{equation}\label{ineq-ag-lmi}
\left(
\begin{array}{cc}
\tilde{A}^{T}\tilde{P}\tilde{A}-\rho \tilde{P} & \tilde{A}^{T}\tilde{P}\tilde{B}
\\
\tilde{B}^{T}\tilde{P}\tilde{A} & \tilde{B}^{T}\tilde{P}\tilde{B}
\end{array}
\right)
-\tilde{X}\preceq 0,
\end{equation}
where $\tilde{X}:=\rho \tilde{X}_{1}+(1-\rho)\tilde{X}_{2} \in \R^{3\times 3}$ with
\begin{align*}
&\tilde{X}_{1}:=
\left(\begin{array}{ccc}
\frac{\beta^{2}\mu}{2} & \frac{-\beta^{2}\mu}{2}  & \frac{-\beta}{2} 
\\
\frac{-\beta^{2}\mu}{2}  & \frac{\beta^{2}\mu}{2}  & \frac{\beta}{2}
\\
\frac{-\beta}{2}  & \frac{\beta}{2}  & \frac{\alpha(2-L\alpha)}{2}
\end{array}\right),
\\
&\tilde{X}_{2}
:=
\left(\begin{array}{ccc}
\frac{(1+\beta)^{2}\mu}{2} & \frac{-\beta(1+\beta)\mu}{2}  & \frac{-(1+\beta)}{2}
\\
-\frac{\beta(1+\beta)\mu}{2}  & \frac{\beta^{2}\mu}{2}  & \frac{\beta}{2} 
\\
\frac{-(1+\beta)}{2} & \frac{\beta}{2}  & \frac{\alpha(2-L\alpha)}{2}
\end{array}\right),
\end{align*}
and $\tilde{A},\tilde{B}$ are given by \eqref{def-AB-AG}, then the deterministic AG iterates defined by \eqref{def-AG-iters} for minimizing $f\in\mathcal{S}_{\mu,L}$ satisfies $f(x_k) - f(x_*)\leq V_P(\xi_{k})\leq \rho^{k} V_P(\xi_0)$ where $V_P$ is defined by \eqref{eqn:lyapunov} and $P = \tilde{P}\otimes I_d$. 
\end{theorem}
In particular, Theorem \ref{thm-hu-lessard} can recover existing convergence rate results for deterministic AG. For example, for the particular choice of
\begin{align}\label{def-P-matrix}
&P_{AG} :=\tilde{P}_{AG}\otimes I_{d},~~ \tilde{P}_{AG}
:= \tilde{u} \tilde{u}^T,
\\
&\tilde{u} := \begin{pmatrix}
\sqrt{L/2} & \sqrt{\mu/2}-\sqrt{L/2}
\end{pmatrix}^T,
\nonumber
\end{align}
and $(\alpha,\beta)=(\alpha_{AG},\beta_{AG})$ with
\vspace{-0.1in}
\begin{equation}\label{alpha:beta:AG}
\alpha_{AG}:=\frac{1}{L},
\qquad
\beta_{AG}:=\frac{\sqrt{\kappa}-1}{\sqrt{\kappa}+1},
\end{equation}
in Theorem \ref{thm-hu-lessard}, we obtain the accelerated convergence rate of  
\begin{equation}\label{eq-rate-AG}
\rho_{AG}:=1-\sqrt{\mu/L}=1-1/\sqrt{\kappa}.
\end{equation}
However, as outlined in the introduction, in a variety of applications in machine learning and stochastic optimization, we do not have access to the true gradient $\nabla f(y_k)$ as in the deterministic AG iterations but we have access to a (noisy) stochastic version $\hat{\nabla}f(y_k)=\nabla f(y_{k})+\varepsilon_{k+1}$,
where $\varepsilon_{k+1}$ is the random gradient noise.
AG algorithm with stochastic gradients has the form
\begin{align}\label{eq-noisy-ag-recursion}
&x_{k+1}=y_{k}-\alpha[\nabla f(y_{k})+\varepsilon_{k+1}],
\\
&y_{k}=(1+\beta)x_{k}-\beta x_{k-1},\nonumber
\end{align} 
which is called the \emph{accelerated stochastic gradient (ASG)} method (see e.g. \cite{jain2017accelerating}). We note that due to the existence of noise, the standard Lyapunov analysis from the literature (see e.g. \cite{wilson2016lyapunov,su2014differential}) does not apply directly.  We make the assumption that the random gradient errors are centered, statistically independent from the past iterates and have a finite second moment following the literature \cite{cohen18,Hardt-blog,GradImprovLearning,StrConvex,flammarion2015averaging}. 
The following assumption is a more formal statement of \textbf{(H1)--(H2)}
adapting to the iterations $\xi_{k}$.

\begin{assumption}[Formal statement of \textbf{(H1)--(H2)}]\label{assump:noise}
On some probability space $(\Omega,\mathcal{F},\mathbb{P})$ 
with a filtration $\mathcal{F}_{k}$
the noise $\varepsilon_k$'s are $\mathcal{F}_{k}$-measurable, stationary and $$\mathbb{E}[\varepsilon_{k}|\mathcal{F}_{k-1}]=0 \quad and \quad \mathbb{E}[\|\varepsilon_{k}\|^2|\mathcal{F}_{k-1}]\leq\sigma^2.$$
\end{assumption}

Under Assumption \ref{assump:noise}, the iterations $\xi_{k}$ forms a time-homogeneous Markov chain which we will study further in Sections  \ref{sec:quadratic} and \ref{sec:strong:convex}. 
 
\subsection{HB method}
For $f \in \mathcal{S}_{\mu,L}$, the HB method was proposed by \cite{Polyak64heavyball}. It consists of the iterations
\begin{equation}\label{eq-hb-iters}
x_{k+1}=x_{k}-\alpha\nabla f(x_{k})+\beta(x_{k}-x_{k-1}),
\end{equation}
where $\alpha>0$ is the step size and $\beta$ is the momentum parameter.
The following asymptotic convergence rate result for HB is well known.  

\begin{theorem}[\cite{polyak1987introduction}, see also \cite{Recht}]\label{thm:HB:literature}
Let the objective function $f\in\Sml$ be a strongly convex quadratic function. Consider the deterministic HB iterations $\{x_k\}_{k\geq 0}$ defined by the recursion \eqref{eq-hb-iters} from an initial point $x_0\in\R^d$  with parameters $(\alpha,\beta)=(\alpha_{HB},\beta_{HB})$ where
\begin{equation}\label{eq-alpha-beta-hb}
 \alpha_{HB}
 :=\frac{4}{(\sqrt{\mu}+\sqrt{L})^{2}},
 \quad
 \beta_{HB}
:=\left(\frac{\sqrt{L/\mu}-1}{\sqrt{L/\mu}+1}\right)^{2}.
\end{equation}
Then, 
$
\Vert x_{k}-x_*\Vert
\leq (\rho_{HB} + \delta_k)^k
\cdot\Vert\xi_{0}-\xi_{\ast}\Vert,
$
where $\delta_k$ is a non-negative sequence that goes to zero and
\beq \rho_{HB} := \frac{\sqrt{\kappa}-1}{\sqrt{\kappa}+1} = 1 - \frac{2}{\sqrt{\kappa}+1}.\label{def-rho-hb-opt}
   \eeq
Furthermore, 
 $f(x_{k})-f(x_*)
\leq \frac{L}{2}(\rho_{HB} + \delta_k)^{2k}
\cdot \Vert\xi_{0}-\xi_{\ast}\Vert^{2}.
$
\end{theorem}

This result has an asymptotic nature as the sequence $\delta_k$ is not explicit. There exist non-asymptotic linear convergence results for HB, but to our knowledge, known linear rate guarantees are slower than the accelerated rate $\rho_{HB}$; with a rate similar to the rate of gradient descent \cite{ghadimiHB}. In Section \ref{sec:HB:quadratic}, we will derive a new non-asymptotic version of this theorem that can guarantee suboptimality for finite $k$ with explicit constants and the accelerated rate $\rho_{HB}$. Note that
the asymptotic rate $\rho_{HB}$ of HB in \eqref{def-rho-hb-opt} on quadratic problems is strictly (smaller) faster than the rate $\rho_{AG}$ of AG from \eqref{eq-rate-AG} in general (except in the particular special case of $\kappa=1$, we have $\rho_{AG}=\rho_{HB}=0$). However, for strongly convex functions, HB iterates given by \eqref{eq-hb-iters} is not globally convergent with parameters $\alpha_{HB}$ and $\beta_{HB}$ \cite{lessard2016analysis}, but if the iterates are started in a small enough neighborhood around the global minimum of a strongly convex function, this rate can be achieved asymptotically \cite{polyak1987introduction}. Since known guarantees for deterministic AG is stronger than deterministic HB on non-quadratic strongly convex functions, we will focus on the AG method for non-quadratic objectives in our paper. 

We will analyze the HB method under noisy gradients:
\begin{equation}\label{HB:noisy}
x_{k+1}=x_{k}-\alpha\left(\nabla f(x_{k})+\varepsilon_{k+1}\right)+\beta(x_{k}-x_{k-1}),
\end{equation}
where the noise satisfies Assumption \ref{assump:noise}. This method is called the \emph{stochastic HB} method \cite{gadat2018stochastic,loizou2018accelerated,flaam2004optimization}.  

In the next section, we show that stochastic momentum methods admit an invariant distribution towards which they converge linearly in a sense we make precise. For illustrative purposes, we first analyze the special case when the objective is a quadratic function, and then move on to the more general case when $f$ is smooth and strongly convex. Also, for quadratic functions we can obtain stronger guarantees exploiting the linearity properties of the gradients.

\section{Special case: strongly convex quadratics}\label{sec:quadratic}
First, we assume that the objective $f\in\Sml$ and is a quadratic function of the form
\vspace{-0.05in}
 \begin{equation}\label{def-quad} 
 f(x) = \frac{1}{2}x^T Q x + a^T x + b,
 \end{equation}
 \vspace{-0.01in}
where $x\in\R^d$, $Q\in \R^{d\times d}$ is symmetric positive definite, $a\in \R^d$ is a column vector and $b\in\R$ is a scalar. We also assume $\mu I_{d}\preceq Q\preceq LI_{d}$ so that $f\in \Sml$. 
In this section, we assume the noise $\varepsilon_{k}$ are i.i.d.
which is a special case of Assumption \ref{assump:noise}.
We next show that both accelerated stochastic gradient and stochastic HB admit a unique invariant distribution towards which the iterates  converge linearly in the 2-Wasserstein metric. 

\subsection{Accelerated linear convergence of AG and ASG}\label{sec:AG:quadratic}

Given vectors, $z_1, z_2 \in \R^{2d}$, we consider 
    \beq \| z_1 - z_2\|_{S_{\alpha,\beta}} := \left( (z_1 - z_2)^T S_{\alpha,\beta} (z_1 - z_2) \right)^{1/2}.\label{def-weighted-norm}
    \eeq
where $S_{\alpha,\beta} \in \R^{2d\times 2d}$ is defined as the symmetric matrix
   \beq\label{def-S-alpha-beta} S_{\alpha,\beta} := P_{\alpha,\beta} + \begin{pmatrix} \frac{1}{2}Q & 0_d \\
                    0_d & 0_d
            \end{pmatrix}, \label{def-S-matrix}
            \eeq
where $P_{\alpha,\beta} := \tilde{P}_{\alpha,\beta} \otimes I_d$ and $\tilde{P}_{\alpha,\beta}$ is a non-zero symmetric positive definite $2\times 2$ matrix (that may depend on the parameters $\alpha$ and $\beta$) with the entry $\tilde{P}_{\alpha,\beta}(2,2)\neq 0$. It can be shown that $S_{\alpha,\beta}$ is positive definite on $\R^{2d}$ (see Lemma \ref {lem-S-positive-def} in the supplementary file), even though $\tilde{P}_{\alpha,\beta}$ can be rank deficient. In this case, due to the positive definiteness of $S_{\alpha,\beta}$, \eqref{def-weighted-norm} defines a weighted $L_2$ norm on $\R^{2d}$. Therefore, if we set $S_{\alpha,\beta}$ in \eqref{def-wass-2-norm}, we can consider the 2-Wasserstein distance between two Borel probability measures $\nu_1$ and $\nu_2$ defined on $\R^{2d}$ with finite second moments (based on the $\| \cdot\|_{S_{\alpha,\beta}}$ norm. 

The ASG iterates $\{\xi_k\}_{k\geq 0}$ defined by \eqref{def-ag-recursion} and \eqref{eq-noisy-ag-recursion}
forms a time-homogeneous Markov chain on $\mathbb{R}^{2d}$. Consider the Markov kernel $\mathcal{P}_{\alpha,\beta}$ associated to this chain. Recall that if $\nu$ is the distribution of $\xi_0$, the distribution of $\xi_k$ is denoted by $\mathcal{P}_{\alpha,\beta}^k \nu$. The following theorem shows that 
this Markov Chain admits a unique equilibrium distribution $\pi_{\alpha,\beta}$ and the distribution of the ASG iterates converges to this distribution exponentially fast with (linear) rate $\rho_{\alpha,\beta}$. This rate achieved by ASG is the same as the rate of the deterministic AG method, except that it is achieved in a different notion (with respect to convergence in $\mathcal{W}_{2,S_{\alpha,\beta}}$). The proof is given in the supplementary file and it is based on studying the contractivity properties of the map $\nu \mapsto \mathcal{P}_{\alpha,\beta}^k \nu$ in the Wasserstein space.\footnote{We also provide numerical experiments in the supplementary file to illustrate the results of Theorem \ref{thm:alpha:beta}.}



\begin{theorem}\label{thm:alpha:beta} 
Let $f\in\Sml$ be a quadratic function \eqref{def-quad}. Consider the Markov chain $\{\xi_k\}_{k\geq 0}$ defined by the ASG recursion \eqref{eq-noisy-ag-recursion} with parameters $\alpha$ and $\beta$ and let $\nu_{k,\alpha,\beta}$ denote the distribution of $\xi_{k}$
with $\nu_{0,\alpha,\beta}\in\mathcal{P}_{2,S_{\alpha,\beta}}(\mathbb{R}^{2d})$. 
Let any convergence rate $\rho_{\alpha,\beta} \in [0,1)$ be given. If there exists a matrix $ \tilde{P}_{\alpha,\beta}$ with $\tilde{P}_{\alpha,\beta}(2,2)\neq 0$ satisfying inequality \eqref{ineq-ag-lmi} with $P=P_{\alpha,\beta}$ and $\rho=\rho_{\alpha,\beta}$, then there exists a unique stationary distribution $\pi_{\alpha,\beta}$.
    $$ \mathcal{W}_{2,S_{\alpha,\beta}}\left(\nu_{k,\alpha,\beta},\pi_{\alpha,\beta}\right) \leq \rho_{\alpha,\beta}^{k} \mathcal{W}_{2,S_{\alpha,\beta}}(\nu_{0,\alpha,\beta},\pi_{\alpha,\beta}),$$
where $\mathcal{W}_{2,S_{\alpha,\beta}}$ is the $2$-Wasserstein distance
\eqref{def-wass-2-norm} equipped with the $\Vert\cdot\Vert_{S_{\alpha,\beta}}$ norm. 
In particular, with $(\alpha,\beta)=(\alpha_{AG},\beta_{AG})$ and $P=P_{AG}$ with $P_{AG}$ defined in \eqref{def-P-matrix}, 
we obtain the optimal accelerated linear rate of convergence:
\begin{equation}\label{eqn:quadratic:1}
\mathcal{W}_{2,S_{\alpha,\beta}}^{2}(\nu_{k,\alpha,\beta},\pi_{\alpha,\beta})
\leq\rho_{AG}^{k} \mathcal{W}_{2,S_{\alpha,\beta}}^{2}(\nu_{0,\alpha,\beta},\pi_{\alpha,\beta}),
\end{equation}
with $\rho_{AG} = 1-\frac{1}{\sqrt{\kappa}}$ as in \eqref{eq-rate-AG}.
\end{theorem}

For the AG method, the choice of $(\alpha,\beta)=(\alpha_{AG},\beta_{AG})$
is popular in practice, however a faster rate can be achieved asymptotically if
\begin{equation}\label{alpha:beta:AG:star}
\alpha_{AG}^{*}:=\frac{4}{3L+\mu},
\qquad
\beta_{AG}^{*}:=\frac{\sqrt{3\kappa+1}-2}{\sqrt{3\kappa+1}+2},
\end{equation}
so that the asymptotic linear convergence rate in distance to the optimality becomes
$\rho_{AG}^{*}:=1-\frac{2}{\sqrt{3\kappa+1}}$,
which translates into the rate $(\rho_{AG}^*)^2$ in function values that is (smaller) faster than $\rho_{AG}$ \cite{lessard2016analysis}; improving the iteration complexity by a factor of $4/\sqrt{3}\approx 2.3$ when $\kappa$ is large. However, these results are asymptotic. Below we provide a first non-asymptotic bound with the faster rate $\rho_{AG}^*$.
\begin{theorem}\label{thm:AG-star-deterministic} 
Let $f\in\Sml$ be a quadratic function \eqref{def-quad}.
Consider the deterministic AG iterations $\{x_k\}_{k\geq 0}$ defined by the recursion \eqref{def-ag-recursion} with initialization $x_0,x_{-1}\in\R^d$ and parameters $(\alpha,\beta)=(\alpha_{AG}^{*},\beta_{AG}^{*})$ as in \eqref{alpha:beta:AG:star}. Then, 
\begin{align}
\| x_{k}-x_* \| 
\leq& C_{k}^{*}(\rho_{AG}^{*})^{k}
\cdot\Vert\xi_{0}-\xi_{\ast}\Vert, \\
f(x_{k})-f(x_{\ast})\leq&\frac{L}{2}(C_{k}^{*})^{2}(\rho_{AG}^{*})^{2k}\cdot\Vert\xi_{0}-\xi_{\ast}\Vert^{2},
\nonumber
\end{align}
where $\rho_{AG}^{*}=1-\frac{2}{\sqrt{3\kappa+1}}$ and
\begin{equation}\label{eqn:Ck:star}
C_{k}^{*}:=\max\left\{\bar{C}^{*},\sqrt{k^{2}((\rho_{AG}^{*})^{2}+1)^{2}+2(\rho_{AG}^{*})^{2}}\right\},
\end{equation}
with $\bar{C}^{\ast}:=\frac{\sqrt{3\kappa+1}+2}{2}((\rho_{AG}^{*})^{2}+1)\tilde{C}^{\ast}$ and 
\begin{equation*}
\tilde{C}^{*}:=\max_{i:\mu<\lambda_{i}<L, \lambda_{i}\neq\frac{3L+\mu}{4}}
\frac{\sqrt{\mu(3L+\mu)}}{\sqrt{(\lambda_{i}-\mu)|3L+\mu-4\lambda_{i}|}},
\end{equation*}
where $\{\lambda_{i}\}_{i=1}^{d}$ are the eigenvalues
of the Hessian $Q$.
\end{theorem}
\begin{remark} The constants $C_k^*$ grows linearly with $k$ in Theorem \ref{thm:AG-star-deterministic} and this dependency is tight in the sense that there are examples achieving it (see the proof in the supplementary file). Our bounds improves the existing results that provide a slower rate $\rho_{AG}$ with bounded constants in front of the linear rate \cite{nesterov2004introductory, Bubeck2014}, if $k$ is large enough (larger than a constant that can be made explicit). 
\end{remark}

Building on this non-asymptotic convergence result for the deterministic AG method, 
we obtain similar non-asymptotic convergence guarantees for the ASG method in $p$-Wasserstein distances 
towards convergence to a stationary distribution.

\begin{theorem}\label{thm:AG-star}
Let $f\in\Sml$ be a quadratic function \eqref{def-quad}.
Consider the ASG iterations $\{x_k\}_{k\geq 0}$ defined by the recursion \eqref{eq-noisy-ag-recursion}.
Let $\nu_{k,\alpha,\beta}$ be the distribution of the $k$-th iterate $\xi_k$ for $k\geq 0$, 
where $\xi_{k}^T:=(x_{k}^T,x_{k-1}^T)$
and parameters $(\alpha,\beta)=(\alpha_{AG}^{*},\beta_{AG}^{*})$
as in \eqref{alpha:beta:AG:star}.
Also assume that $\nu_{0,\alpha_{AG}^{*},\beta_{AG}^{*}}\in\mathcal{P}_{p}(\mathbb{R}^{2d})$
and the noise $\varepsilon_{k}$ has finite $p$-th moment.
Then, there exists a unique stationary distribution $\pi_{\alpha,\beta}$ and for any $p\geq 1$,
\begin{align}
&\mathcal{W}_{p}\left(\nu_{k,\alpha,\beta},
\pi_{\alpha,\beta}\right)
\leq
C_{k}^{*}(\rho_{AG}^{*})^{k}
\cdot
\mathcal{W}_{p}\left(\nu_{0,\alpha,\beta},
\pi_{\alpha,\beta}\right),
\end{align}
where $\rho_{AG}^{*}=1-\frac{2}{\sqrt{3\kappa+1}}$, 
$C_{k}^{*}$ is defined in \eqref{eqn:Ck:star} and
$\mathcal{W}_{p}$ is the standard
the $p$-Wasserstein distance.
\end{theorem}


We can also control the expected suboptimality $\mathbb{E}[f(x_{k})]-f(x_{\ast})$ after $k$ iterations.

\begin{theorem}\label{thm:f:AG-star}
With the same assumptions as in Theorem \ref{thm:AG-star},
\begin{equation}\label{ineq-AG-star}
\mathbb{E}[f(x_{k})]-f(x_{\ast})
\leq
\frac{L}{2}\text{Tr}(X_{AG}^{*})
+V_{AG}^{*}(\xi_{0})(C_{k}^{*})^{2}(\rho_{AG}^{*})^{2k},
\end{equation}
where $\rho_{AG}^{*}=1-\frac{2}{\sqrt{3\kappa+1}}$,
$C_{k}^{*}$ is defined in \eqref{eqn:Ck:star}, $X_{AG}^{*}$ is the covariance matrix
of $\xi_{\infty}-\xi_{\ast}$ and $V_{AG}^{*}(\xi_{0})$ is a constant depending
on any initial state $\xi_{0}$ and both $X$ and $V_{AG}^{*}(\xi_{0})$
will be spelled out in explicit form in the supplementary file.
\end{theorem}


\subsection{Accelerated linear convergence of HB and SHB}\label{sec:HB:quadratic}
We first give a non-asymptotic convergence result for the deterministic HB method with explicit constants, which also implies a bound on the suboptimality $f(x_{k})-f(x_{\ast})$. This refines the asymptotic results in the literature  
(Theorem \ref{thm:HB:literature}).  

\begin{theorem}\label{thm:HB-deterministic} 
Let $f\in\Sml$ be a quadratic function \eqref{def-quad}.
Consider the deterministic HB iterations $\{x_k\}_{k\geq 0}$ defined by the recursion \eqref{eq-hb-iters} with initialization $x_0,x_{-1}\in\R^d$ and parameters $(\alpha,\beta)=(\alpha_{HB},\beta_{HB})$ as in \eqref{eq-alpha-beta-hb}. Then, 
\begin{align}
\| x_{k}-x_* \| 
\leq& C_{k}\rho_{HB}^{k}
\cdot\Vert\xi_{0}-\xi_{\ast}\Vert, \\
f(x_{k})-f(x_{\ast})\leq&\frac{L}{2}C_{k}^{2}\rho_{HB}^{2k}\cdot\Vert\xi_{0}-\xi_{\ast}\Vert^{2},
\nonumber
\end{align}
where $\rho_{HB}$ is defined by \eqref{def-rho-hb-opt} and
\begin{align}\label{eqn:Ck}
C_{k}&:=\max\Bigg\{\bar{C},
\sqrt{4k^{2}\left(\frac{L+\mu}{L-\mu}\right)^{2}+2}\Bigg\},
\end{align}
with 
$\bar{C} :=\max_{i:\mu<\lambda_{i}<L}\frac{\mu+L}{2\sqrt{(\lambda_{i}-\mu)(L-\lambda_{i})}}$,
where $\{\lambda_{i}\}_{i=1}^{d}$ are the eigenvalues
of the Hessian matrix of $f$.
\end{theorem}


\begin{remark} It is clear from the definition of $C_{k}$ in Theorem \ref{thm:HB-deterministic} that the leading coefficient $C_{k}$ grows
at most linearly in the number of iterates $k$ and this dependency cannot be removed in the sense that there are some examples achieving our upper bounds in terms of $k$ dependency (see the supplementary file).
\end{remark}

Building on this non-asymptotic convergence result for the deterministic HB method, we obtain similar non-asymptotic convergence guarantees for the SHB method in Wasserstein distances towards convergence to a stationary distribution.


\begin{theorem}\label{thm:HB}
Let $f\in\Sml$ be a quadratic function \eqref{def-quad}.
Consider the HB iterations $\{x_k\}_{k\geq 0}$ defined by the recursion \eqref{HB:noisy}.
Let $\nu_{k,\alpha,\beta}$ be the distribution of the $k$-th iterate $\xi_k$ for $k\geq 0$, where $\xi_{k}^T:=(x_{k}^T,x_{k-1}^T)$
and parameters $(\alpha,\beta)=(\alpha_{HB},\beta_{HB})$ where $(\alpha_{HB},\beta_{HB})$ is defined as in \eqref{eq-alpha-beta-hb}.
Also assume that $\nu_{0,\alpha_{HB},\beta_{HB}}\in\mathcal{P}_{p}(\mathbb{R}^{2d})$
and the noise $\varepsilon_{k}$ has finite $p$-th moment.
Then, there exists a unique stationary distribution $\pi_{\alpha,\beta}$ and for any $p\geq 1$,
\begin{align}
&\mathcal{W}_{p}\left(\nu_{k,\alpha,\beta},
\pi_{\alpha,\beta}\right)
\leq
C_{k}\rho_{HB}^{k}
\cdot
\mathcal{W}_{p}\left(\nu_{0,\alpha,\beta},
\pi_{\alpha,\beta}\right),\label{bound:HB}
\end{align}
where $\rho_{HB}=1-\frac{2}{\sqrt{k}+1}$ as defined in \eqref{def-rho-hb-opt}, $C_{k}$ is defined in \eqref{eqn:Ck} and
$\mathcal{W}_{p}$ is the standard
the $p$-Wasserstein distance.
\end{theorem}



Similarly, for SHB we can show that the suboptimality $\mathbb{E}[f(x_{k})]-f(x_{\ast})$ decays linearly in $k$ with the fast rate $\rho_{HB}$ to a constant determined by the variance of the equilibrium distribution.

\begin{theorem}\label{thm:f:HB}
With the same assumptions as in Theorem \ref{thm:HB},
\begin{equation}\label{ineq-shb-subopt}
\mathbb{E}[f(x_{k})]-f(x_{\ast})
\leq
\frac{L}{2}\text{Tr}(X_{HB})
+V_{HB}(\xi_{0})\cdot C_{k}^{2}\cdot\rho_{HB}^{2k},
\end{equation}
where $\rho_{HB}=1-\frac{2}{\sqrt{\kappa}+1}$ as in \eqref{def-rho-hb-opt},
$C_{k}$ is defined in \eqref{eqn:Ck}, $X_{HB}$ is the covariance matrix
of $\xi_{\infty}-\xi_{\ast}$, $V_{HB}(\xi_{0})$ is a constant depending
on any initial state $\xi_{0}$ and both $X$ and $V_{HB}(\xi_{0})$
will be spelled out in explicit form in the supplementary file.
\end{theorem}


\section{Strongly convex smooth optimization}\label{sec:strong:convex}

In this section, we study the more general case
when the objective function $f$ is strongly convex,
but not necessarily a quadratic. The proof technique we use for Wasserstein distances can be adapted to obtain a linear rate for a strongly convex objective but this approach does not yield the accelerated rates $\rho_{AG}$ with a $\sqrt{\kappa}$ dependency to the condition number even if the noise magnitude is small. However, we can show accelerated rates in the following alternative metric which implies convergence in the 1-Wasserstein metric.
For any two probability measures $\mu_{1},\mu_{2}$ on $\mathbb{R}^{2d}$,
and any positive constant $\psi$, 
we define the weighted total variation distance (introduced by \cite{hairer}) as
\begin{equation*}
d_{\psi}(\mu_{1},\mu_{2}):=\int_{\mathbb{R}^{2d}}(1+\psi V_{P}(\xi))|\mu_{1}-\mu_{2}|(d\xi).
\end{equation*}
where $V_{P}$ is the Lyapunov function defined in \eqref{eqn:lyapunov}. 
Moreover, since $\psi$ and $V_{P}$ are non-negative, $d_{\psi}(\mu_{1},\mu_{2})\geq 2\Vert\mu_{1}-\mu_{2}\Vert_{TV}$, where $\Vert\cdot\Vert_{TV}$ is the 
standard total variation norm.
Moreover, when $\tilde{P}(2,2)\neq 0$, 
we will show in the supplementary file (Lemma \ref{lem:prop:metric}
and Proposition \ref{prop:metric}) that
\begin{equation*}
\mathcal{W}_{1}(\mu_{1},\mu_{2})
\leq c_{0}^{-1}d_{\psi}(\mu_{1},\mu_{2}),
\end{equation*}
for some explicit constant $c_0$ (to be given in the supplementary file), 
where $\mathcal{W}_{1}$ is the standard $1$-Wasserstein distance.

We will consider the accelerated stochastic gradient (ASG) method for unconstrained optimization problems.
We will also assume in this section that the random gradient error $\varepsilon_{k}$
admits a continuous density so that
conditional on $\xi_{k}=(x_{k}^{T},x_{k-1}^{T})^{T}$, $x_{k+1}$
also admits a continuous density, i.e.
$\mathbb{P}(x_{k+1}\in dx|\xi_{k}=\xi)=p(\xi,x)dx$,
where $p(\xi,x)>0$ is continuous in both $\xi$ and $x$.

\subsection{Accelerated linear convergence of ASG}\label{sec:strong:convex:AG}

For the ASG method with any given $\alpha,\beta$
so that $\rho_{\alpha,\beta}$, $P_{\alpha,\beta}$ satisfy
the LMI inequality \eqref{ineq-ag-lmi}.
Let $\nu_{k,\alpha,\beta}$ be the distribution of the $k$-th iterate $\xi_k$ for $k\geq 0$, 
where $\xi_{k}^T:=(x_{k}^T,x_{k-1}^T)$ and the iterates $x_{k}$ are 
given in \eqref{eq-noisy-ag-recursion} so that
$\mathbb{E}[V_{P_{\alpha,\beta}}(\xi_{0})]$ is finite.
The next result gives a bound of $k$-th iterate to
stationary distribution in the weighted total variation
distance $d_{\psi}$. We also control the expected
suboptimality $\mathbb{E}[f(x_{k})]-f(x_{\ast})$ after $k$ iterations.

\begin{theorem}\label{thm:general:alpha:beta}
Given any $\eta\in(0,1)$ and $M>0$ so that
$
\int_{\Vert x-x_{\ast}\Vert\leq M}p(\xi_{\ast},x)dx\geq\sqrt{\eta},
$
and any $R>0$ so that
\begin{equation*}
\inf_{\xi\in\mathbb{R}^{2d},x\in\mathbb{R}^{d}:V_{P_{\alpha,\beta}}(\xi)\leq R,
\Vert x-x_{\ast}\Vert\leq M}
\frac{p(\xi,x)}{p(\xi_{\ast},x)}
\geq\sqrt{\eta}.
\end{equation*}
Then there is a unique stationary distribution $\pi_{\alpha,\beta}$ so that
\begin{align*}
\mathcal{W}_{1}(\nu_{k,\alpha,\beta},\pi_{\alpha,\beta})
&\leq
c_{0}^{-1}d_{\psi}(\nu_{k,\alpha,\beta},\pi_{\alpha,\beta})
\\
&
\leq
(1-\bar{\eta})^{k}c_{0}^{-1}
d_{\psi}(\nu_{0,\alpha,\beta},\pi_{\alpha,\beta}),
\end{align*}
where $\mathcal{W}_{1}$ is the standard $1$-Wasserstein distance and
$\psi:=\frac{\eta}{2K_{\alpha,\beta}}$ and 
\begin{align*}
&K_{\alpha,\beta}:=\left(\frac{L}{2}+\tilde{P}_{\alpha,\beta}(1,1)\right)\alpha^{2}\sigma^{2},
\\
&\bar{\eta}:=\min\left\{\frac{\eta}{2},
\left(\frac{1}{2}-\frac{\rho_{\alpha,\beta}}{2}-\frac{K_{\alpha,\beta}}{R}\right)
\frac{R\eta}{4K_{\alpha,\beta}+R\eta}\right\}.
\end{align*}
\end{theorem}


Next, we obtain the optimal convergence rate
and provide a bound on the expected suboptimality
by choosing $(\alpha,\beta)=(\alpha_{AG},\beta_{AG})$.
\begin{proposition}\label{prop:general}
Given $(\alpha,\beta)=(\alpha_{AG},\beta_{AG})$.
Define $M$ and $R$ as in Theorem \ref{thm:general:alpha:beta}
with $\eta=1/\kappa^{1/2}$.
Also assume that the noise has small variance, i.e.
$\sigma^{2}\leq RL/(4\sqrt{\kappa}).$
Then, with $\psi:=\frac{L}{2\sqrt{\kappa}\sigma^{2}}$, we have
\begin{align}\label{ineq-str-cvx-rate}
\mathcal{W}_{1}(\nu_{k,\alpha,\beta},\pi_{\alpha,\beta})
&\leq
c_{0}^{-1}d_{\psi}(\nu_{k,\alpha,\beta},\pi_{\alpha,\beta})
\\
&
\leq\left(1-\frac{1}{8\sqrt{\kappa}}\right)^{k}c_{0}^{-1}
d_{\psi}(\nu_{0,\alpha,\beta},\pi_{\alpha,\beta}),
\nonumber
\end{align}
where $\mathcal{W}_{1}$ is the standard $1$-Wasserstein distance and
for any initial state $\xi_{0}$,
\begin{equation}\label{eqn:difference:AG}
\mathbb{E}[f(x_{k})]-f(x_{\ast})
\leq
V_{P_{AG}}(\xi_{0})\left(1-\frac{1}{\sqrt{\kappa}}\right)^{k}
+\frac{\sqrt{\kappa}\sigma^{2}}{L}.
\end{equation}
\end{proposition}
\vspace{-0.1in}

The bound \eqref{eqn:difference:AG} is similar in spirit to Corollary 4.7. in \cite{StrConvex}
but with a different assumption on noise. 
We can see that the expected value of the objective with respect
to the $k$-th iterate is close to the true minimum
of the objective if $k$ is large, and the variance
of the noise $\sigma^{2}$ is small. 
In the special case when the noise are i.i.d. Gaussian, one
can compute the constants in closed-form.
\vspace{0.1in}
\begin{corollary}\label{cor:Gaussian}
If the noise $\varepsilon_{k}$ are i.i.d. Gaussian $\mathcal{N}(0,\Sigma)$,
where $\Sigma \prec L^2 I_{d}$. 
Then, Proposition \ref{prop:general} holds with
\begin{align*}
&M:=\left(-2\log\left(\left(1-\frac{1}{\kappa^{1/4}}\right)\sqrt{\det(I_{d}-L^{-2}\Sigma)}\right)\right)^{1/2},
\\
&R:=\left(-M+\sqrt{M^{2}+\frac{\log(L/\mu)}{2L^{2}\Vert\Sigma^{-1}\Vert}}\right)^{2}
\frac{(L-\mu)^{2}}{8(3\sqrt{L}-\sqrt{\mu})^{3}}.
\end{align*}
If we take $\mu=\Theta(1)$, then $L=\Theta(\kappa)$ and it follows that we have $M=O(\kappa^{-1/8})$ and $R=O\left(\kappa^{-13/4}\log^{2}(\kappa)\right)$. 
\end{corollary}


We note that Proposition \ref{prop:general} and Corollary \ref{cor:Gaussian} provide explicit bounds on the admissable noise level $\sigma^2$ to ensure accelerated convergence with respect to Wasserstein distances and expected suboptimality after $k$ iterations.
%
\section{ASPG and the weakly convex setting}
\textbf{Constrained optimization and ASPG.}
Our analysis for AG can be adapted to study the \emph{accelerated stochastic  projected gradient} (ASPG) method for constrained optimization problems 
$\min_{x\in\mathcal{C}}  f(x)$, 
where $\mathcal{C}\subset \mathbb{R}^d$ is a compact set with diameter $\diam:=\sup_{x,y\in\mathcal{C}}\|x-y\|_{2}$. 
Theorem \ref{thm:general:alpha:beta}, Proposition \ref{prop:general} and Corollary \ref{cor:Gaussian} extends to ASPG in a natural fashion with modified constants that reflect the diameter of the constraint set (see the supplementary file). Furthermore, due to the finiteness of the diameter, it can be shown that the metric $d_\psi$ implies the standard $p$-Wasserstein metric for any $p\geq 1$. We also provide bounds in expected suboptimality for ASPG. 

\textbf{Weakly convex functions.} If the objective is (weakly) convex but not strongly convex and the constraint set is bounded, our analysis for the strongly convex case can be adapted with minor modifications. Following standard regularization techniques (see e.g. \cite{lessard2016analysis,Bubeck2014}), that allow to approximate a weakly convex function with a strongly convex function, we provide explicit bounds on the noise level to obtain the accelerated $O(\varepsilon^{-1/2})$ rate up to a log factor on $\varepsilon$ in expected suboptimality in function values (see the supplementary file). 
\section{Conclusion}

We have studied accelerated convergence guarantees for a number
of stochastic momentum methods (SHB, ASG, ASPG)
for strongly and (weakly) convex smooth problems. 
First, we studied the special case when the objective
is quadratic and the gradient noise is additive and i.i.d.
with a finite second moment. Non-asymptotic guarantees for  accelerated linear convergence are obtained for the deterministic 
and stochastic AG and HB methods for any $p$-Wasserstein distance ($p\geq 1$),
and also for the ASG method in the weighted $2$-Wasserstein distance,
which builds on the dissipativity theory from the deterministic setting. Our analysis for HB and AG also leads to improved non-asymptotic convergence bounds in suboptimality after $k$ iterations for both deterministic and stochastic settings which is of independent interest. Second, we studied the (non-quadratic) strongly convex
optimization under the stochastic oracle model \textbf{(H1)--(H2)}.
Accelerated linear convergence rate is obtained for the ASG
method in the $1$-Wasserstein distance. 
Third, we studied the ASPG method for constrained stochastic strongly
convex optimization on a bounded domain. 
Accelerated linear convergence rate is obtained
in any $p$-Wasserstein distance ($p\geq 1$), and 
extension to the (weakly) convex setting will be discussed in the supplementary file. Our results provide performance bounds for stochastic momentum methods in expected suboptimality and in Wasserstein distances. Finally, the proofs of all the results in our paper will be given
in the supplementary file.

\section*{Acknowledgements}
Mert G\"urb\"uzbalaban and Bugra Can acknowledge support from the grants NSF DMS-1723085
and NSF CCF-1814888. Lingjiong Zhu is grateful
to the support from the grant NSF DMS-1613164.

\newpage

\bibliography{langevin,robust}
\bibliographystyle{alpha}

\appendix
 
\onecolumn

\section{Constrained Optimization and ASPG}\label{sec:projected:AG}

Consider the constrained optimization problem $\min_{x\in\mathcal{C}}f(x)$,
where $\mathcal{C}\subset \mathbb{R}^d$ is a compact set with a finite diameter $\diam:=\sup_{x,y\in \mathcal{C}} \|x-y\|_2$ and $G_M:=\max_{x\in\mathcal{C}} \|\nabla f(x)\|$. 
The accelerated stochastic projected gradient method (ASPG) consists of the iterations
\begin{align}\label{iter: ASPG1}
&\tilde{x}_{k+1}=\prox\left(\tilde{y}_{k}-\alpha(\nabla f(\tilde{y}_{k}) +\varepsilon_{k+1})\right),
\\
&\tilde{y}_{k}=(1+\beta)\tilde{x}_{k}-\beta \tilde{x}_{k-1},\label{iter: ASPG2}
\end{align}
where $\varepsilon_k$ is the random gradient error satisfying Assumption \ref{assump:noise}, $\alpha,\beta>0$ are the stepsize and momentum parameter and $\prox(x)$ denotes the projection of a point $x$ to the compact set $\mathcal{C}$.
For constrained problems, algorithms based on projection steps that restricts the iterates to the constraint set are more natural compared to the standard AG algorithm primarily designed for the unconstrained optimization \cite{Bubeck2014}. Accelerated projected gradient methods can also be viewed as a special case of the accelerated proximal gradient methods as the proximal operator reduces to a projection in a special case (see e.g. \cite{parikh2014proximal}).

We will show in Proposition \ref{prop:metric:2}
that the metric $d_\psi$ implies the standard $p$-Wasserstein metric in the sense that for any two  probability measures $\mu_1,\mu_2$ on the product space $\mathcal{C}^{2}:=\mathcal{C}\times\mathcal{C}$,
\begin{equation*}
\mathcal{W}_{p}(\mu_1,\mu_2) 
\leq 
2^{1/p}\mathcal{D}_{\mathcal{C}^{2}} \Vert \mu_{1}-\mu_{2}\Vert_{TV}^{1/p}
\leq\mathcal{D}_{\mathcal{C}^{2}} d_{\psi}^{1/p}(\mu_{1},\mu_{2}),
\end{equation*}
where $\mathcal{D}_{\mathcal{C}^{2}}=\sqrt{2}D_C$ is the diameter of $\mathcal{C}^{2}$.

Under Assumption \ref{assump:noise}, $\tilde{\xi}_{k}=(\tilde{x}_{k}^T,\tilde{x}_{k-1}^T)^T$
forms a time-homogeneous Markov chain and we assume $\tilde{\xi}_{0}\in\mathcal{C}^{2}$.
In addition to Assumption \ref{assump:noise}, we also assume that 
the random gradient error $\varepsilon_{k}$
admits a continuous density so that
conditional on $\tilde{\xi}_{k}=(\tilde{x}_{k}^T,\tilde{x}_{k-1}^T)^T$, 
$\tilde{x}_{k+1}$
also admits a continuous density, i.e.
\begin{equation*}
\mathbb{P}(\tilde{x}_{k+1}\in d\tilde{x}|\tilde{\xi}_{k}=\tilde{\xi})
=\tilde{p}(\tilde{\xi},\tilde{x})d\tilde{x},
\end{equation*}
where $\tilde{p}(\tilde{\xi},\tilde{x})>0$ is continuous in both $\tilde{\xi}$ and $\tilde{x}$.

For the ASPG method with any given $\alpha,\beta$
so that $\rho_{\alpha,\beta}$, $P_{\alpha,\beta}$ satisfy
the LMI inequality \eqref{ineq-ag-lmi}, 
the next result gives a bound of $k$-th iterate to
stationary distribution in the weighted total variation
distance and standard $p$-Wasserstein distance, 
and also a bound on the expected
suboptimality $\mathbb{E}[f(\tilde{x}_{k})]-f(\tilde{x}_{\ast})$ after $k$ iterations.

\begin{theorem}\label{thm:general:alpha:beta:projected}
Given any $\eta\in(0,1)$ and $R>0$ so that
\begin{equation*}
\inf_{\tilde{x}\in\mathcal{C}:\tilde{\xi}\in\mathcal{C}^{2},V_{P_{\alpha,\beta}}(\tilde{\xi})\leq R}
\frac{\tilde{p}(\tilde{\xi},\tilde{x})}
{\tilde{p}(\tilde{\xi}_{\ast},\tilde{x})}
\geq\eta.
\end{equation*}
Consider the Markov chain generated by the iterates $\tilde{\xi}_k^T = (\tilde{x}_k^T, \tilde{x}_{k-1}^T)$ of the ASPG algorithm. Then the distribution $\tilde{\nu}_{k,\alpha,\beta}$ of $\tilde{\xi_k}$ converges linearly to a unique invariant distribution $\tilde{\pi}_{\alpha,\beta}$
satisfying 
\begin{equation}
\mathcal{W}_{p}(\tilde{\nu}_{k,\alpha,\beta},\tilde{\pi}_{\alpha,\beta})
\leq
\mathcal{D}_{\mathcal{C}^{2}}d_{\tilde{\psi}}^{1/p}(\tilde{\nu}_{k,\alpha,\beta},\tilde{\pi}_{\alpha,\beta})
\leq(1-\tilde{\eta})^{k}
\mathcal{D}_{\mathcal{C}^{2}}
d_{\tilde{\psi}}^{1/p}(\tilde{\nu}_{0,\alpha,\beta},\tilde{\pi}_{\alpha,\beta}),
\end{equation}
where $\mathcal{W}_{p}$ is the standard $p$-Wasserstein metric ($p\geq 1$) and
\begin{equation}\label{eqn:difference:projected}
\mathbb{E}[f(\tilde{x}_{k})]-f(\tilde{x}_{\ast})
\leq
V_{P_{\alpha,\beta}}(\tilde{\xi}_{0})\rho_{\alpha,\beta}^{k}
+\frac{\tilde{K}_{\alpha,\beta}}{1-\rho_{\alpha,\beta}},
\end{equation}
where  
\begin{align*}
&\tilde{K}_{\alpha,\beta}
:=
\alpha\sigma\left(\left(\alpha\sigma+2\mathcal{D}_{\mathcal{C}} \right)\Vert P_{\alpha,\beta}\Vert
+G_M + \frac{\alpha\sigma L}{2}\right),
\\ 
&\tilde{\eta}:=\min\left\{\frac{\eta}{2},
\left(\frac{1}{2}-\frac{\rho_{\alpha,\beta}}{2}-\frac{\tilde{K}_{\alpha,\beta}}{R}\right)
\frac{R\eta}{4\tilde{K}_{\alpha,\beta}+R\eta}\right\},
\end{align*} 
and $\tilde{\psi}:=\frac{\eta}{2\tilde{K}_{\alpha,\beta}}$.
\end{theorem}


We can see from \eqref{eqn:difference:projected} that
the expected value of the objective with respect
to the $k$-th iterate is close to the true minimum
of the objective if $k$ is large, and the stepsize $\alpha$ or the variance
of the noise $\sigma^{2}$ is small. By choosing $(\alpha,\beta)=(\alpha_{AG},\beta_{AG})$, 
we obtain the optimal convergence in the next theorem.

\begin{proposition}\label{prop:general:projected}
Given $(\alpha,\beta)=(\alpha_{AG},\beta_{AG})$.
Define $R$ as in Theorem \ref{thm:general:alpha:beta:projected}
with $\eta=1/\kappa^{1/2}$.
Also assume that the noise has small variance, i.e.
\begin{equation*}
\sigma^{2}
<\frac{1}{4a_{1}^{2}}\left(-b_{1}+\sqrt{b_{1}^{2}+\left(a_{1}R/\sqrt{\kappa}\right)}\right)^{2},
\end{equation*}
where
$a_{1}:=\frac{1}{L^{2}} \left(\frac{\mu}{2}((1-\sqrt{\kappa})^{2}+\kappa) + \frac{L}{2}\right)$
and 
$b_{1}:=\frac{1}{L} \left(\diam  \mu((1-\sqrt{\kappa})^{2}+\kappa) + G_M\right)$.
Then, we have
\begin{equation}
\mathcal{W}_{p}(\tilde{\nu}_{k,\alpha,\beta},\tilde{\pi}_{\alpha,\beta})
\leq\mathcal{D}_{\mathcal{C}^{2}}d_{\tilde{\psi}}^{1/p}(\tilde{\nu}_{k,\alpha,\beta},\tilde{\pi}_{\alpha,\beta})
\leq\left(1-\frac{1}{8\sqrt{\kappa}}\right)^{k}
\mathcal{D}_{\mathcal{C}^{2}}
d_{\tilde{\psi}}^{1/p}(\tilde{\nu}_{0,\alpha,\beta},\tilde{\pi}_{\alpha,\beta}),
\end{equation}
where $\mathcal{W}_{p}$ is the standard $p$-Wasserstein metric ($p\geq 1$) and
\begin{equation}\label{eqn:difference:AG:projected}
\mathbb{E}[f(\tilde{x}_{k})]-f(\tilde{x}_{\ast})
\leq
V_{P_{AG}}(\tilde{\xi}_{0})\left(1-\frac{1}{\sqrt{\kappa}}\right)^{k}
+\sqrt{\kappa}\tilde{K},
\end{equation}
where 
$\tilde{K}
:=\frac{2\sigma\diam L+\sigma^{2}}{2L^{2}}
\mu((1-\sqrt{\kappa})^{2}+\kappa) + \frac{\sigma G_M}{L}
+\frac{\sigma^{2}}{2L}$ and $\tilde{\psi}:=\frac{1}{2\sqrt{\kappa}\tilde{K}}$.
\end{proposition}





\section{Weakly Convex Constrained Optimization}

In this section, we extend the constrained optimization for the accelerated
stochastic projected gradient method (ASPG) from the strongly convex objectives
studied in Section \ref{sec:projected:AG} to the (weakly) convex objectives.


Consider the constrained optimization problem $\min_{x\in\mathcal{C}}f(x)$ for $f\in \mathcal{S}_{0,L}$ on the convex compact domain $\mathcal{C}\subseteq \mathbb{R}^{d}$ with diameter $\diam$. Consider the following (regularized) function
\begin{equation*}
f_\varepsilon(x) = f(x) + \frac{\varepsilon}{2\diam^2}\|x\|^2,
\end{equation*}
which is strongly convex with parameter $\mu_{\varepsilon}= \varepsilon/ \diam^2$ and smooth with parameter $L_\varepsilon = L+\varepsilon/\diam^2$, i.e. $f_\varepsilon \in \mathcal{S}_{\mu_\varepsilon, L_\varepsilon}$ with a condition number $\kappa_\varepsilon := L_\varepsilon/\mu_\varepsilon = 1 + L\diam^2/\varepsilon$. 
Let $\tilde{x}^{\varepsilon}_{k}$ denote iterates of ASPG defined by $f_{\varepsilon}$  (i.e $f=f_{\varepsilon}(x)$) in \eqref{iter: ASPG1} and \eqref{iter: ASPG2}) with optimal value $\tilde{x}_*^{\varepsilon}$ and define $\tilde{x}_*$ to be one of the minimizers of $f(x)$ (the optimizer
may not be unique).
 By applying Proposition \ref{prop:general:projected}, we can control the expected suboptimality
after $k$ iterations as follows:
\begin{equation*}
\mathbb{E}[f_{\varepsilon}(\tilde{x}_{k}^{\varepsilon})]-f_{\varepsilon}(\tilde{x}_{\ast}^{\varepsilon})
\leq
V_{P^{\varepsilon}_{AG}}(\tilde{\xi}_{0})\left(1-\frac{1}{\sqrt{\kappa_{\varepsilon}}}\right)^{k}
+\sqrt{\kappa_{\varepsilon}}\tilde{K}_{\varepsilon},
\end{equation*}
where
\begin{equation*}
\tilde{K}_{\varepsilon}
:=\frac{2\sigma\diam L_{\varepsilon}+\sigma^{2}}{2L_{\varepsilon}^{2}}
\mu_{\varepsilon}((1-\sqrt{\kappa_{\varepsilon}})^{2}+\kappa_{\varepsilon}) + \frac{\sigma G_M^{\varepsilon}}{L_{\varepsilon}}
+\frac{\sigma^{2}}{2L_{\varepsilon}}.
\end{equation*}
\\
Therefore, 
\begin{eqnarray*}
\mathbb{E}[f(\tx_{k}^{\varepsilon})]-f(\tilde{x}_*)&=&\mathbb{E}\big[f_\varepsilon(\tx_k^{\varepsilon})\big]-f_{\varepsilon}(\tilde{x}_*)
+\frac{\varepsilon}{2\diam^2}\left(\Vert \tilde{x}_*\Vert^2- \mathbb{E}[\Vert\tx_k^{\varepsilon}\Vert^2]\right) 
\\
&\leq&\mathbb{E}\big[f_\varepsilon(\tx_k^{\varepsilon})\big]-f_{\varepsilon}(\tilde{x}_{*}^{\varepsilon})
+\frac{\varepsilon}{2\diam^2}\left(\Vert \tilde{x}_*\Vert^2- \mathbb{E}[\Vert\tx_k^{\varepsilon}\Vert^2]\right) 
\\
&\leq& V_{P^{\varepsilon}_{AG}}(\tilde{\xi}_{0})\left(1-\frac{1}{\sqrt{\kappa_{\varepsilon}}}\right)^{k}+\sqrt{\kappa_{\varepsilon}}\tilde{K}_{\varepsilon} + \frac{\varepsilon}{2},
\end{eqnarray*}
where we used the fact that $\tx_k^{\varepsilon},\tx_* \in \mathcal{C}$. 
Therefore, if the noise level $\sigma$ is small enough such that $\sqrt{\kappa_{\varepsilon}}\tilde{K}_{\varepsilon}\leq \frac{\varepsilon}{2}$ and if
$$ k \geq \frac{|\log(\varepsilon)-\log(V_{P^{\varepsilon}_{AG}}(\tilde{\xi}_{0}))|}{|\log(1-\frac{1}{\sqrt{\kappa_\varepsilon}})|} = O\left(\frac{1}{\sqrt{\varepsilon}}\log\left(\frac{1}{\varepsilon}\right)\right),$$
we obtain \begin{eqnarray}
\mathbb{E}[f(\tx_{k}^{\varepsilon})]-f(\tilde{x}_*) \leq 2\varepsilon.
\end{eqnarray}
This shows that if the noise is small is enough, it suffices to have 
$$O\left(\frac{1}{\sqrt{\varepsilon}}\log\left(\frac{1}{\varepsilon}\right)\right)$$
many iterations to sample an $\varepsilon$-optimal point in expectation.
\section{Proofs of Results in Section \ref{sec:quadratic}}

In this section, we prove the results for Section \ref{sec:quadratic},
in which the objective is quadratic: $f(x)=\frac{1}{2}x^{T}Qx+a^{T}x+b$
and $f\in\mathcal{S}_{\mu,L}$, which satisfies the inequalities:
\begin{align*}
f(x)-f(y)\geq\nabla f(y)^{T}(x-y)+\frac{\mu}{2}\Vert x-y\Vert^{2},
\\
f(y)-f(x)\geq\nabla f(y)^{T}(y-x)-\frac{L}{2}\Vert x-y\Vert^{2},
\end{align*}
(see e.g. \cite{nesterov2004introductory}). 

\subsection{Proofs of Results in Section \ref{sec:AG:quadratic}}

Before we proceed to the proofs of the results in 
Section \ref{sec:AG:quadratic}, we first show
that the matrix $S_{\alpha,\beta}$ defined in \eqref{def-S-matrix}
is positive definite so that the weighted 2-Wasserstein metric $\mathcal{W}_{2,S_{\alpha,\beta}}$ 
given in \eqref{def-wass-2-norm} is well-defined.

\begin{lemma}\label{lem-S-positive-def}
The matrix $S_{\alpha,\beta}\in \R^{2d\times 2d}$ defined by \eqref{def-S-matrix} is positive definite if $\tilde{P}_{\alpha,\beta}(2,2)\neq 0$. 
\end{lemma}

\begin{proof} For brevity of the notation, we will not explicitly write the dependency of the matrices to ${\alpha,\beta}$ and set $P = P_{\alpha,\beta}$ and $\tilde{P} = P_{\alpha,\beta}$ in our discussion. It is known that if $A\in R^{n\times n} $ is a symmetric matrix with eigenvalues $\{\lambda_i\}_{i=1}^m$ and eigenvectors $\{a_i\}_{i=1}^n$, and $B\in \R^{d\times d}$ is a symmetric matrix with eigenvalues $\{\mu_j\}_{j=1}^d$ and eigenvectors $\{b_j\}_{j=1}^n$, the eigenvalues of the Kronecker product $A\otimes B$ are exactly $\lambda_i \mu_j$ with corresponding eigenvectors $a_i \otimes b_j$ for $i=1,2,\dots, n$ and $j=1,2,\dots,d$. Since $P = \tilde{P} \otimes I_d$ and $\tilde{P}$ is positive-semi definite by assumption, this implies that $P$ is positive semi-definite and in case $P$ has a zero eigenvalue, any eigenvector $z$ of $P$ (corresponding to a zero eigenvalue of $P$) can be written as
    $$z = \begin{pmatrix} c_1 \\ c_2 \end{pmatrix} \otimes s = \begin{pmatrix}
     c_1 s
      \\
       c_2 s 
        \end{pmatrix} \in \R^{2d},  \label{eq-eigvec-format}               $$ 
for some $s\in \R^d$, $s\neq 0$ where $ c= [c_1 ~ c_2]^T$ is an eigenvector of $\tilde{P}$ corresponding to a zero eigenvalue. The symmetric matrix
  \beq S := P + \hat{Q}, \qquad \text{where }\quad\hat{Q}:= \begin{pmatrix} \frac{1}{2}Q & 0_d \\
                    0_d & 0_d
                    \label{eq-S}
            \end{pmatrix},
   \eeq
is the sum of two positive semi-definite matrices, therefore it is positive semi-definite by the eigenvalue interlacing property of the sum of symmetric matrices (see e.g. \cite{golub1996matrix}). Thus, it suffices to show that $S$ is non-singular, i.e. it does not have a zero eigenvalue. If $\tilde{P}$ is of full rank, then such a vector $z$ cannot exist and $P$ cannot have a zero eigenvalue. Therefore, $P$ is positive definite and hence $S$ is positive definite which completes the proof. 

The remaining case is when $\tilde{P}$ is of rank one ($\tilde{P}=0$ is excluded as $\tilde{P}_{22}\neq 0$) in which case we can write $\tilde{P} = uu^T$ for some $u = \begin{pmatrix}
 u_1 & u_2 
\end{pmatrix}^T \in\R^{2d}$ and $u_2 \neq 0$. We will prove the claim by contradiction. Assume that there exists a non-zero $v\in\R^{2d}$ such that $Sv = 0$. Then, 
\begin{equation*}
0 = v^T S v  =  v^T P v + v^T \hat{Q} v. 
\end{equation*}
Since both of the matrices $P$ and $\hat{Q}$ are positive semi-definite, this is true if and only if $v^TPv = 0$ and $v^T\hat{Q}v = 0$. Since $v^T\hat{Q}v = 0$ and $Q$ is positive definite, from the structure of $\hat{Q}$, it follows that the first $d$ entries of $v$ has to be zero, i.e. $v = [0 \quad v_2^T]^T$ for some $v_2 \in \R^d$.

It is easy to see that the eigenvalues of the two by two symmetric rank-one matrix $\tilde{P}=u u^T $ are $\lambda_1 = \|u\|^2 > 0 $ and $\lambda_2 = 0$ with corresponding eigenvectors $\begin{pmatrix}
 u_1 & u_2 
\end{pmatrix}^T$ and $ \begin{pmatrix}
 u_2 & -u_1 
\end{pmatrix}^T$ respectively.  Since $v$ is an eigenvector of $P$ corresponding to an eigenvalue zero (i.e. $Pv =0$), then using \eqref{eq-eigvec-format} we can write
    $$v = \begin{pmatrix} u_2 \\ -u_1 \end{pmatrix} \otimes s = \begin{pmatrix}
     u_2 s
      \\
    - u_1 s
                          \end{pmatrix} \in \R^{2d},$$ 
for some $s\in \R^d$, $s\neq 0$. Since $v = [0 \quad v_2^T]^T$ for some $v_2 \in \R^d$, this implies $u_2 = 0$ as $s\neq 0$. This is a contradiction.
\end{proof}

Next, before we proceed to the proofs of the results
in Section \ref{sec:AG:quadratic}, let us first recall
that throughout Section \ref{sec:quadratic}, 
the noise $\varepsilon_{k}$ are assumed to be i.i.d. 
Let us define the coupling
\begin{align}
&x_{k+1}^{(j)}=y_{k}^{(j)}-\alpha\left[\nabla f\left(y_{k}^{(j)}\right)+\varepsilon_{k+1}\right],
\\
&y_{k}^{(j)}=(1+\beta)x_{k}^{(j)}-\beta x_{k-1}^{(j)},
\end{align}
with $j=1,2$. Then, we have
\begin{equation*}
\xi_{k+1}=A\xi_{k}+Bw_{k},
\end{equation*}
where 
$A=\tilde{A}\otimes I_{d}$,
$B=\tilde{B}\otimes I_{d}$, 
for
\begin{equation*}
\tilde{A}
=
\left(
\begin{array}{cc}
1+\beta & -\beta
\\
1 & 0
\end{array}
\right),
\qquad
\tilde{B}
=\left(
\begin{array}{c}
-\alpha
\\
0
\end{array}
\right),
\end{equation*}
and
\begin{align}
&\xi_{k}=\left(\left(x_{k}^{(1)}-x_{k}^{(2)}\right)^{T},\left(x_{k-1}^{(1)}-x_{k-1}^{(2)}\right)^{T}\right)^{T},
\\
&w_{k}=\nabla f\left((1+\beta)x_{k}^{(1)}-\beta x_{k-1}^{(1)}\right)
-\nabla f\left((1+\beta)x_{k}^{(2)}-\beta x_{k-1}^{(2)}\right).
\end{align}

Let us define:
\begin{equation}\label{def:X}
\tilde{X}=\rho\tilde{X}_{1}+(1-\rho)\tilde{X}_{2},
\end{equation}
where
\begin{equation}\label{def:X1}
\tilde{X}_{1}=\frac{1}{2}
\left(
\begin{array}{ccc}
\beta^{2}\mu & -\beta^{2}\mu & -\beta
\\
-\beta^{2}\mu & \beta^{2}\mu & \beta
\\
-\beta & \beta & \alpha(2-L\alpha)
\end{array}
\right),
\end{equation}
and
\begin{equation}\label{def:X2}
\tilde{X}_{2}=\frac{1}{2}
\left(
\begin{array}{ccc}
(1+\beta)^{2}\mu & -\beta(1+\beta)\mu & -(1+\beta)
\\
-\beta(1+\beta)\mu & \beta^{2}\mu & \beta
\\
-(1+\beta) & \beta & \alpha(2-L\alpha)
\end{array}
\right),
\end{equation}
and $X=\tilde{X}\otimes I_{d}$, 
$X_{1}=\tilde{X}_{1}\otimes I_{d}$, $X_{2}=\tilde{X}_{2}\otimes I_{d}$.

Before we proceed, let us recall the following lemma
from \cite{hu2017dissipativity}.

\begin{lemma}[Theorem 2 \cite{hu2017dissipativity}]\label{lemma:servicerate}
Let $X$ be a symmetric matrix with $X\in\mathbb{R}^{(n_{\varepsilon}+n_{w})\times(n_{\varepsilon}+n_{w})}$.
If there exists a matrix $P\in\mathbb{R}^{n_{\varepsilon}\times n_{\varepsilon}}$ with $P\geq 0$
so that
\begin{equation*}
\left(
\begin{array}{cc}
A^{T}PA-\rho P & A^{T}PB
\\
B^{T}PA & B^{T}PB
\end{array}
\right)
-X\preceq 0,
\end{equation*}
then, we have
\begin{equation*}
V(\xi_{k+1})-\rho V(\xi_{k})
\leq
S(\xi_{k},w_{k}),
\end{equation*}
where $V(\xi):=\xi^{T}P\xi$, and
\begin{equation*}
S(\xi,w):=
\left(
\begin{array}{c}
\xi
\\
w
\end{array}
\right)^{T}
X
\left(
\begin{array}{c}
\xi
\\
w
\end{array}
\right),
\end{equation*}
and
\begin{equation*}
\xi_{k+1}=A\xi_{k}+Bw_{k}.
\end{equation*}
\end{lemma}

The proof of Theorem \ref{thm:alpha:beta} relies
on the following lemma.

\begin{lemma}\label{lem:quadratic}
Assume the coupling:
\begin{align}
&x_{k+1}^{(j)}=y_{k}^{(j)}-\alpha\left[\nabla f\left(y_{k}^{(j)}\right)+\varepsilon_{k+1}\right],\label{1:couple}
\\
&y_{k}^{(j)}=(1+\beta)x_{k}^{(j)}-\beta x_{k-1}^{(j)},\label{2:couple}
\end{align}
with $j=1,2$.
Assume that $f$ is quadratic and $f(x)=\frac{1}{2}x^{T}Q x+a^{T}x+b$,
where $Q$ is positive definite. 

Let $\rho=\rho_{\alpha,\beta}\in(0,1)$ 
that can depend
on $\alpha$ and $\beta$ so that
there exists some $P=P_{\alpha,\beta}$ symmetric
and positive semi-definite that can depend
on $\alpha$ and $\beta$ such that
\begin{equation}\label{ineq:ABPX}
\left(
\begin{array}{cc}
A^{T}PA-\rho P & A^{T}PB
\\
B^{T}PA & B^{T}PB
\end{array}
\right)
-X\preceq 0,
\end{equation}
where $X:=\tilde{X}\otimes I_{d}$, where $\tilde{X}$
is defined in \eqref{def:X}.
Then, we have
\begin{align*}
&\mathbb{E}\Bigg[\left(
\begin{array}{c}
x_{k+1}^{(1)}-x_{k+1}^{(2)}
\\
x_{k}^{(1)}-x_{k}^{(2)}
\end{array}
\right)^{T}
P_{\alpha,\beta}
\left(
\begin{array}{c}
x_{k+1}^{(1)}-x_{k+1}^{(2)}
\\
x_{k}^{(1)}-x_{k}^{(2)}
\end{array}
\right)
+\frac{1}{2}\left(x_{k+1}^{(1)}-x_{k+1}^{(2)}\right)^{T}Q\left(x_{k+1}^{(1)}-x_{k+1}^{(2)}\right)\Bigg]
\\
&\leq
\rho_{\alpha,\beta}
\left(\mathbb{E}\Bigg[\Bigg(
\begin{array}{c}
x_{k}^{(1)}-x_{k}^{(2)}
\\
x_{k-1}^{(1)}-x_{k-1}^{(2)}
\end{array}
\right)^{T}
P_{\alpha,\beta}
\left(
\begin{array}{c}
x_{k}^{(1)}-x_{k}^{(2)}
\\
x_{k-1}^{(1)}-x_{k-1}^{(2)}
\end{array}
\right)
\\
&\qquad\qquad\qquad\qquad\qquad
+\frac{1}{2}\left(x_{k}^{(1)}-x_{k}^{(2)}\right)^{T}Q\left(x_{k}^{(1)}-x_{k}^{(2)}\right)\Bigg]\Bigg).
\end{align*}
\end{lemma}

\begin{proof}[Proof of Lemma \ref{lem:quadratic}]
First of all, since $f$ is $L$-smooth and $\mu$-strongly convex, we have for every $x,y\in\mathbb{R}^{d}$:
\begin{align}
&f(x)-f(y)\geq\nabla f(y)^{T}(x-y)+\frac{\mu}{2}\Vert x-y\Vert^{2},\label{S1}
\\
&f(y)-f(x)\geq\nabla f(y)^{T}(y-x)-\frac{L}{2}\Vert y-x\Vert^{2}.\label{S2}
\end{align}
Note that since $f$ is $L$-smooth, we also have for every $x,y\in\mathbb{R}^{d}$:
\begin{equation*}
\Vert\nabla f(x)-\nabla f(y)\Vert
\leq L\Vert x-y\Vert.
\end{equation*}

Let us first consider the simpler case $f(x)=\frac{1}{2}x^{T}Qx$.
Since $f$ is quadratic, $\nabla f$ is linear. 
Applying \eqref{S1} and the linearity of $\nabla f$, we get
\begin{align*}
&f\left(x_{k}^{(1)}-x_{k}^{(2)}\right)
-f\left(y_{k}^{(1)}-y_{k}^{(2)}\right)
\\
&\geq\left(\nabla f\left(y_{k}^{(1)}\right)-\nabla f\left(y_{k}^{(2)}\right)\right)^{T}
\left(x_{k}^{(1)}-x_{k}^{(2)}-\left(y_{k}^{(1)}-y_{k}^{(2)}\right)\right)
\\
&\qquad\qquad\qquad
+\frac{\mu}{2}\left\Vert x_{k}^{(1)}-x_{k}^{(2)}-\left(y_{k}^{(1)}-y_{k}^{(2)}\right)\right\Vert^{2}.
\end{align*}
Applying \eqref{S2} and the linearity of $\nabla f$, we get
\begin{align*}
&f\left(y_{k}^{(1)}-y_{k}^{(2)}\right)
-f\left(y_{k}^{(1)}-y_{k}^{(2)}-\alpha\nabla f\left(y_{k}^{(1)}-y_{k}^{(2)}\right)\right)
\\
&\geq\frac{\alpha}{2}(2-L\alpha)\left\Vert\nabla f\left(y_{k}^{(1)}\right)-\nabla f\left(y_{k}^{(2)}\right)\right\Vert^{2}.
\end{align*}
Using the identity:
\begin{equation*}
x_{k+1}^{(1)}-x_{k+1}^{(2)}
=y_{k}^{(1)}-y_{k}^{(2)}-\alpha\nabla f\left(y_{k}^{(1)}-y_{k}^{(2)}\right),
\end{equation*}
we get
\begin{equation*}
f\left(y_{k}^{(1)}-y_{k}^{(2)}\right)-f\left(x_{k+1}^{(1)}-x_{k+1}^{(2)}\right)
\geq\frac{\alpha}{2}(2-L\alpha)
\left\Vert\nabla f\left(y_{k}^{(1)}\right)-\nabla f\left(y_{k}^{(2)}\right)\right\Vert^{2}.
\end{equation*}
Hence, we get
\begin{align*}
&f\left(x_{k}^{(1)}-x_{k}^{(2)}\right)-f\left(x_{k+1}^{(1)}-x_{k+1}^{(2)}\right)
\\
&\geq\left(\nabla f\left(y_{k}^{(1)}\right)-\nabla f\left(y_{k}^{(2)}\right)\right)^{T}\left(x_{k}^{(1)}-x_{k}^{(2)}-\left(y_{k}^{(1)}-y_{k}^{(2)}\right)\right)
\\
&+\frac{\mu}{2}\left\Vert x_{k}^{(1)}-x_{k}^{(2)}-\left(y_{k}^{(1)}-y_{k}^{(2)}\right)\right\Vert^{2}
+\frac{\alpha}{2}(2-L\alpha)\left\Vert\nabla f\left(y_{k}^{(1)}\right)-\nabla f\left(y_{k}^{(2)}\right)\right\Vert^{2}.
\end{align*}
By the definition of $\tilde{X}_{1}$ from \eqref{def:X1}, with $X_{1}=\tilde{X}_{1}\otimes I_{d}$, we get
\begin{align*}
&\left(
\begin{array}{c}
x_{k}^{(1)}-x_{k}^{(2)}
\\
x_{k-1}^{(1)}-x_{k-1}^{(2)}
\\
\nabla f(y_{k}^{(1)})-\nabla f(y_{k}^{(2)})
\end{array}
\right)^{T}
X_{1}
\left(
\begin{array}{c}
x_{k}^{(1)}-x_{k}^{(2)}
\\
x_{k-1}^{(1)}-x_{k-1}^{(2)}
\\
\nabla f(y_{k}^{(1)})-\nabla f(y_{k}^{(2)})
\end{array}
\right)
\\
&\leq
f\left(x_{k}^{(1)}-x_{k}^{(2)}\right)-f\left(x_{k+1}^{(1)}-x_{k+1}^{(2)}\right).
\end{align*}
Similarly, by applying \eqref{S1} with $(x,y)\mapsto(0,y_{k}^{(1)}-y_{k}^{(2)})$, 
by the definition of $\tilde{X}_{2}$ from \eqref{def:X2}, 
with $X_{2}=\tilde{X}_{2}\otimes I_{d}$, we get
\begin{equation*}
\left(
\begin{array}{c}
x_{k}^{(1)}-x_{k}^{(2)}
\\
x_{k-1}^{(1)}-x_{k-1}^{(2)}
\\
\nabla f(y_{k}^{(1)})-\nabla f(y_{k}^{(2)})
\end{array}
\right)^{T}
X_{2}
\left(
\begin{array}{c}
x_{k}^{(1)}-x_{k}^{(2)}
\\
x_{k-1}^{(1)}-x_{k-1}^{(2)}
\\
\nabla f(y_{k}^{(1)})-\nabla f(y_{k}^{(2)})
\end{array}
\right)
\leq
f(0)-f\left(x_{k+1}^{(1)}-x_{k+1}^{(2)}\right).
\end{equation*}
By using $\tilde{X}=\rho \tilde{X}_{1}+(1-\rho )\tilde{X}_{2}$
and $X=\tilde{X}\otimes I_{d}$, we get
\begin{align*}
&\left(
\begin{array}{c}
x_{k}^{(1)}-x_{k}^{(2)}
\\
x_{k-1}^{(1)}-x_{k-1}^{(2)}
\\
\nabla f(y_{k}^{(1)})-\nabla f(y_{k}^{(2)})
\end{array}
\right)^{T}
X
\left(
\begin{array}{c}
x_{k}^{(1)}-x_{k}^{(2)}
\\
x_{k-1}^{(1)}-x_{k-1}^{(2)}
\\
\nabla f(y_{k}^{(1)})-\nabla f(y_{k}^{(2)})
\end{array}
\right)
\\
&\qquad\qquad\qquad\leq
-\left(f\left(x_{k+1}^{(1)}-x_{k+1}^{(2)}\right)-f(0)\right)
+\rho \left(f\left(x_{k}^{(1)}-x_{k}^{(2)}\right)-f(0)\right).
\end{align*}
By Lemma \ref{lemma:servicerate} and the definition
of $\rho_{\alpha,\beta}$, $P_{\alpha,\beta}$
the inequality  \eqref{ineq:ABPX} holds. Thus
\begin{align*}
&\left(
\begin{array}{c}
x_{k+1}^{(1)}-x_{k+1}^{(2)}
\\
x_{k}^{(1)}-x_{k}^{(2)}
\end{array}
\right)^{T}
P_{\alpha,\beta}
\left(
\begin{array}{c}
x_{k+1}^{(1)}-x_{k+1}^{(2)}
\\
x_{k}^{(1)}-x_{k}^{(2)}
\end{array}
\right)
+f\left(x_{k+1}^{(1)}-x_{k+1}^{(2)}\right)-f(0)
\\
&\leq
\rho_{\alpha,\beta}\left(
\left(
\begin{array}{c}
x_{k}^{(1)}-x_{k}^{(2)}
\\
x_{k-1}^{(1)}-x_{k-1}^{(2)}
\end{array}
\right)^{T}
P_{\alpha,\beta}
\left(
\begin{array}{c}
x_{k}^{(1)}-x_{k}^{(2)}
\\
x_{k-1}^{(1)}-x_{k-1}^{(2)}
\end{array}
\right)
+f\left(x_{k}^{(1)}-x_{k}^{(2)}\right)-f(0)\right).
\end{align*}
Since $f$ is quadratic, and we assumed that
$f(x)=\frac{1}{2}x^{T}Q x$, where $Q$ is positive definite, we get 
\begin{align*}
&\left(
\begin{array}{c}
x_{k+1}^{(1)}-x_{k+1}^{(2)}
\\
x_{k}^{(1)}-x_{k}^{(2)}
\end{array}
\right)^{T}
P_{\alpha,\beta}
\left(
\begin{array}{c}
x_{k+1}^{(1)}-x_{k+1}^{(2)}
\\
x_{k}^{(1)}-x_{k}^{(2)}
\end{array}
\right)
+\frac{1}{2}\left(x_{k+1}^{(1)}-x_{k+1}^{(2)}\right)^{T}Q\left(x_{k+1}^{(1)}-x_{k+1}^{(2)}\right)
\\
&\leq
\rho_{\alpha,\beta}
\left(\left(
\begin{array}{c}
x_{k}^{(1)}-x_{k}^{(2)}
\\
x_{k-1}^{(1)}-x_{k-1}^{(2)}
\end{array}
\right)^{T}
P_{\alpha,\beta}
\left(
\begin{array}{c}
x_{k}^{(1)}-x_{k}^{(2)}
\\
x_{k-1}^{(1)}-x_{k-1}^{(2)}
\end{array}
\right)
+\frac{1}{2}\left(x_{k}^{(1)}-x_{k}^{(2)}\right)^{T}Q\left(x_{k}^{(1)}-x_{k}^{(2)}\right)\right).
\end{align*}

Previously, we assumed $f(x)=\frac{1}{2}x^{T}Qx$, 
so that $\nabla f(x-y)=\nabla f(x)-\nabla f(y)$.
In general, the quadratic function takes the form
\begin{equation*}
f(x)=\frac{1}{2}x^{T}Qx+a^{T}x+b.
\end{equation*}
In this case,
\begin{equation*}
\nabla f(x-y)-(\nabla f(x)-\nabla f(y))=a^{T} (x-y).
\end{equation*}
By the definition of $\tilde{X}_{1}$ from \eqref{def:X1}, with $X_{1}=\tilde{X}_{1}\otimes I_{d}$, we get
\begin{align*}
&\left(
\begin{array}{c}
x_{k}^{(1)}-x_{k}^{(2)}
\\
x_{k-1}^{(1)}-x_{k-1}^{(2)}
\\
\nabla f(y_{k}^{(1)})-\nabla f(y_{k}^{(2)})
\end{array}
\right)^{T}
X_{1}
\left(
\begin{array}{c}
x_{k}^{(1)}-x_{k}^{(2)}
\\
x_{k-1}^{(1)}-x_{k-1}^{(2)}
\\
\nabla f(y_{k}^{(1)})-\nabla f(y_{k}^{(2)})
\end{array}
\right)
\\
&\leq
f\left(x_{k}^{(1)}-x_{k}^{(2)}\right)-f\left(x_{k+1}^{(1)}-x_{k+1}^{(2)}\right)
\\
&\qquad
+\left(\nabla f\left(y_{k}^{(1)}-y_{k}^{(2)}\right)-\nabla f\left(y_{k}^{(1)}\right)+\nabla f\left(y_{k}^{(2)}\right)\right)^{T}
\left(x_{k+1}^{(1)}-x_{k+1}^{(2)}-\left(x_{k}^{(1)}-x_{k}^{(2)}\right)\right).
\\
&=
f\left(x_{k}^{(1)}-x_{k}^{(2)}\right)-f\left(x_{k+1}^{(1)}-x_{k+1}^{(2)}\right)
+a^{T}\left(x_{k+1}^{(1)}-x_{k+1}^{(2)}-\left(x_{k}^{(1)}-x_{k}^{(2)}\right)\right).
\end{align*}
By the definition of $\tilde{X}_{2}$ from \eqref{def:X2}, 
with $X_{2}=\tilde{X}_{2}\otimes I_{d}$, we get
\begin{align*}
&\left(
\begin{array}{c}
x_{k}^{(1)}-x_{k}^{(2)}
\\
x_{k-1}^{(1)}-x_{k-1}^{(2)}
\\
\nabla f\left(y_{k}^{(1)}\right)-\nabla f\left(y_{k}^{(2)}\right)
\end{array}
\right)^{T}
X_{2}
\left(
\begin{array}{c}
x_{k}^{(1)}-x_{k}^{(2)}
\\
x_{k-1}^{(1)}-x_{k-1}^{(2)}
\\
\nabla f\left(y_{k}^{(1)}\right)-\nabla f\left(y_{k}^{(2)}\right)
\end{array}
\right)
\\
&\leq
f(0)-f\left(x_{k+1}^{(1)}-x_{k+1}^{(2)}\right)
+\left(\nabla f\left(y_{k}^{(1)}-y_{k}^{(2)}\right)-\nabla f\left(y_{k}^{(1)}\right)+\nabla f\left(y_{k}^{(2)}\right)\right)^{T}
\left(x_{k+1}^{(1)}-x_{k+2}^{(2)}\right)
\\
&=f(0)-f\left(x_{k+1}^{(1)}-x_{k+1}^{(2)}\right)
+a^{T}\left(x_{k+1}^{(1)}-x_{k+1}^{(2)}\right).
\end{align*}
Using $\tilde{X}=\rho \tilde{X}_{1}+(1-\rho )\tilde{X}_{2}$
and $X=\tilde{X}\otimes I_{d}$, we get
\begin{align*}
&\left(
\begin{array}{c}
x_{k}^{(1)}-x_{k}^{(2)}
\\
x_{k-1}^{(1)}-x_{k-1}^{(2)}
\\
\nabla f(y_{k}^{(1)})-\nabla f(y_{k}^{(2)})
\end{array}
\right)^{T}
X
\left(
\begin{array}{c}
x_{k}^{(1)}-x_{k}^{(2)}
\\
x_{k-1}^{(1)}-x_{k-1}^{(2)}
\\
\nabla f(y_{k}^{(1)})-\nabla f(y_{k}^{(2)})
\end{array}
\right)
\\
&\leq
-\left(f\left(x_{k+1}^{(1)}-x_{k+1}^{(2)}\right)-f(0)\right)
+\rho \left(f\left(x_{k}^{(1)}-x_{k}^{(2)}\right)-f(0)\right)
\\
&\qquad\qquad
+a^{T}\left(x_{k+1}^{(1)}-x_{k+1}^{(2)}-\rho \left(x_{k}^{(1)}-x_{k}^{(2)}\right)\right)
\\
&=-\frac{1}{2}\left(x_{k+1}^{(1)}-x_{k+1}^{(2)}\right)^{T}Q\left(x_{k+1}^{(1)}-x_{k+1}^{(2)}\right)
+\rho \frac{1}{2}\left(x_{k}^{(1)}-x_{k}^{(2)}\right)Q\left(x_{k}^{(1)}-x_{k}^{(2)}\right).
\end{align*}
Hence, 
by Lemma \ref{lemma:servicerate} and the definition
of $\rho_{\alpha,\beta}$, $P_{\alpha,\beta}$
so that \eqref{ineq:ABPX} holds,
we get the same result as before:
\begin{align*}
&\left(
\begin{array}{c}
x_{k+1}^{(1)}-x_{k+1}^{(2)}
\\
x_{k}^{(1)}-x_{k}^{(2)}
\end{array}
\right)^{T}
P_{\alpha,\beta}
\left(
\begin{array}{c}
x_{k+1}^{(1)}-x_{k+1}^{(2)}
\\
x_{k}^{(1)}-x_{k}^{(2)}
\end{array}
\right)
+\frac{1}{2}\left(x_{k+1}^{(1)}-x_{k+1}^{(2)}\right)^{T}Q\left(x_{k+1}^{(1)}-x_{k+1}^{(2)}\right)
\\
&\leq
\rho_{\alpha,\beta}
\left(\left(
\begin{array}{c}
x_{k}^{(1)}-x_{k}^{(2)}
\\
x_{k-1}^{(1)}-x_{k-1}^{(2)}
\end{array}
\right)^{T}
P_{\alpha,\beta}
\left(
\begin{array}{c}
x_{k}^{(1)}-x_{k}^{(2)}
\\
x_{k-1}^{(1)}-x_{k-1}^{(2)}
\end{array}
\right)
+\frac{1}{2}\left(x_{k}^{(1)}-x_{k}^{(2)}\right)^{T}Q\left(x_{k}^{(1)}-x_{k}^{(2)}\right)\right).
\end{align*}
\end{proof}

By taking $\alpha=\alpha_{AG}$, $\beta=\beta_{AG}$, $\rho=\rho_{AG}$
and $P_{AG}$ in definition \eqref{def-P-matrix}, 
we recall the following result from \cite{hu2017dissipativity}.

\begin{lemma}[\cite{hu2017dissipativity}]\label{lem:X},
With the choice 
\begin{equation*}
\alpha=\alpha_{AG}=\frac{1}{L},\qquad
\beta=\beta_{AG}=\frac{\sqrt{\kappa}-1}{\sqrt{\kappa}+1},
\qquad
\rho =\rho_{AG}=1-\frac{1}{\sqrt{\kappa}},
\end{equation*}
where $\kappa=L/\mu$ is the condition number,
there exists a matrix $\tilde{P}_{AG}\in\mathbb{R}^{2\times 2}$
with $\tilde{P}_{AG}\geq 0$, where
\begin{equation*}
\tilde{P}_{AG}
:= \tilde{u} \tilde{u}^T, \quad \tilde{u} = \begin{pmatrix}
\sqrt{\frac{L}{2}} & \sqrt{\frac{\mu}{2}}-\sqrt{\frac{L}{2}}
\end{pmatrix}^T,
\end{equation*}
such that $P_{AG}=\tilde{P}_{AG}\otimes I_{d}$ and
\begin{equation*}
\left(
\begin{array}{cc}
A^{T}P_{AG}A-\rho P_{AG} & A^{T}P_{AG}B
\\
B^{T}P_{AG}A & B^{T}P_{AG}B
\end{array}
\right)
-X\preceq 0,
\end{equation*}
where $X:=\tilde{X}\otimes I_{d}$, where $\tilde{X}$
is defined in \eqref{def:X}.
\end{lemma}

We immediately obtain the following result.

\begin{lemma}\label{lem:quadratic:2}
Assume the coupling \eqref{1:couple}-\eqref{2:couple}.
Assume that $f$ is quadratic and $f(x)=\frac{1}{2}x^{T}Q x+a^{T}x+b$,
where $Q$ is positive definite. 
Then, we have
\begin{align*}
&\mathbb{E}\Bigg[\left(
\begin{array}{c}
x_{k+1}^{(1)}-x_{k+1}^{(2)}
\\
x_{k}^{(1)}-x_{k}^{(2)}
\end{array}
\right)^{T}
P_{AG}
\left(
\begin{array}{c}
x_{k+1}^{(1)}-x_{k+1}^{(2)}
\\
x_{k}^{(1)}-x_{k}^{(2)}
\end{array}
\right)
+\frac{1}{2}\left(x_{k+1}^{(1)}-x_{k+1}^{(2)}\right)^{T}Q\left(x_{k+1}^{(1)}-x_{k+1}^{(2)}\right)\Bigg]
\\
&\leq
\rho_{AG}
\left(\mathbb{E}\Bigg[\Bigg(
\begin{array}{c}
x_{k}^{(1)}-x_{k}^{(2)}
\\
x_{k-1}^{(1)}-x_{k-1}^{(2)}
\end{array}
\right)^{T}
P_{AG}
\left(
\begin{array}{c}
x_{k}^{(1)}-x_{k}^{(2)}
\\
x_{k-1}^{(1)}-x_{k-1}^{(2)}
\end{array}
\right)
\\
&\qquad\qquad\qquad\qquad
+\frac{1}{2}\left(x_{k}^{(1)}-x_{k}^{(2)}\right)^{T}Q\left(x_{k}^{(1)}-x_{k}^{(2)}\right)\Bigg]\Bigg),
\end{align*}
where $P$ is defined in \eqref{def-P-matrix}.
\end{lemma}

Now, we are ready to state the proof of Theorem \ref{thm:alpha:beta}.

\begin{proof}[Proof of Theorem \ref{thm:alpha:beta}]
Recall the iterates $\xi_{k}=(x_{k}^{T},x_{k-1}^{T})^{T}$,
the Markov kernel $\mathcal{P}_{\alpha,\beta}$ and
the definition of the weighted $2$-Wasserstein distance \eqref{def-wass-2-norm} with the weighted norm \eqref{def-weighted-norm}-\eqref{def-S-matrix} and $P=P_{\alpha,\beta}$. Then
showing Theorem \ref{thm:alpha:beta} is equivalent to show 
\begin{align}
&\mathcal{W}_{2,S_{\alpha,\beta}}^{2}(R^{k}_{\alpha,\beta}((x_{0},x_{-1}),\cdot),\pi_{\alpha,\beta})
\\
&\leq\rho_{\alpha,\beta}^{k}
\int_{\mathbb{R}^{d}\times\mathbb{R}^{d}}
\Bigg[\left(
\begin{array}{c}
x_{0}-\hat{x}_{0}
\\
x_{-1}-\hat{x}_{-1}
\end{array}
\right)^{T}
P_{\alpha,\beta}
\left(
\begin{array}{c}
x_{0}-\hat{x}_{0}
\\
x_{-1}-\hat{x}_{-1}
\end{array}
\right)
\\
&\qquad\qquad\qquad
+\frac{1}{2}(x_{0}-\hat{x}_{0})^{T}Q(x_{0}-\hat{x}_{0})\Bigg]d\pi_{\alpha,\beta}(\hat{x}_{0},\hat{x}_{-1}).
\nonumber
\end{align}

Let $(((x_{k}^{(i)})^{T},(x_{k-1}^{(i)})^{T})^{T})_{k=0}^{\infty}$, $i=1,2$ be 
a coupling of $((x_{k}^{T},x_{k-1}^{T})^{T})_{k=0}^{\infty}$ defined as before.
We have shown before that for every $k$,
\begin{align*}
&\left(
\begin{array}{c}
x_{k+1}^{(1)}-x_{k+1}^{(2)}
\\
x_{k}^{(1)}-x_{k}^{(2)}
\end{array}
\right)^{T}
P_{\alpha,\beta}
\left(
\begin{array}{c}
x_{k+1}^{(1)}-x_{k+1}^{(2)}
\\
x_{k}^{(1)}-x_{k}^{(2)}
\end{array}
\right)
+\frac{1}{2}\left(x_{k+1}^{(1)}-x_{k+1}^{(2)}\right)^{T}Q\left(x_{k+1}^{(1)}-x_{k+1}^{(2)}\right)
\\
&\leq
\rho_{\alpha,\beta}
\left[\left(
\begin{array}{c}
x_{k}^{(1)}-x_{k}^{(2)}
\\
x_{k-1}^{(1)}-x_{k-1}^{(2)}
\end{array}
\right)^{T}
P_{\alpha,\beta}
\left(
\begin{array}{c}
x_{k}^{(1)}-x_{k}^{(2)}
\\
x_{k-1}^{(1)}-x_{k-1}^{(2)}
\end{array}
\right)
+\frac{1}{2}\left(x_{k}^{(1)}-x_{k}^{(2)}\right)^{T}Q
\left(x_{k}^{(1)}-x_{k}^{(2)}\right)\right].
\end{align*}
Using induction on $k$, we get
\begin{align*}
&\left(
\begin{array}{c}
x_{k}^{(1)}-x_{k}^{(2)}
\\
x_{k-1}^{(1)}-x_{k-1}^{(2)}
\end{array}
\right)^{T}
P_{\alpha,\beta}
\left(
\begin{array}{c}
x_{k}^{(1)}-x_{k}^{(2)}
\\
x_{k-1}^{(1)}-x_{k-1}^{(2)}
\end{array}
\right)
+\frac{1}{2}\left(x_{k}^{(1)}-x_{k}^{(2)}\right)^{T}Q\left(x_{k}^{(1)}-x_{k}^{(2)}\right)
\\
&\leq
\rho_{\alpha,\beta}^{k}
\left[\left(
\begin{array}{c}
x_{0}^{(1)}-x_{0}^{(2)}
\\
x_{-1}^{(1)}-x_{-1}^{(2)}
\end{array}
\right)^{T}
P_{\alpha,\beta}
\left(
\begin{array}{c}
x_{0}^{(1)}-x_{0}^{(2)}
\\
x_{-1}^{(1)}-x_{-1}^{(2)}
\end{array}
\right)
+\frac{1}{2}\left(x_{0}^{(1)}-x_{0}^{(2)}\right)^{T}Q\left(x_{0}^{(1)}-x_{0}^{(2)}\right)\right].
\end{align*}
By taking expectation and since $\frac{1}{2}x^{T}Qx\geq 0$ for any $x$, 
we get
\begin{align*}
&
\mathbb{E}\left[\left(
\begin{array}{c}
x_{k}^{(1)}-x_{k}^{(2)}
\\
x_{k-1}^{(1)}-x_{k-1}^{(2)}
\end{array}
\right)^{T}
P_{\alpha,\beta}
\left(
\begin{array}{c}
x_{k}^{(1)}-x_{k}^{(2)}
\\
x_{k-1}^{(1)}-x_{k-1}^{(2)}
\end{array}
\right)\right]
\\
&\leq
\rho_{\alpha,\beta}^{k}
\mathbb{E}\bigg[\left(
\begin{array}{c}
x_{0}^{(1)}-x_{0}^{(2)}
\\
x_{-1}^{(1)}-x_{-1}^{(2)}
\end{array}
\right)^{T}
P_{\alpha,\beta}
\left(
\begin{array}{c}
x_{0}^{(1)}-x_{0}^{(2)}
\\
x_{-1}^{(1)}-x_{-1}^{(2)}
\end{array}
\right)
+\frac{1}{2}\left(x_{0}^{(1)}-x_{0}^{(2)}\right)^{T}Q
\left(x_{0}^{(1)}-x_{0}^{(2)}\right)\bigg].
\end{align*}

Let $\lambda_{1},\lambda_{2}\in\mathcal{P}_{2,S_{\alpha,\beta}}(\mathbb{R}^{2d})$.
There exist a couple of random vectors $(x_{0}^{(1)},x_{-1}^{(1)})$,
and $(x_{0}^{(2)},x_{-1}^{(2)})$,
independent of $(\varepsilon_{k})_{k=0}^{\infty}$ such that
\begin{align*}
\mathcal{W}_{2,S_{\alpha,\beta}}^{2}(\lambda_{1},\lambda_{2})
&=\mathbb{E}\bigg[\left(
\begin{array}{c}
x_{0}^{(1)}-x_{0}^{(2)}
\\
x_{-1}^{(1)}-x_{-1}^{(2)}
\end{array}
\right)^{T}
P_{\alpha,\beta}
\left(
\begin{array}{c}
x_{0}^{(1)}-x_{0}^{(2)}
\\
x_{-1}^{(1)}-x_{-1}^{(2)}
\end{array}
\right)
\\
&\qquad\qquad\qquad\qquad
+\frac{1}{2}\left(x_{0}^{(1)}-x_{0}^{(2)}\right)^{T}Q
\left(x_{0}^{(1)}-x_{0}^{(2)}\right)\bigg].
\end{align*}
Then, we get
\begin{equation*}
\mathcal{W}_{2,S_{\alpha,\beta}}^{2}\left(\mathcal{P}_{\alpha,\beta}^{k}\lambda_{1},\mathcal{P}_{\alpha,\beta}^{k}\lambda_{2}\right)
\leq
\rho_{\alpha,\beta}^{k}I^{2}(\lambda_{1},\lambda_{2}),
\end{equation*}
where
\begin{align*}
I^{2}(\lambda_{1},\lambda_{2})
&=
\mathbb{E}_{(x_{0}^{(j)},x_{-1}^{(j)})\sim\lambda_{j},j=1,2}
\Bigg[\left(
\begin{array}{c}
x_{0}^{(1)}-x_{0}^{(2)}
\\
x_{-1}^{(1)}-x_{-1}^{(2)}
\end{array}
\right)^{T}
P_{\alpha,\beta}
\left(
\begin{array}{c}
x_{0}^{(1)}-x_{0}^{(2)}
\\
x_{-1}^{(1)}-x_{-1}^{(2)}
\end{array}
\right)
\\
&\qquad\qquad\qquad
+\frac{1}{2}\left(x_{0}^{(1)}-x_{0}^{(2)}\right)^{T}Q\left(x_{0}^{(1)}-x_{0}^{(2)}\right)\Bigg].
\end{align*}
Therefore,
\begin{equation*}
\sum_{k=1}^{\infty}\mathcal{W}_{2,S_{\alpha,\beta}}^{2}\left(\mathcal{P}_{\alpha,\beta}^{k}\lambda_{1},\mathcal{P}_{\alpha,\beta}^{k}\lambda_{2}\right)<\infty.
\end{equation*}
By taking $\lambda_{2}=\mathcal{P}_{\alpha,\beta}\lambda_{1}$, we get
\begin{equation*}
\sum_{k=1}^{\infty}\mathcal{W}_{2,S_{\alpha,\beta}}^{2}\left(\mathcal{P}_{\alpha,\beta}^{k}\lambda_{1},\mathcal{P}_{\alpha,\beta}^{k+1}\lambda_{1}\right)<\infty.
\end{equation*}
Hence $\mathcal{P}_{\alpha,\beta}^{k}\lambda_{1}$ is a Cauchy sequence and converges
to a limit $\pi_{\alpha,\beta}^{\lambda_{1}}$:
\begin{equation*}
\lim_{k\rightarrow\infty}\mathcal{W}_{2,S_{\alpha,\beta}}\left(\mathcal{P}_{\alpha,\beta}^{k}\lambda_{1},\pi_{\alpha,\beta}^{\lambda_{1}}\right)=0.
\end{equation*}
Next, let us show that $\pi_{\alpha,\beta}^{\lambda_{1}}$ does not depend on $\lambda_{1}$.
Assume that there exists $\pi_{\alpha,\beta}^{\lambda_{2}}$ 
so that $\lim_{k\rightarrow\infty}\mathcal{W}_{2,S_{\alpha,\beta}}(\mathcal{P}_{\alpha,\beta}^{k}\lambda_{2},\pi_{\alpha,\beta}^{\lambda_{2}})=0$.
Since $\mathcal{W}_{2,S_{\alpha,\beta}}$ is a metric, by the triangle inequality,
\begin{align*}
\mathcal{W}_{2,S_{\alpha,\beta}}\left(\pi_{\alpha,\beta}^{\lambda_{1}},\pi_{\alpha,\beta}^{\lambda_{2}}\right)
&\leq
\mathcal{W}_{2,S_{\alpha,\beta}}\left(\pi_{\alpha,\beta}^{\lambda_{1}},\mathcal{P}_{\alpha,\beta}^{k}\lambda_{1}\right)
\\
&\qquad\qquad\qquad
+\mathcal{W}_{2,S_{\alpha,\beta}}\left(\mathcal{P}_{\alpha,\beta}^{k}\lambda_{1},\mathcal{P}_{\alpha,\beta}^{k}\lambda_{2}\right)
+\mathcal{W}_{2,S_{\alpha,\beta}}\left(\pi_{\alpha,\beta}^{\lambda_{2}},\mathcal{P}_{\alpha,\beta}^{k}\lambda_{2}\right),
\end{align*}
which goes to zero as $k\rightarrow\infty$. Hence, $\pi_{\alpha,\beta}^{\lambda_{1}}=\pi_{\alpha,\beta}^{\lambda_{2}}$.
The limit is therefore the same for any initial distributions
and we can denote it by $\pi_{\alpha,\beta}$.
Indeed,
\begin{equation*}
\mathcal{W}_{2,S_{\alpha,\beta}}\left(\mathcal{P}_{\alpha,\beta}\pi_{\alpha,\beta},\pi_{\alpha,\beta}\right)
\leq 
\mathcal{W}_{2,S_{\alpha,\beta}}\left(\mathcal{P}_{\alpha,\beta}\pi_{\alpha,\beta},\mathcal{P}_{\alpha,\beta}^{k}\pi_{\alpha,\beta}\right)
+\mathcal{W}_{2,S_{\alpha,\beta}}\left(\mathcal{P}_{\alpha,\beta}^{k}\pi_{\alpha,\beta},\pi_{\alpha,\beta}\right),
\end{equation*}
which goes to zero as $k\rightarrow\infty$. Hence $\mathcal{P}_{\alpha,\beta}\pi_{\alpha,\beta}=\pi_{\alpha,\beta}$
gives the invariant distribution. We can also show similarly as before 
that it is unique. 
\end{proof}

\begin{remark}
If $\alpha \in (0,1/L]$ and $\beta =  \frac{1-\sqrt{\alpha\mu}}{1+\sqrt{\alpha\mu}}$, then we can take the matrix $P_{\alpha,\beta}$ appearing in Theorem \ref{thm:alpha:beta} according to the $P_\alpha$ matrix defined in \cite[Theorem 2.3]{aybat2019universally} to obtain $\rho(\alpha,\beta)=1-\sqrt{\alpha\mu}$. For $\alpha =  \frac{\log^2(k)}{\mu k^2}$, then this leads to $\mathcal{W}_{2,S_{\alpha,\beta}}\left(\nu_{k,\alpha,\beta},\pi_{\alpha,\beta}\right) \leq \frac{1}{k} \mathcal{W}_{2,S_{\alpha,\beta}}(\nu_{0,\alpha,\beta},\pi_{\alpha,\beta})$ and it can be shown with an analysis similar to that of \cite{aybat2019universally} that the second moment of $\pi_{\alpha,\beta}$ is also $O(1/k)$; ignoring some logarithmic factors in $k$. Therefore, our results do not violate (and are in agreement with) the $\Omega(1/k)$ lower bounds studied in \cite{min-max,raginsky2011information,agarwal-minmax} for strongly convex stochastic optimization. 
\end{remark}

\begin{proof}[Proof of Theorem \ref{thm:AG-star-deterministic}]
First let us recall the AG method:
\begin{align*}
&x_{k+1}=y_{k}-\alpha[\nabla f(y_{k})],
\\
&y_{k}=(1+\beta)x_{k}-\beta x_{k-1},
\end{align*}
where $\alpha>0$ is the step size and $\beta$ is the momentum parameter.
In the case when $f$ is quadratic and
$f(x)=\frac{1}{2}x^{T}Qx+a^{T}x+b$, we can compute that
\begin{align*}
&x_{k+1}=y_{k}-\alpha[Qy_{k}+a],
\\
&y_{k}=(1+\beta)x_{k}-\beta x_{k-1},
\end{align*}
and with the optimizer $x_{*}$ we get
\begin{align*}
&x_{k+1}-x_{*}=y_{k}-x_{*}-\alpha[Q(y_{k}-x_{*})],
\\
&y_{k}-y_{*}=(1+\beta)(x_{k}-x_{*})-\beta (x_{k-1}-x_{*}),
\end{align*}
which implies that
\begin{equation*}
\left(
\begin{array}{c}
x_{k+1}-x_{*}
\\
x_{k}-x_{*}
\end{array}
\right)
=
\left(
\begin{array}{cc}
(1+\beta)(I_{d}-\alpha Q) & -\beta(I_{d}-\alpha Q)
\\
I_{d} & 0_{d}
\end{array}
\right)
\left(
\begin{array}{c}
x_{k}-x_{*}
\\
x_{k-1}-x_{*}
\end{array}
\right),
\end{equation*}
which yields that
\begin{equation*}
\left(
\begin{array}{c}
x_{k}-x_{*}
\\
x_{k-1}-x_{*}
\end{array}
\right)
=
\left(
\begin{array}{cc}
(1+\beta)(I_{d}-\alpha Q) & -\beta(I_{d}-\alpha Q)
\\
I_{d} & 0_{d}
\end{array}
\right)^{k}
\left(
\begin{array}{c}
x_{0}-x_{*}
\\
x_{-1}-x_{*}
\end{array}
\right),
\end{equation*}
and we aim to provide an upper bound to the 2-norm of the matrix, that is:
\begin{equation*}
\left\Vert\left(
\begin{array}{cc}
(1+\beta)(I_{d}-\alpha Q) & -\beta(I_{d}-\alpha Q)
\\
I_{d} & 0_{d}
\end{array}
\right)^{k}\right\Vert.
\end{equation*}

Let us assume that $Q$ has the decomposition
\begin{equation*}
Q=VDV^{T},
\end{equation*}
where $D$ is diagonal consisting
of eigenvalues $\lambda_{i}$, $1\leq i\leq d$
in increasing order:
\begin{equation*}
\mu=\lambda_{1}\leq\lambda_{2}\leq\cdots\leq\lambda_{d}=L,
\end{equation*}
then we have
\begin{equation*}
I_{d}-\alpha Q
=V\tilde{D}V^{T},
\end{equation*}
where $\tilde{D}=I_{d}-\alpha D$ is diagonal matrix with entries 
\begin{equation*}   
1-\alpha\lambda_{i},\qquad 1\leq i\leq d.
\end{equation*}
Therefore, the matrix
\begin{equation*}
\left(
\begin{array}{cc}
(1+\beta)(I_{d}-\alpha Q) & -\beta (I_{d}-\alpha Q)
\\
I_{d} & 0_{d}
\end{array}
\right)
\end{equation*}
has the same eigenvalues as the matrix
\begin{equation*}
\left(
\begin{array}{cc}
(1+\beta)(I_{d}-\alpha D) & -\beta (I_{d}-\alpha D)
\\
I_{d} & 0_{d}
\end{array}
\right),
\end{equation*}
which has the same eigenvalues as the matrix:
\begin{equation*}
\left(
\begin{array}{cccc}
T_{1} & \cdots & 0 & 0
\\
0 & T_{2} & \cdots & 0
\\
\vdots & \cdots & \ddots & \vdots 
\\ 
0 & 0 & \cdots & T_{d}
\end{array}
\right),
\end{equation*}
where 
\begin{equation*}
T_{i}=\left(
\begin{array}{cc}
(1+\beta)(1-\alpha\lambda_{i}) & -\beta(1-\alpha\lambda_{i})
\\
1 & 0
\end{array}
\right),
\qquad 1\leq i\leq d,
\end{equation*}
are $2\times 2$ matrices with eigenvalues:
\begin{equation*}
\mu_{i,\pm}=\frac{(1+\beta)(1-\alpha\lambda_{i})
\pm\sqrt{(1+\beta)^{2}(1-\alpha\lambda_{i})^{2}-4\beta(1-\alpha\lambda_{i})}}{2},
\end{equation*}
where $1\leq i\leq d$, and therefore
\begin{equation}\label{Ti:ineq:2}
\left\Vert\left(
\begin{array}{cc}
(1+\beta)(I_{d}-\alpha Q) & -\beta (I_{d}-\alpha Q)
\\
I_{d} & 0_{d}
\end{array}
\right)^{k}\right\Vert
\leq
\max_{1\leq i\leq d}\left\Vert T_{i}^{k}\right\Vert.
\end{equation}
Next, we upper bound $\Vert T_{i}^{k}\Vert$.
We recall the choice:
\begin{equation}
\alpha=\frac{4}{3L+\mu},
\qquad
\beta=\frac{\sqrt{3\kappa+1}-2}{\sqrt{3\kappa+1}+2},
\qquad
\rho=1-\frac{2}{\sqrt{3\kappa+1}}.
\end{equation}

We can compute that
\begin{equation}
\Delta_{i}:=(1+\beta)^{2}(1-\alpha\lambda_{i})^{2}-4\beta(1-\alpha\lambda_{i})
=16\frac{(1-\alpha\lambda_{i})}{(\sqrt{3\kappa+1}+2)^{2}}
\left(1-\frac{\lambda_{i}}{\mu}\right).
\end{equation}
Therefore $\Delta_{i}=0$
if and only if $\lambda_{i}=\mu$
or $\lambda_{i}=\frac{3L+\mu}{4}$,
and moreover $\Delta_{i}<0$
for $\mu<\lambda_{i}<\frac{3L+\mu}{4}$
and $\Delta_{i}>0$
for $\lambda_{i}>\frac{3L+\mu}{4}$.

(1) Consider the case $\mu<\lambda_{i}<\frac{3L+\mu}{4}$.
Then $\Delta_{i}<0$.
It is known that the $k$-th power 
of a $2\times 2$ matrix $A$ with distinct eigenvalues
$\mu_{\pm}$ is given by
\begin{equation*}
A^{k}=\frac{\mu_{+}^{k}}{\mu_{+}-\mu_{-}}(A-\mu_{-}I)
+\frac{\mu_{-}^{k}}{\mu_{-}-\mu_{+}}(A-\mu_{+}I),
\end{equation*}
where $I$ is the $2\times 2$ identity matrix \cite{williams2by2}.
In our context, $A=T_{i}$ and $\mu_{\pm}=\mu_{i,\pm}$, 
we get
\begin{equation}\label{Ti:1:AG}
T_{i}^{k}
=\frac{\mu_{i,+}^{k}}{\mu_{i,+}-\mu_{i,-}}(T_{i}-\mu_{i,-}I)
+\frac{\mu_{i,-}^{k}}{\mu_{i,-}-\mu_{i,+}}(T_{i}-\mu_{i,+}I).
\end{equation}
We can compute that
\begin{align}
|\mu_{i,+}|
=|\mu_{i,-}|
=\left(\beta(1-\alpha\lambda_{i})\right)^{1/2}
&=\left(\frac{\sqrt{3\kappa+1}-2}{\sqrt{3\kappa+1}+2}
\frac{3L+\mu-4\lambda_{i}}{3L+\mu}\right)^{1/2}
\\
&\leq\left(\frac{\sqrt{3\kappa+1}-2}{\sqrt{3\kappa+1}+2}
\frac{3\kappa-3}{3\kappa+1}\right)^{1/2},
\nonumber
\end{align}
and notice that
\begin{equation}
3\kappa-3=\left(\sqrt{3\kappa+1}+2\right)\left(\sqrt{3\kappa+1}-2\right),
\end{equation}
and thus we get
\begin{equation}\label{Ti:2:AG}
|\mu_{i,+}|
=|\mu_{i,-}|
\leq\left(\frac{(\sqrt{3\kappa+1}-2)^{2}}{3\kappa+1}\right)^{1/2}
=1-\frac{2}{\sqrt{3\kappa+1}}=\rho.
\end{equation}
Moreover,
\begin{equation}\label{Ti:3:AG}
\frac{1}{|\mu_{i,+}-\mu_{i,-}|}
=\frac{1}{\sqrt{|\Delta_{i}|}}
\leq\frac{\sqrt{3\kappa+1}+2}{4}
\max_{i:\mu<\lambda_{i}<\frac{3L+\mu}{4}}
\frac{\sqrt{\mu}}{\sqrt{(\lambda_{i}-\mu)(1-\frac{4\lambda_{i}}{3L+\mu})}}.
\end{equation}
Furthermore, 
\begin{equation*}
T_{i}-\mu_{i,-}I
=\left(
\begin{array}{cc}
\mu_{i,+} & -\beta(1-\alpha\lambda_{i})
\\
1 & -\mu_{i,-}
\end{array}
\right)
=\left(\begin{array}{c}
\mu_{i,+}
\\
1
\end{array}\right)
\left(\begin{array}{cc}
1 & -\mu_{i,-}
\end{array}\right),
\end{equation*}
and
\begin{equation*}
T_{i}-\mu_{i,+}I
=\left(
\begin{array}{cc}
\mu_{i,-} & -\beta(1-\alpha\lambda_{i})
\\
1 & -\mu_{i,+}
\end{array}
\right)
=\left(\begin{array}{c}
\mu_{i,-} 
\\
1
\end{array}\right)
\left(\begin{array}{cc}
1 & -\mu_{i,+}
\end{array}\right).
\end{equation*}
Therefore,
\begin{equation}\label{Ti:4:AG}
\left\Vert T_{i}-\mu_{i,-}I\right\Vert
\leq
\left\Vert\left(\begin{array}{c}
\mu_{i,+}
\\
1
\end{array}\right)\right\Vert
\left\Vert\left(\begin{array}{cc}
1 & -\mu_{i,-}
\end{array}\right)\right\Vert
=\rho^{2}+1,
\end{equation}
and
\begin{equation}\label{Ti:5:AG}
\left\Vert T_{i}-\mu_{i,+}I\right\Vert
\leq\left\Vert\left(\begin{array}{c}
\mu_{i,-}
\\
1
\end{array}\right)\right\Vert
\left\Vert\left(\begin{array}{cc}
1 & -\mu_{i,+}
\end{array}\right)\right\Vert
=\rho^{2}+1.
\end{equation}
Hence, it follows from 
\eqref{Ti:1:AG}, \eqref{Ti:2:AG}, \eqref{Ti:3:AG}, \eqref{Ti:4:AG}
and \eqref{Ti:5:AG} that
\begin{equation*}
\left\Vert T_{i}^{k}\right\Vert
\leq\frac{\sqrt{3\kappa+1}+2}{2}
\max_{i:\mu<\lambda_{i}<\frac{3L+\mu}{4}}
\frac{\sqrt{\mu}}{\sqrt{(\lambda_{i}-\mu)(1-\frac{4\lambda_{i}}{3L+\mu})}}\rho^{k}(\rho^{2}+1).
\end{equation*}

(2) Consider the case $\frac{3L+\mu}{4}<\lambda_{i}<L$.
Then, $\Delta_{i}>0$.
As before, we have
\begin{equation}\label{Ti:1:AG:2}
T_{i}^{k}
=\frac{\mu_{i,+}^{k}}{\mu_{i,+}-\mu_{i,-}}(T_{i}-\mu_{i,-}I)
+\frac{\mu_{i,-}^{k}}{\mu_{i,-}-\mu_{i,+}}(T_{i}-\mu_{i,+}I).
\end{equation}
We can compute that
\begin{align}\label{Ti:2:AG:2}
|\mu_{i,+}|
&\leq|\mu_{i,-}|
=\frac{1}{2}(1+\beta)(\alpha\lambda_{i}-1)+\frac{1}{2}\sqrt{\Delta_{i}}
\\
&\leq
\frac{1}{2}(1+\beta)(\alpha L-1)+\frac{1}{2}\sqrt{16\frac{(\alpha L-1)}{(\sqrt{3\kappa+1}+2)^{2}}
\frac{L-\mu}{\mu}}
\nonumber
\\
&=\frac{\sqrt{3\kappa+1}}{\sqrt{3\kappa+1}+2}\frac{\kappa-1}{3\kappa+1}
+\frac{1}{2}\sqrt{16\frac{\kappa-1}{(\sqrt{3\kappa+1}+2)^{2}}
\frac{\kappa-1}{3\kappa+1}}=1-\frac{2}{\sqrt{3\kappa+1}}=\rho.
\nonumber
\end{align}
Moreover,
\begin{equation}\label{Ti:3:AG:2}
\frac{1}{|\mu_{i,+}-\mu_{i,-}|}
=\frac{1}{\sqrt{\Delta_{i}}}
\leq\frac{\sqrt{3\kappa+1}+2}{4}
\max_{i:\frac{3L+\mu}{4}<\lambda_{i}<L}
\frac{\sqrt{\mu}}{\sqrt{(\lambda_{i}-\mu)(\frac{4\lambda_{i}}{3L+\mu}-1)}}.
\end{equation}
Furthermore, 
\begin{equation*}
T_{i}-\mu_{i,-}I
=\left(
\begin{array}{cc}
\mu_{i,+} & -\beta(1-\alpha\lambda_{i})
\\
1 & -\mu_{i,-}
\end{array}
\right)
=\left(\begin{array}{c}
\mu_{i,+}
\\
1
\end{array}\right)
\left(\begin{array}{cc}
1 & -\mu_{i,-}
\end{array}\right),
\end{equation*}
and
\begin{equation*}
T_{i}-\mu_{i,+}I
=\left(
\begin{array}{cc}
\mu_{i,-} & -\beta(1-\alpha\lambda_{i})
\\
1 & -\mu_{i,+}
\end{array}
\right)
=\left(\begin{array}{c}
\mu_{i,-} 
\\
1
\end{array}\right)
\left(\begin{array}{cc}
1 & -\mu_{i,+}
\end{array}\right).
\end{equation*}
Therefore,
\begin{equation}\label{Ti:4:AG:2}
\left\Vert T_{i}-\mu_{i,-}I\right\Vert
\leq
\left\Vert\left(\begin{array}{c}
\mu_{i,+}
\\
1
\end{array}\right)\right\Vert
\left\Vert\left(\begin{array}{cc}
1 & -\mu_{i,-}
\end{array}\right)\right\Vert
\leq\rho^{2}+1,
\end{equation}
and
\begin{equation}\label{Ti:5:AG:2}
\left\Vert T_{i}-\mu_{i,+}I\right\Vert
\leq\left\Vert\left(\begin{array}{c}
\mu_{i,-}
\\
1
\end{array}\right)\right\Vert
\left\Vert\left(\begin{array}{cc}
1 & -\mu_{i,+}
\end{array}\right)\right\Vert
\leq\rho^{2}+1.
\end{equation}
Hence, it follows from 
\eqref{Ti:1:AG:2}, \eqref{Ti:2:AG:2}, \eqref{Ti:3:AG:2}, \eqref{Ti:4:AG:2}
and \eqref{Ti:5:AG:2} that
\begin{equation*}
\left\Vert T_{i}^{k}\right\Vert
\leq\frac{\sqrt{3\kappa+1}+2}{2}
\max_{i:\frac{3L+\mu}{4}<\lambda_{i}<L}
\frac{\sqrt{\mu}}{\sqrt{(\lambda_{i}-\mu)(\frac{4\lambda_{i}}{3L+\mu}-1)}}\rho^{k}(\rho^{2}+1).
\end{equation*}

(3) Consider the case $\lambda_{i}=\mu$. 
Then $\Delta_{i}=0$.
It is known that the $k$-th power 
of a $2\times 2$ matrix $A$ with two equal eigenvalues
$\mu_{+}=\mu_{-}=\mu$ is given by
\begin{equation*}
A^{k}=\mu^{k-1}(kA-(k-1)\mu I),
\end{equation*}
where $I$ is the $2\times 2$ identity matrix \cite{williams2by2}.
In our context, $A=T_{i}$ and
\begin{equation}
\mu=\mu_{\pm}=\mu_{i,\pm}
=\frac{1}{2}(1+\beta)(1-\alpha\lambda_{i})=1-\frac{2}{\sqrt{3\kappa+1}}=\rho.
\end{equation}
Therefore, with $\lambda_{i}=\mu$, we have
\begin{align*}
T_{i}^{k}&=\rho^{k}(kT_{i}-(k-1)\rho I)
\\
&=\rho^{k}\left(\begin{array}{cc}
k(1+\beta)(1-\alpha\lambda_{i})-(k-1)\rho & -k\beta(1-\alpha\lambda_{i})
\\
k & -(k-1)\rho
\end{array}
\right)
\\
&=\left(\begin{array}{cc}
(k+1)\rho & -k\rho^{2}
\\
k & -(k-1)\rho
\end{array}\right),
\end{align*}
and therefore
\begin{align}
\Vert T_{i}^{k}\Vert
&\leq\sqrt{\text{Tr}\left(T_{i}^{k}(T_{i}^{k})^{T}\right)}
\\
&=\rho^{k}\left((k+1)^{2}\rho^{2}+(k-1)^{2}\rho^{2}+k^{2}\rho^{4}+k^{2}\right)^{1/2}
\\
&=\rho^{k}\sqrt{k^{2}(\rho^{2}+1)^{2}+2\rho^{2}}.\label{ineq-tik-upper-bound}
\end{align}
Furthermore, we see that the sequence $T_i^k/k$ converges to a non-zero matrix. Therefore, $\| T_i^k\| \geq ck$ for some constant $c$ for every $k$. This means that the linear dependency to $k$ of our upper bound in \eqref{ineq-tik-upper-bound} is tight. This behavior is expected due to the fact that $T_i^k$ has double roots.

(4) Consider the case $\lambda_{i}=\frac{3L+\mu}{4}$. 
Then $\Delta_{i}=0$. We can compute that
\begin{equation}
\mu_{i,\pm}
=\frac{1}{2}(1+\beta)(1-\alpha\lambda_{i})=1-\frac{2}{\sqrt{3\kappa+1}}=0.
\end{equation}
In this case, $T_{i}=0$.

Finally, combining the three cases (1) $\mu<\lambda_{i}<\frac{3L+\mu}{4}$;
(2) $\lambda_{i}>\frac{3L+\mu}{4}$; (3) $\lambda_{i}=\mu$; (4) $\lambda_{i}=\frac{3L+\mu}{4}$, 
and recall \eqref{Ti:ineq:2}, we get
\begin{align*}
&\left\Vert\left(
\begin{array}{cc}
(1+\beta)(I_{d}-\alpha Q) & -\beta (I_{d}-\alpha Q)
\\
I_{d} & 0_{d}
\end{array}
\right)^{k}\right\Vert
\\
&\leq
\max_{1\leq i\leq d}\left\Vert T_{i}^{k}\right\Vert
\\
&\leq
\rho^{k}\max\left\{\frac{\sqrt{3\kappa+1}+2}{2}(\rho^{2}+1)
\max_{i:\mu<\lambda_{i}\neq\frac{3L+\mu}{4}}
\frac{\sqrt{\mu}}{\sqrt{(\lambda_{i}-\mu)|1-\frac{4\lambda_{i}}{3L+\mu}|}},
\sqrt{k^{2}(\rho^{2}+1)^{2}+2\rho^{2}}\right\}.
\end{align*}
The proof is complete.
\end{proof}

\begin{proof}[Proof of Theorem \ref{thm:AG-star}]
First let us recall the ASG method:
\begin{align*}
&x_{k+1}=y_{k}-\alpha[\nabla f(y_{k})+\varepsilon_{k+1}],
\\
&y_{k}=(1+\beta)x_{k}-\beta x_{k-1},
\end{align*}
where $\alpha>0$ is the step size and $\beta$ is the momentum parameter.
In the case when $f$ is quadratic and
$f(x)=\frac{1}{2}x^{T}Qx+a^{T}x+b$, we can compute that
\begin{align*}
&x_{k+1}=y_{k}-\alpha[Qy_{k}+a+\varepsilon_{k+1}],
\\
&y_{k}=(1+\beta)x_{k}-\beta x_{k-1},
\end{align*}
so that with two couplings $x_{k}^{(1)},x_{k}^{(2)}$:
\begin{align*}
&x_{k+1}^{(j)}=y_{k}^{(j)}-\alpha\left[Qy_{k}^{(j)}+a+\varepsilon_{k+1}\right],
\\
&y_{k}^{(j)}=(1+\beta)x_{k}^{(j)}-\beta x_{k-1}^{(j)},
\end{align*}
with $j=1,2$, we get
\begin{align*}
&x_{k+1}^{(1)}-x_{k+1}^{(2)}=y_{k}^{(1)}-y_{k}^{(2)}-\alpha Q\left(y_{k}^{(1)}-y_{k}^{(2)}\right),
\\
&y_{k}^{(1)}-y_{k}^{(2)}=(1+\beta)(x_{k}^{(1)}-x_{k}^{(2)})-\beta (x_{k-1}^{(1)}-x_{k-1}^{(2)}),
\end{align*}
which implies that
\begin{equation*}
\left(
\begin{array}{c}
x_{k+1}^{(1)}-x_{k+1}^{(2)}
\\
x_{k}^{(1)}-x_{k}^{(2)}
\end{array}
\right)
=
\left(
\begin{array}{cc}
(1+\beta)(I_{d}-\alpha Q) & -\beta(I_{d}-\alpha Q)
\\
I_{d} & 0_{d}
\end{array}
\right)
\left(
\begin{array}{c}
x_{k}^{(1)}-x_{k}^{(2)}
\\
x_{k-1}^{(1)}-x_{k-1}^{(2)}
\end{array}
\right),
\end{equation*}
which yields that
\begin{equation*}
\left\Vert\left(
\begin{array}{c}
x_{k}^{(1)}-x_{k}^{(2)}
\\
x_{k-1}^{(1)}-x_{k-1}^{(2)}
\end{array}
\right)\right\Vert
\leq\left\Vert
\left(
\begin{array}{cc}
(1+\beta)(I_{d}-\alpha Q) & -\beta(I_{d}-\alpha Q)
\\
I_{d} & 0_{d}
\end{array}
\right)^{k}\right\Vert
\left\Vert\left(
\begin{array}{c}
x_{0}^{(1)}-x_{0}^{(2)}
\\
x_{-1}^{(1)}-x_{-1}^{(2)}
\end{array}
\right)\right\Vert.
\end{equation*}
Following from the proof of Theorem \ref{thm:alpha:beta}, 
we can show by constructing a Cauchy sequence that 
there exists a unique stationary distribution $\pi_{\alpha,\beta}$.
Finally, we assume that 
$(x_{0}^{(1)},x_{-1}^{(1)})$ starts
from the given $(x_{0},x_{-1})$ distributed
as $\nu_{0,\alpha,\beta}$
and $(x_{0}^{(2)},x_{-1}^{(2)})$ starts
from the stationary distribution $\pi_{\alpha,\beta}$
so that their $L_{p}$ distance is exactly
the $\mathcal{W}_{p}$ distance. 
Then we get
\begin{equation*}
\mathcal{W}_{p}^{p}\left(\nu_{k,\alpha,\beta},
\pi_{\alpha,\beta}\right)
\leq
\mathbb{E}\left\Vert
\left(
\begin{array}{c}
x_{k}^{(1)}-x_{k}^{(2)}
\\
x_{k-1}^{(1)}-x_{k-1}^{(2)}
\end{array}
\right)
\right\Vert^{p}
\leq
(C_{k}^{*})^{p}(\rho_{AG}^{*})^{pk}
\mathcal{W}_{p}^{p}\left(\nu_{0,\alpha,\beta},
\pi_{\alpha,\beta}\right),
\nonumber
\end{equation*}
and the proof is complete 
by taking the power $1/p$ in the above equation.
\end{proof}

Before we state the proof of Theorem \ref{thm:f:AG-star},
let us spell out $X$ and $V_{AG}^{*}(\xi_{0})$ in
the statement of Theorem \ref{thm:f:AG-star} explicitly here.
We will show that Theorem \ref{thm:f:AG-star} holds
with $V_{AG}^{*}(\xi_{0})$ given by
\begin{equation*}
V_{AG}^{*}(\xi_{0}):=\mathbb{E}\left[\left\Vert(\xi_{0}-\xi_{\ast})(\xi_{0}-\xi_{\ast})^{T}\right\Vert\right]
+\frac{(\alpha_{AG}^{*})^{2}\Vert\Sigma\Vert}{1-(\rho_{AG}^{*})^{2}},
\end{equation*}
where $\Sigma:=\mathbb{E}[\varepsilon_{k}\varepsilon_{k}^{T}]$ and
$X_{AG}^{*}=\mathbb{E}[(\xi_{\infty}-\xi_{\ast})(\xi_{\infty}-\xi_{\ast})^{T}]$ 
satisfies the discrete Lyapunov equation:
\begin{equation*}
X_{AG}^{*}=A_{Q}^{*}X_{AG}^{*}(A_{Q}^{*})^{T}+\left(\begin{array}{cc}
(\alpha_{AG}^{*})^{2}\Sigma & 0_{d}
\\
0_{d} & 0_{d}
\end{array}\right),
\end{equation*}
and
\begin{equation*}
A_{Q}^{*}:=\left(
\begin{array}{cc}
(1+\beta_{AG}^{*})(I_{d}-\alpha_{AG}^{*} Q) & -\beta_{AG}^{*} (I_{d}-\alpha_{AG}^{*} Q)
\\
I_{d} & 0_{d}
\end{array}
\right).
\end{equation*}

In the special case $\Sigma = c^2 I_d $ for some constant $c\geq 0$, 
it follows from \cite{StrConvex} that
\begin{equation}\label{eq-trace-cov-AG-star}
\text{Tr}(X_{AG}^{*}) =c^{2}\sum_{i=1}^d  \frac{\alpha_{AG}^{*}}{\lambda_{i}(1-\beta_{AG}^{*}(1-\alpha_{AG}^{*}\lambda_{i}))},
\end{equation}
where $\{\lambda_i\}_{i=1}^d$ are the eigenvalues of $Q$.

Now, we are ready to prove Theorem \ref{thm:f:AG-star}.

\begin{proof}[Proof of Theorem \ref{thm:f:AG-star}]
For the ASG method,
\begin{equation*}
x_{k+1}=(1+\beta)x_{k}-\beta x_{k-1}
-\alpha(\nabla f((1+\beta)x_{k}-\beta x_{k-1})+\varepsilon_{k+1}),
\end{equation*}
where we consider the quadratic objective 
$f(x)=\frac{1}{2}x^{T}Qx+a^{T}x+b$ so that
\begin{equation*}
x_{k+1}=(1+\beta)x_{k}-\beta x_{k-1}
-\alpha(Q((1+\beta)x_{k}-\beta x_{k-1})+a+\varepsilon_{k+1}),
\end{equation*}
and the minimizer $x_{\ast}$ satisfies:
\begin{equation*}
x_{*}=(1+\beta)x_{*}-\beta x_{*}
-\alpha(Q((1+\beta)x_{*}-\beta x_{*})+a),
\end{equation*}
so that
\begin{equation*}
x_{k+1}-x_{*}=(1+\beta)(x_{k}-x_{*})-\beta (x_{k-1}-x_{*})
-\alpha(Q((1+\beta)(x_{k}-x_{*})-\beta (x_{k-1}-x_{*}))+\varepsilon_{k+1}),
\end{equation*}
and
\begin{equation*}
\left(
\begin{array}{c}
x_{k}-x_{\ast}
\\
x_{k-1}-x_{\ast}
\end{array}
\right)
=
\left(
\begin{array}{cc}
(1+\beta)(I_{d}-\alpha Q) & -\beta (I_{d}-\alpha Q)
\\
I_{d} & 0_{d}
\end{array}
\right)
\left(
\begin{array}{c}
x_{k-1}-x_{\ast}
\\
x_{k-2}-x_{\ast}
\end{array}
\right)
+
\left(\begin{array}{c}
-\alpha\varepsilon_{k}
\\
0_{d}
\end{array}\right),
\end{equation*}
and with $\Sigma:=\mathbb{E}[\varepsilon_{k}\varepsilon_{k}^{T}]$, we get
\begin{equation}\label{eqn:iterate:k}
\mathbb{E}\left[(\xi_{k}-\xi_{\ast})(\xi_{k}-\xi_{\ast})^{T}\right]
=
A_{Q}^{*}\mathbb{E}\left[(\xi_{k-1}-x_{\ast})(\xi_{k-1}-x_{\ast})^{T}\right](A_{Q}^{*})^{T}
+
\left(\begin{array}{cc}
\alpha^{2}\Sigma & 0_{d}
\\
0_{d} & 0_{d}
\end{array}\right),
\end{equation}
where
\begin{equation*}
A_{Q}^{*}=\left(
\begin{array}{cc}
(1+\beta)(I_{d}-\alpha Q) & -\beta (I_{d}-\alpha Q)
\\
I_{d} & 0_{d}
\end{array}
\right).
\end{equation*}
Therefore, 
\begin{equation*}
X=\mathbb{E}\left[(\xi_{\infty}-\xi_{\ast})(\xi_{\infty}-\xi_{\ast})^{T}\right]
\end{equation*}
satisfies the discrete Lyapunov equation:
\begin{equation*}
X=A_{Q}^{*}X(A_{Q}^{*})^{T}+\left(\begin{array}{cc}
\alpha^{2}\Sigma & 0_{d}
\\
0_{d} & 0_{d}
\end{array}\right).
\end{equation*}

Next by iterating equation \eqref{eqn:iterate:k} over $k$, we immediately obtain
\begin{align*}
\mathbb{E}\left[(\xi_{k}-\xi_{\ast})(\xi_{k}-\xi_{\ast})^{T}\right]
&=\left(A_{Q}^{*}\right)^{k}\mathbb{E}\left[(\xi_{0}-\xi_{\ast})(\xi_{0}-\xi_{\ast})^{T}\right]\left((A_{Q}^{*})^{T}\right)^{k}
\\
&\qquad\qquad\qquad
+\sum_{j=0}^{k-1}
\left(A_{Q}^{*}\right)^{j}\left(\begin{array}{cc}
\alpha^{2}\Sigma & 0_{d}
\\
0_{d} & 0_{d}
\end{array}\right)\left((A_{Q}^{*})^{T}\right)^{j},
\end{align*}
so that
\begin{align*}
&\mathbb{E}\left[(\xi_{k}-\xi_{\ast})(\xi_{k}-\xi_{\ast})^{T}\right]
\\
&=\mathbb{E}\left[(\xi_{\infty}-\xi_{\ast})(\xi_{\infty}-\xi_{\ast})^{T}\right]
+\left(A_{Q}^{*}\right)^{k}\mathbb{E}\left[(\xi_{0}-\xi_{\ast})(\xi_{0}-\xi_{\ast})^{T}\right]\left((A_{Q}^{*})^{T}\right)^{k}
\\
&\qquad\qquad\qquad
-\sum_{j=k}^{\infty}
\left(A_{Q}^{*}\right)^{j}\left(\begin{array}{cc}
\alpha^{2}\Sigma & 0_{d}
\\
0_{d} & 0_{d}
\end{array}\right)\left((A_{Q}^{*})^{T}\right)^{j},
\end{align*}
which implies that
\begin{align*}
&\text{Tr}\left(\mathbb{E}\left[(\xi_{k}-\xi_{\ast})(\xi_{k}-\xi_{\ast})^{T}\right]\right)
\\
&=\text{Tr}\left(\mathbb{E}\left[(\xi_{\infty}-\xi_{\ast})(\xi_{\infty}-\xi_{\ast})^{T}\right]\right)
+\left(A_{Q}^{*}\right)^{k}\mathbb{E}\left[(\xi_{0}-\xi_{\ast})(\xi_{0}-\xi_{\ast})^{T}\right]\left((A_{Q}^{*})^{T}\right)^{k}
\\
&\qquad\qquad
-\sum_{j=k}^{\infty}
\left(A_{Q}^{*}\right)^{j}\left(\begin{array}{cc}
\alpha^{2}\Sigma & 0_{d}
\\
0_{d} & 0_{d}
\end{array}\right)\left((A_{Q}^{*})^{T}\right)^{j}
\\
&\leq
\text{Tr}(X)
+\left\Vert (A_{Q}^{*})^{k}\right\Vert^{2}
\mathbb{E}\left[\left\Vert(\xi_{0}-\xi_{\ast})(\xi_{0}-\xi_{\ast})^{T}\right\Vert\right]
+\sum_{j=k}^{\infty}\left\Vert (A_{Q}^{*})^{j}\right\Vert^{2}\alpha^{2}\Vert\Sigma\Vert
\\
&\leq\text{Tr}(X)
+(C_{k}^{*})^{2}(\rho_{AG}^{*})^{2k}\mathbb{E}\left[\left\Vert(\xi_{0}-\xi_{\ast})(\xi_{0}-\xi_{\ast})^{T}\right\Vert\right]
+\alpha^{2}\Vert\Sigma\Vert (C_{k}^{*})^{2}\frac{(\rho_{AG}^{*})^{2k}}{1-(\rho_{AG}^{*})^{2}},
\end{align*}
where we used the estimate $\Vert (A_{Q}^{*})^{k}\Vert\leq C_{k}^{*}(\rho_{AG}^{*})^{k}$
from the proof of Theorem \ref{thm:AG-star-deterministic}.

Finally, since $\nabla f$ is $L$-Lipschtiz,
\begin{equation*}
\mathbb{E}[f(x_{k})]-f(x_{\ast})
\leq\frac{L}{2}\mathbb{E}\Vert x_{k}-x_{\ast}\Vert^{2}
\leq\frac{L}{2}\mathbb{E}\Vert \xi_{k}-\xi_{\ast}\Vert^{2}
=\frac{L}{2}\text{Tr}\left(\mathbb{E}\left[(\xi_{k}-\xi_{\ast})(\xi_{k}-\xi_{\ast})^{T}\right]\right).
\end{equation*}
The proof of \eqref{ineq-AG-star} is complete. 
\end{proof}

\begin{remark}
Note that our results in $p$-Wasserstein distances would hold if there exists some $p\geq 1$ so that $p$-th moment of the noise is finite. For instance, the $p<2$ case can arise in applications where the noise has heavy tail (see e.g. \cite{simsekli2019tail}).
\end{remark}

\subsection{Proofs of Results in Section \ref{sec:HB:quadratic}}

\begin{proof}[Proof of Theorem \ref{thm:HB-deterministic}]
First let us recall the HB method:
\begin{equation*}
x_{k+1}=x_{k}-\alpha\nabla f(x_{k})+\beta(x_{k}-x_{k-1}),
\end{equation*}
where $\alpha>0$ is the step size and $\beta$ is the momentum parameter.
In the case when $f$ is quadratic and
$f(x)=\frac{1}{2}x^{T}Qx+a^{T}x+b$, we can compute that
\begin{equation*}
x_{k+1}=x_{k}-\alpha(Qx_{k}+a)+\beta(x_{k}-x_{k-1}),
\end{equation*}
and the minimizer $x_{\ast}$ satisfies
\begin{equation*}
x_{\ast}=x_{\ast}-\alpha(Qx_{\ast}+a)+\beta(x_{\ast}-x_{\ast}),
\end{equation*}
which implies that
\begin{equation*}
\left(
\begin{array}{c}
x_{k+1}-x_{\ast}
\\
x_{k}-x_{\ast}
\end{array}
\right)
=
\left(
\begin{array}{cc}
(1+\beta)I_{d}-\alpha Q & -\beta I_{d}
\\
I_{d} & 0_{d}
\end{array}
\right)
\left(
\begin{array}{c}
x_{k}-x_{\ast}
\\
x_{k-1}-x_{\ast}
\end{array}
\right),
\end{equation*}
which yields that
\begin{equation*}
\left(
\begin{array}{c}
x_{k}-x_{\ast}
\\
x_{k-1}-x_{\ast}
\end{array}
\right)
=
\left(
\begin{array}{cc}
(1+\beta)I_{d}-\alpha Q & -\beta I_{d}
\\
I_{d} & 0_{d}
\end{array}
\right)^{k}
\left(
\begin{array}{c}
x_{0}-x_{\ast}
\\
x_{-1}-x_{\ast}
\end{array}
\right),
\end{equation*}
and we aim to provide an upper bound to the 2-norm of the matrix, that is:
\begin{equation*}
\left\Vert\left(
\begin{array}{cc}
(1+\beta)I_{d}-\alpha Q & -\beta I_{d}
\\
I_{d} & 0_{d}
\end{array}
\right)^{k}\right\Vert.
\end{equation*}

Let us assume that $Q$ has the decomposition
\begin{equation*}
Q=VDV^{T},
\end{equation*}
where $D$ is diagonal consisting
of eigenvalues $\lambda_{i}$, $1\leq i\leq d$
in increasing order:
\begin{equation*}
\mu=\lambda_{1}\leq\lambda_{2}\leq\cdots\leq\lambda_{d}=L,
\end{equation*}
then we have
\begin{equation*}
(1+\beta)I_{d}-\alpha Q
=V\tilde{D}V^{T},
\end{equation*}
where $\tilde{D}=(1+\beta)I_{d}-\alpha D$ is diagonal matrix with entries 
\begin{equation*}   
1+\beta-\alpha\lambda_{i},\qquad 1\leq i\leq d.
\end{equation*}
Therefore, the matrix
\begin{equation*}
\left(
\begin{array}{cc}
(1+\beta)I_{d}-\alpha Q & -\beta I_{d}
\\
I_{d} & 0_{d}
\end{array}
\right)
\end{equation*}
has the same eigenvalues as the matrix
\begin{equation*}
\left(
\begin{array}{cc}
(1+\beta)I_{d}-\alpha D & -\beta I_{d}
\\
I_{d} & 0_{d}
\end{array}
\right),
\end{equation*}
which has the same eigenvalues as the matrix:
\begin{equation*}
\left(
\begin{array}{cccc}
T_{1} & \cdots & 0 & 0
\\
0 & T_{2} & \cdots & 0
\\
\vdots & \cdots & \ddots & \vdots 
\\ 
0 & 0 & \cdots & T_{d}
\end{array}
\right),
\end{equation*}
where 
\begin{equation*}
T_{i}=\left(
\begin{array}{cc}
1+\beta-\alpha\lambda_{i} & -\beta
\\
1 & 0
\end{array}
\right),
\qquad 1\leq i\leq d,
\end{equation*}
are $2\times 2$ matrices with eigenvalues:
\begin{equation*}
\mu_{i,\pm}=\frac{1+\beta-\alpha\lambda_{i}\pm\sqrt{(1+\beta-\alpha\lambda_{i})^{2}-4\beta}}{2},
\end{equation*}
where $1\leq i\leq d$, and therefore
\begin{equation}\label{Ti:ineq}
\left\Vert\left(
\begin{array}{cc}
(1+\beta)I_{d}-\alpha Q & -\beta I_{d}
\\
I_{d} & 0_{d}
\end{array}
\right)^{k}\right\Vert
\leq
\max_{1\leq i\leq d}\left\Vert T_{i}^{k}\right\Vert.
\end{equation}
Next, we upper bound $\Vert T_{i}^{k}\Vert$.
We consider three cases (1) $\mu<\lambda_{i}<L$; (2) $\lambda_{i}=\mu$;
(3) $\lambda_{i}=L$.

(1) Consider the case $\mu<\lambda_{i}<L$.
With the choice of $\alpha$ and $\beta$ in \eqref{eq-alpha-beta-hb},
we can compute that for those $\mu<\lambda_{i}<L$, we have
\begin{equation*}
1+\beta-\alpha\lambda_{i}
<1+\beta-\alpha\mu=2\sqrt{\beta},
\end{equation*}
and
\begin{equation*}
1+\beta-\alpha\lambda_{i}
>1+\beta-\alpha L=-2\sqrt{\beta},
\end{equation*}
and thus
the eigenvalues are complex and
\begin{equation*}
\mu_{i,\pm}=\frac{1+\beta-\alpha\lambda_{i}\pm \mathbf{i}\sqrt{4\beta-(1+\beta-\alpha\lambda_{i})^{2}}}{2},
\end{equation*}
where $1\leq i\leq d$.
It is known that the $k$-th power 
of a $2\times 2$ matrix $A$ with distinct eigenvalues
$\mu_{\pm}$ is given by
\begin{equation*}
A^{k}=\frac{\mu_{+}^{k}}{\mu_{+}-\mu_{-}}(A-\mu_{-}I)
+\frac{\mu_{-}^{k}}{\mu_{-}-\mu_{+}}(A-\mu_{+}I),
\end{equation*}
where $I$ is the $2\times 2$ identity matrix \cite{williams2by2}.
In our context, $A=T_{i}$ and $\mu_{\pm}=\mu_{i,\pm}$, 
we get
\begin{equation}\label{Ti:1}
T_{i}^{k}
=\frac{\mu_{i,+}^{k}}{\mu_{i,+}-\mu_{i,-}}(T_{i}-\mu_{i,-}I)
+\frac{\mu_{i,-}^{k}}{\mu_{i,-}-\mu_{i,+}}(T_{i}-\mu_{i,+}I).
\end{equation}
We can compute that
\begin{equation}\label{Ti:2}
|\mu_{i,+}|
=|\mu_{i,-}|
=\left(\frac{1}{4}\left[(1+\beta-\alpha\lambda_{i})^{2}
+(4\beta-(1+\beta-\alpha\lambda_{i})^{2})\right]\right)^{1/2}
=\sqrt{\beta},
\end{equation}
and
\begin{align}\label{Ti:3}
\frac{1}{|\mu_{i,+}-\mu_{i,-}|}
&=\frac{1}{\sqrt{4\beta-(1+\beta-\alpha\lambda_{i})^{2}}}
\\
&=\frac{1}{\sqrt{(2\sqrt{\beta}-1-\beta+\alpha\lambda_{i})
(2\sqrt{\beta}+1+\beta-\alpha\lambda_{i})}}
\nonumber
\\
&=\frac{1}{\sqrt{(-(\sqrt{\beta}-1)^{2}+\alpha\lambda_{i})
((\sqrt{\beta}+1)^{2}-\alpha\lambda_{i})}}
\nonumber
\\
&=\frac{(\sqrt{\mu}+\sqrt{L})^{2}}{4\sqrt{(\lambda_{i}-\mu)(L-\lambda_{i})}}.
\nonumber
\end{align}
Moreover, 
\begin{equation*}
T_{i}-\mu_{i,-}I
=\left(
\begin{array}{cc}
\mu_{i,+} & -\beta
\\
1 & -\mu_{i,-}
\end{array}
\right)
=\left(\begin{array}{c}
\mu_{i,+}
\\
1
\end{array}\right)
\left(\begin{array}{cc}
1 & -\mu_{i,-}
\end{array}\right),
\end{equation*}
and
\begin{equation*}
T_{i}-\mu_{i,+}I
=\left(
\begin{array}{cc}
\mu_{i,-} & -\beta
\\
1 & -\mu_{i,+}
\end{array}
\right)
=\left(\begin{array}{c}
\mu_{i,-} 
\\
1
\end{array}\right)
\left(\begin{array}{cc}
1 & -\mu_{i,+}
\end{array}\right).
\end{equation*}
Therefore,
\begin{equation}\label{Ti:4}
\left\Vert T_{i}-\mu_{i,-}I\right\Vert
\leq
\left\Vert\left(\begin{array}{c}
\mu_{i,+}
\\
1
\end{array}\right)\right\Vert
\left\Vert\left(\begin{array}{cc}
1 & -\mu_{i,-}
\end{array}\right)\right\Vert
=\beta+1,
\end{equation}
and
\begin{equation}\label{Ti:5}
\left\Vert T_{i}-\mu_{i,+}I\right\Vert
\leq\left\Vert\left(\begin{array}{c}
\mu_{i,-}
\\
1
\end{array}\right)\right\Vert
\left\Vert\left(\begin{array}{cc}
1 & -\mu_{i,+}
\end{array}\right)\right\Vert
=\beta+1.
\end{equation}
Hence, it follows from 
\eqref{Ti:1}, \eqref{Ti:2}, \eqref{Ti:3}, \eqref{Ti:4}
and \eqref{Ti:5} that
\begin{equation*}
\left\Vert T_{i}^{k}\right\Vert
\leq(\sqrt{\beta})^{k}
\frac{(\beta+1)(\sqrt{\mu}+\sqrt{L})^{2}}{4\sqrt{(\lambda_{i}-\mu)(L-\lambda_{i})}}
=\left(\frac{\sqrt{L}-\sqrt{\mu}}{\sqrt{L}+\sqrt{\mu}}\right)^{k}
\frac{\mu+L}{2\sqrt{(\lambda_{i}-\mu)(L-\lambda_{i})}}.
\end{equation*}

(2) Consider the case $\lambda_{i}=\mu$.
With the choice of $\alpha$ and $\beta$ in \eqref{eq-alpha-beta-hb},
we can compute that for those $\lambda_{i}=\mu$, we have
\begin{equation*}
(1+\beta-\alpha\lambda_{i})^{2}=(1+\beta-\alpha\mu)^{2}=4\beta,
\end{equation*}
so we have double eigenvalues
and indeed $1+\beta-\alpha\lambda_{i}=2\sqrt{\beta}$,
and
\begin{equation*}
T_{i}=\left(
\begin{array}{cc}
2\sqrt{\beta} & -\beta
\\
1 & 0
\end{array}
\right),
\qquad 1\leq i\leq d,
\end{equation*}
and by a direct computation (e.g. induction on $k$), we get:
\begin{equation*}
T_{i}^{k}=(\sqrt{\beta})^{k}\left(
\begin{array}{cc}
(k+1) & -k\beta^{1/2}
\\
k\beta^{-1/2} & -(k-1)
\end{array}
\right),
\quad 1\leq i\leq d.
\end{equation*}
Thus,
\begin{align}
\left\Vert T_{i}^{k}\right\Vert
&\leq\sqrt{\text{Tr}\left(T_{i}^{k}(T_{i}^{k})^{T}\right)}
\\
&=(\sqrt{\beta})^{k}\sqrt{2k^{2}+2+k^{2}(\beta+\beta^{-1})}
\\
&=\left(\frac{\sqrt{L}-\sqrt{\mu}}{\sqrt{L}+\sqrt{\mu}}\right)^{k}
\sqrt{4k^{2}\left(\frac{L+\mu}{L-\mu}\right)^{2}+2}.\label{ineq-tik-bound}
\end{align}
Finally, we note that the matrix $T_i^k/(\sqrt{\beta}^k k)$ as $k$ goes to infinity converges to the $2\times 2$ matrix 
$$M_{2,2}(\beta): =\begin{pmatrix}
1 & -\beta^{1/2}
\\
\beta^{-1/2} & -1
\end{pmatrix}, \quad \| M_{2,2}(\beta)\| > 0.$$ 
Therefore, the linear dependency of our bound in \eqref{ineq-tik-bound} with respect to $k$ is tight. This behavior is expected due to the fact that $T_i^k$ has double roots.

(3) Consider the case $\lambda_{i}=L$.
With the choice of $\alpha$ and $\beta$ in \eqref{eq-alpha-beta-hb},
we can compute that for those $\lambda_{i}=L$, we have
\begin{equation*}
(1+\beta-\alpha\lambda_{i})^{2}=(1+\beta-\alpha L)^{2}=4\beta,
\end{equation*}
so we have double eigenvalues
and indeed $1+\beta-\alpha\lambda_{i}=-2\sqrt{\beta}$,
and
\begin{equation*}
T_{i}=\left(
\begin{array}{cc}
-2\sqrt{\beta} & -\beta
\\
1 & 0
\end{array}
\right),
\qquad 1\leq i\leq d,
\end{equation*}
and by a direct computation (e.g. induction on $k$), we get:
\begin{equation*}
T_{i}^{k}=(\sqrt{\beta})^{k}\left(
\begin{array}{cc}
(k+1) & k\beta^{1/2}
\\
-k\beta^{-1/2} & -(k-1)
\end{array}
\right),
\quad 1\leq i\leq d.
\end{equation*}
Thus,
\begin{align*}
\left\Vert T_{i}^{k}\right\Vert
&\leq\sqrt{\text{Tr}\left(T_{i}^{k}(T_{i}^{k})^{T}\right)}
\\
&=(\sqrt{\beta})^{k}\sqrt{2k^{2}+2+k^{2}(\beta+\beta^{-1})}
\\
&=\left(\frac{\sqrt{L}-\sqrt{\mu}}{\sqrt{L}+\sqrt{\mu}}\right)^{k}
\sqrt{4k^{2}\left(\frac{L+\mu}{L-\mu}\right)^{2}+2}.
\end{align*}

Finally, combining the three cases (1) $\mu<\lambda_{i}<L$;
(2) $\lambda_{i}=\mu$; (3) $\lambda_{i}=L$, we get
\begin{equation}\label{eqn:final:ineq}
\max_{1\leq i\leq d}\left\Vert T_{i}^{k}\right\Vert
\leq
\left(\frac{\sqrt{L}-\sqrt{\mu}}{\sqrt{L}+\sqrt{\mu}}\right)^{k}
\max\left\{\max_{i:\mu<\lambda_{i}<L}\frac{\mu+L}{2\sqrt{(\lambda_{i}-\mu)(L-\lambda_{i})}}
,\sqrt{4k^{2}\left(\frac{L+\mu}{L-\mu}\right)^{2}+2}\right\}.
\end{equation}
Then it follows from \eqref{Ti:ineq} that
\begin{align}\label{ineq:final}
&\left\Vert\left(
\begin{array}{cc}
(1+\beta)I_{d}-\alpha Q & -\beta I_{d}
\\
I_{d} & 0_{d}
\end{array}
\right)^{k}\right\Vert
\\
&\leq
\left(\frac{\sqrt{L}-\sqrt{\mu}}{\sqrt{L}+\sqrt{\mu}}\right)^{k}
\max\left\{\max_{i:\mu<\lambda_{i}<L}\frac{\mu+L}{2\sqrt{(\lambda_{i}-\mu)(L-\lambda_{i})}}
,\sqrt{4k^{2}\left(\frac{L+\mu}{L-\mu}\right)^{2}+2}\right\}.
\nonumber
\end{align}
Recall that
\begin{equation*}
\left(
\begin{array}{c}
x_{k}-x_{\ast}
\\
x_{k-1}-x_{\ast}
\end{array}
\right)
=
\left(
\begin{array}{cc}
(1+\beta)I_{d}-\alpha Q & -\beta I_{d}
\\
I_{d} & 0_{d}
\end{array}
\right)^{k}
\left(
\begin{array}{c}
x_{0}-x_{\ast}
\\
x_{-1}-x_{\ast}
\end{array}
\right),
\end{equation*}
and the proof is complete by applying \eqref{ineq:final}.
\end{proof}

Before we state the proof of Theorem \ref{thm:HB}, 
let us state the following result, which is built on Theorem \ref{thm:HB-deterministic}.

\begin{lemma}\label{lem:HB}
Let us consider two couplings $(x_{k}^{(1)})_{k\geq 0}$
and $(x_{k}^{(2)})_{k\geq 0}$ with the common noise $(\varepsilon_{k+1})_{k\geq 0}$ that starts from $x_{0}^{(1)}$ and $x_{0}^{(2)}$:
\begin{align}
&x_{k+1}^{(1)}=x_{k}^{(1)}-\alpha\nabla f(x_{k}^{(1)})+\beta(x_{k}^{(1)}-x_{k-1}^{(1)})+\varepsilon_{k+1},\label{coupling:1}
\\
&x_{k+1}^{(2)}=x_{k}^{(2)}-\alpha\nabla f(x_{k}^{(2)})+\beta(x_{k}^{(2)}-x_{k-1}^{(2)})+\varepsilon_{k+1},
\label{coupling:2}
\end{align}
where $f$ is quadratic and $f(x)=\frac{1}{2}x^{T}Qx+a^{T}x+b$. Then, we have
\begin{align}
\left\Vert
\left(
\begin{array}{c}
x_{k+1}^{(1)}-x_{k+1}^{(2)}
\\
x_{k}^{(1)}-x_{k}^{(2)}
\end{array}
\right)
\right\Vert
\leq
C_{k}\rho_{HB}^{k}
\left\Vert
\left(
\begin{array}{c}
x_{1}^{(1)}-x_{1}^{(2)}
\\
x_{0}^{(1)}-x_{0}^{(2)}
\end{array}
\right)\right\Vert,
\nonumber
\end{align}
where $\rho_{HB}$ and $C_k$ are defined by \eqref{def-rho-hb-opt} and \eqref{eqn:Ck} respectively.
\end{lemma}

\begin{proof}[Proof of Lemma \ref{lem:HB}]
We can compute that
\begin{equation*}
\left(
\begin{array}{c}
x_{k+1}^{(1)}-x_{k+1}^{(2)}
\\
x_{k}^{(1)}-x_{k}^{(2)}
\end{array}
\right)
=
\left(
\begin{array}{cc}
(1+\beta)I_{d}-\alpha Q & -\beta I_{d}
\\
I_{d} & 0_{d}
\end{array}
\right)^{k}
\left(
\begin{array}{c}
x_{1}^{(1)}-x_{1}^{(2)}
\\
x_{0}^{(1)}-x_{0}^{(2)}
\end{array}
\right).
\end{equation*}
It follows from the estimate 
\eqref{ineq:final}
in the proof of Theorem \ref{thm:HB-deterministic}
and the definitions of
$\rho_{HB}$ and $C_k$ in \eqref{def-rho-hb-opt} and \eqref{eqn:Ck}
that we have
\begin{equation*}
\left\Vert\left(
\begin{array}{cc}
(1+\beta)I_{d}-\alpha Q & -\beta I_{d}
\\
I_{d} & 0_{d}
\end{array}
\right)^{k}\right\Vert
\leq
C_{k}\rho_{HB}^{k}.
\end{equation*}
The proof is complete.
\end{proof}

\begin{proof}[Proof of Theorem \ref{thm:HB}]
We recall from Lemma \ref{lem:HB} that
for any coupling $x^{(1)}$ and $x^{(2)}$
\begin{equation*}
\left\Vert
\left(
\begin{array}{c}
x_{k}^{(1)}-x_{k}^{(2)}
\\
x_{k-1}^{(1)}-x_{k-1}^{(2)}
\end{array}
\right)
\right\Vert
\leq
C_{k}\left(\frac{\sqrt{L}-\sqrt{\mu}}{\sqrt{L}+\sqrt{\mu}}\right)^{k}
\left\Vert
\left(
\begin{array}{c}
x_{0}^{(1)}-x_{0}^{(2)}
\\
x_{-1}^{(1)}-x_{-1}^{(2)}
\end{array}
\right)\right\Vert.
\nonumber
\end{equation*}
Following from the proof of Theorem \ref{thm:alpha:beta}, 
we can show by constructing a Cauchy sequence that 
there exists a unique stationary distribution $\pi_{\alpha,\beta}$.
Finally, we assume that 
$(x_{0}^{(1)},x_{-1}^{(1)})$ starts
from the given $(x_{0},x_{-1})$ distributed
as $\nu_{0,\alpha,\beta}$
and $(x_{0}^{(2)},x_{-1}^{(2)})$ starts
from the stationary distribution $\pi_{\alpha,\beta}$
so that their $L_{p}$ distance is exactly
the $\mathcal{W}_{p}$ distance. 
Then we get
\begin{align*}
\mathcal{W}_{p}^{p}\left(\nu_{k,\alpha,\beta},
\pi_{\alpha,\beta}\right)
&\leq
\mathbb{E}\left\Vert
\left(
\begin{array}{c}
x_{k}^{(1)}-x_{k}^{(2)}
\\
x_{k-1}^{(1)}-x_{k-1}^{(2)}
\end{array}
\right)
\right\Vert^{p}
\\
&\leq
C_{k}^{p}\left(\frac{\sqrt{L}-\sqrt{\mu}}{\sqrt{L}+\sqrt{\mu}}\right)^{pk}
\mathcal{W}_{p}^{p}\left(\nu_{0,\alpha,\beta},
\pi_{\alpha,\beta}\right),
\nonumber
\end{align*}
and the proof is complete 
by taking the power $1/p$ in the above equation.
\end{proof}

Before we state the proof of Theorem \ref{thm:f:HB},
let us spell out $X$ and $V_{HB}(\xi_{0})$ in
the statement of Theorem \ref{thm:f:HB} explicitly here.
We will show that Theorem \ref{thm:f:HB} holds
with $V_{HB}(\xi_{0})$ given by
\begin{equation*}
V_{HB}(\xi_{0}):=\mathbb{E}\left[\Vert(\xi_{0}-\xi_{\ast})(\xi_{0}-\xi_{\ast})^{T}\Vert\right]
+\frac{\alpha_{HB}^{2}\Vert\Sigma\Vert}{1-\rho_{HB}^{2}},
\end{equation*}
where $\Sigma:=\mathbb{E}[\varepsilon_{k}\varepsilon_{k}^{T}]$ and
$X_{HB}=\mathbb{E}[(\xi_{\infty}-\xi_{\ast})(\xi_{\infty}-\xi_{\ast})^{T}]$ 
satisfies the discrete Lyapunov equation:
\begin{equation*}
X_{HB}=A_{Q}X_{HB}A_{Q}^{T}+\left(\begin{array}{cc}
\alpha_{HB}^{2}\Sigma & 0_{d}
\\
0_{d} & 0_{d}
\end{array}\right).
\end{equation*}
and
\begin{equation*}
A_{Q}:=\left(
\begin{array}{cc}
(1+\beta_{HB})I_{d}-\alpha_{HB} Q & -\beta_{HB} I_{d}
\\
I_{d} & 0_{d}
\end{array}
\right).
\end{equation*}
In the special case $\Sigma = c^2 I_d $ for some constant $c\geq 0$, we obtain
\begin{equation}\label{eq-trace-cov}
\text{Tr}(X_{HB}) =c^{2}\sum_{i=1}^d  \frac{2\alpha_{HB} (1+\beta_{HB})}{(1-\beta_{HB})\lambda_{i} (2+2\beta_{HB} - \alpha_{HB} \lambda_i)},
\end{equation}
where $\{\lambda_i\}_{i=1}^d$ are the eigenvalues of $Q$. 

Now, we are ready to prove Theorem \ref{thm:f:HB}.

\begin{proof}[Proof of Theorem \ref{thm:f:HB}]
For the stochastic heavy ball method
\begin{equation*}
x_{k+1}=x_{k}-\alpha(\nabla f(x_{k})+\varepsilon_{k+1})+\beta(x_{k}-x_{k-1}),
\end{equation*}
where we consider the quadratic objective 
$f(x)=\frac{1}{2}x^{T}Qx+a^{T}x+b$ so that
\begin{equation*}
x_{k+1}=x_{k}-\alpha(Qx_{k}+a+\varepsilon_{k+1})+\beta(x_{k}-x_{k-1}),
\end{equation*}
and the minimizer $x_{\ast}$ satisfies:
\begin{equation*}
x_{\ast}=x_{\ast}-\alpha(Qx_{\ast}+a)+\beta(x_{\ast}-x_{\ast}),
\end{equation*}
so that
\begin{equation*}
(x_{k+1}-x_{\ast})=(x_{k}-x_{\ast})-\alpha(Q(x_{k}-x_{\ast})
+\varepsilon_{k+1})+\beta((x_{k}-x_{\ast})-(x_{k-1}-x_{\ast})),
\end{equation*}
and
\begin{equation*}
\left(
\begin{array}{c}
x_{k}-x_{\ast}
\\
x_{k-1}-x_{\ast}
\end{array}
\right)
=
\left(
\begin{array}{cc}
(1+\beta)I_{d}-\alpha Q & -\beta I_{d}
\\
I_{d} & 0_{d}
\end{array}
\right)
\left(
\begin{array}{c}
x_{k-1}-x_{\ast}
\\
x_{k-2}-x_{\ast}
\end{array}
\right)
+
\left(\begin{array}{c}
-\alpha\varepsilon_{k}
\\
0_{d}
\end{array}\right),
\end{equation*}
and with $\Sigma:=\mathbb{E}[\varepsilon_{k}\varepsilon_{k}^{T}]$, we get
\begin{equation}\label{eqn:iterate:k:2}
\mathbb{E}\left[(\xi_{k}-\xi_{\ast})(\xi_{k}-\xi_{\ast})^{T}\right]
=
A_{Q}\mathbb{E}\left[(\xi_{k-1}-x_{\ast})(\xi_{k-1}-x_{\ast})^{T}\right]A_{Q}^{T}
+
\left(\begin{array}{cc}
\alpha^{2}\Sigma & 0_{d}
\\
0_{d} & 0_{d}
\end{array}\right),
\end{equation}
where
\begin{equation*}
A_{Q}=\left(
\begin{array}{cc}
(1+\beta)I_{d}-\alpha Q & -\beta I_{d}
\\
I_{d} & 0_{d}
\end{array}
\right).
\end{equation*}
Therefore, 
\begin{equation*}
X=\mathbb{E}\left[(\xi_{\infty}-\xi_{\ast})(\xi_{\infty}-\xi_{\ast})^{T}\right]
\end{equation*}
satisfies the discrete Lyapunov equation: 
\begin{equation*}
X=A_{Q}XA_{Q}^{T}+\left(\begin{array}{cc}
\alpha^{2}\Sigma & 0_{d}
\\
0_{d} & 0_{d}
\end{array}\right).
\end{equation*}

Next by iterating equation \eqref{eqn:iterate:k:2} over $k$, we immediately obtain
\begin{align*}
\mathbb{E}\left[(\xi_{k}-\xi_{\ast})(\xi_{k}-\xi_{\ast})^{T}\right]
&=\left(A_{Q}\right)^{k}\mathbb{E}\left[(\xi_{0}-\xi_{\ast})(\xi_{0}-\xi_{\ast})^{T}\right]\left(A_{Q}^{T}\right)^{k}
+
\\
&\qquad\qquad\qquad
\sum_{j=0}^{k-1}
\left(A_{Q}\right)^{j}\left(\begin{array}{cc}
\alpha^{2}\Sigma & 0_{d}
\\
0_{d} & 0_{d}
\end{array}\right)\left(A_{Q}^{T}\right)^{j},
\end{align*}
so that
\begin{align*}
&\mathbb{E}\left[(\xi_{k}-\xi_{\ast})(\xi_{k}-\xi_{\ast})^{T}\right]
\\
&=\mathbb{E}\left[(\xi_{\infty}-\xi_{\ast})(\xi_{\infty}-\xi_{\ast})^{T}\right]
+\left(A_{Q}\right)^{k}\mathbb{E}\left[(\xi_{0}-\xi_{\ast})(\xi_{0}-\xi_{\ast})^{T}\right]\left(A_{Q}^{T}\right)^{k}
\\
&\qquad\qquad\qquad
-\sum_{j=k}^{\infty}
\left(A_{Q}\right)^{j}\left(\begin{array}{cc}
\alpha^{2}\Sigma & 0_{d}
\\
0_{d} & 0_{d}
\end{array}\right)\left(A_{Q}^{T}\right)^{j},
\end{align*}
which implies that
\begin{align*}
&\text{Tr}\left(\mathbb{E}\left[(\xi_{k}-\xi_{\ast})(\xi_{k}-\xi_{\ast})^{T}\right]\right)
\\
&=\text{Tr}\left(\mathbb{E}\left[(\xi_{\infty}-\xi_{\ast})(\xi_{\infty}-\xi_{\ast})^{T}\right]\right)
+\left(A_{Q}\right)^{k}\mathbb{E}\left[(\xi_{0}-\xi_{\ast})(\xi_{0}-\xi_{\ast})^{T}\right]\left(A_{Q}^{T}\right)^{k}
\\
&\qquad\qquad\qquad
-\sum_{j=k}^{\infty}
\left(A_{Q}\right)^{j}\left(\begin{array}{cc}
\alpha^{2}\Sigma & 0_{d}
\\
0_{d} & 0_{d}
\end{array}\right)\left(A_{Q}^{T}\right)^{j}
\\
&\leq
\text{Tr}(X)
+\left\Vert A_{Q}^{k}\right\Vert^{2}
\mathbb{E}\left[\Vert(\xi_{0}-\xi_{\ast})(\xi_{0}-\xi_{\ast})^{T}\Vert\right]
+\sum_{j=k}^{\infty}\left\Vert A_{Q}^{j}\right\Vert^{2}\alpha^{2}\Vert\Sigma\Vert
\\
&\leq\text{Tr}(X)
+C_{k}^{2}\rho_{HB}^{2k}\mathbb{E}\left[\Vert(\xi_{0}-\xi_{\ast})(\xi_{0}-\xi_{\ast})^{T}\Vert\right]
+\alpha^{2}\Vert\Sigma\Vert C_{k}^{2}\frac{\rho_{HB}^{2k}}{1-\rho_{HB}^{2}},
\end{align*}
where we used the estimate $\Vert A_{Q}^{k}\Vert\leq C_{k}\rho_{HB}^{k}$
from the proof of Theorem \ref{thm:HB-deterministic}.

Finally, since $\nabla f$ is $L$-Lipschtiz,
\begin{equation*}
\mathbb{E}[f(x_{k})]-f(x_{\ast})
\leq\frac{L}{2}\mathbb{E}\Vert x_{k}-x_{\ast}\Vert^{2}
\leq\frac{L}{2}\mathbb{E}\Vert \xi_{k}-\xi_{\ast}\Vert^{2}
=\frac{L}{2}\text{Tr}\left(\mathbb{E}\left[(\xi_{k}-\xi_{\ast})(\xi_{k}-\xi_{\ast})^{T}\right]\right).
\end{equation*}
The proof of \eqref{ineq-shb-subopt} is complete. To show \eqref{eq-trace-cov}, we can adapt the proof technique of \cite[Proposition 3.2]{StrConvex} for gradient descent to HB. Without loss of generality, due to the scaling of the Lyapunov equation, we can assume $c=1$. Consider the eigenvalue decomposition $A_Q = V \Lambda V^T$ where $Q$ is orthogonal and $\Lambda$ is diagonal with $\Lambda(i,i) = \lambda_i$. We can write 
$$ A_Q = \bar{V} A_\Lambda \bar{V}^{T},$$
where 
$$ \bar{V} = \begin{pmatrix} 
V & 0_d \\
0_d & V
\end{pmatrix}, \quad A_\Lambda = \begin{pmatrix} (1+\beta)I_d - \alpha \Lambda & -\beta I_d \\
I_d     & 0_d
\end{pmatrix}.
$$
Futhermore, following \cite{Recht}, let $P\in\R^{2d\times 2d}$ be the permutation matrix with entries
$$ P(i,j) = \begin{cases} 
1 & \mbox{if } i \mbox{ is odd}, j=i, \\
            1 & \mbox{if } i \mbox{ is even}, j=2d+i, \\
            0 & \mbox{otherwise}.
            \end{cases} 
$$ 
Then, we have
$$A_M := PA_\Lambda P^T =  \begin{pmatrix} M_1 & 0_d &  \hdots & 0_d \\
0_d & M_2 & \hdots & 0_d \\
\vdots &\vdots &\ddots&  \vdots\\
0_d & 0_d & \hdots & M_d
\end{pmatrix} \quad \mbox{where} \quad M_i= \begin{pmatrix} (1+\beta) - \alpha \lambda_i & -\beta  \\
1    & 0
\end{pmatrix} \in \R^{2\times 2}.  
$$ 

If we define $Y := U X U^{-1}$ for the orthogonal matrix $U=P\bar{V}^T$, it solves 
 \begin{equation*}
 	A_M Y A_M^T - Y + S = 0, \quad S:=P \begin{pmatrix}
 	    \alpha^2 I_d & 0_d \\
 	        0_d      & 0_d
 	    \end{pmatrix} P^T,
 \end{equation*}
where the latter matrix $S$ is a $2d\times 2d$ diagonal matrix with entries $S(i,i) = \alpha^2$ if $i$ is odd, and zero if $i$ is even. Due to the special structure of $S$ and $A_M$, the solution $Y$ has the structure
$$
Y = \begin{pmatrix} Y_1 & 0_d &  \hdots & 0_d \\
0_d & Y_2 & \hdots & 0_d \\
\vdots &\vdots &\ddots&  \vdots\\
0_d & 0_d & \hdots & Y_d
\end{pmatrix},
$$
where $Y_i$ solves the $2\times 2$ Lyapunov equation
\begin{equation*}
M_i Y_i M_i^T - Y_i + \begin{pmatrix}
	\alpha^2   & 0 \\
     0 &   0
\end{pmatrix} = 0.
\end{equation*}
If we write \begin{equation*} Y_i = \begin{pmatrix} x_i & y_i \\
								y_i & w_i
      			\end{pmatrix}
          \end{equation*} with scalars $x_i$, $y_i$ and $w_i$, this equation is equivalent to the linear system
$$ 
\begin{pmatrix}
a^2 - 1 & 2ab & b^2 \\
a    &  b-1   &  0 \\
1    &   0    & -1
\end{pmatrix} \begin{pmatrix} x_i \\ y_i \\ w_i \end{pmatrix} = \begin{pmatrix} -\alpha^2 \\ 0 \\ 0 \end{pmatrix}, 
$$        
with $$ a = 1+\beta - \alpha \lambda_{i}, \quad b= -\beta.$$
After a simple computation, we obtain

$$ x_i = w_i = \frac{\alpha^{2}(b-1)}{(b+1)(a-b+1)(a+b-1)}=\frac{\alpha (1+\beta)}{(1-\beta)\lambda_{i} (2+2\beta -\alpha\lambda_{i})}.$$
Therefore we obtain
  $$\text{Tr}(X) =  \text{Tr}(Y) = \sum_{i=1}^{d} \text{Tr}(Y_i) = 2 \sum_{i=1}^{d} x_i = \sum_{i=1}^{d}   \frac{2 \alpha (1+\beta)}{(1-\beta)\lambda_{i} (2+2\beta -\alpha\lambda_{i})},$$
which completes the proof. 

\end{proof}

\section{Proofs of Results in Section \ref{sec:strong:convex}}

Before we proceed to prove the main results in Section \ref{sec:strong:convex},
let us first show that the weighted total variation distance $d_{\psi}$
upper bounds the standard $1$-Wasserstein distance.

\begin{proposition}\label{prop:metric}
Assume $\tilde{P}(2,2)\neq 0$. Then,
\begin{equation*}
\mathcal{W}_{1}(\mu_{1},\mu_{2})
\leq c_{0}^{-1}d_{\psi}(\mu_{1},\mu_{2}),
\end{equation*}
where $\mathcal{W}_{1}$ is the standard $1$-Wasserstein distance
and 
\begin{equation} 
c_{0}:=\min\{\hat{c}_{0}\psi,1\},\label{def-c0}
\end{equation} 
where
$\hat{c}_{0}$ is the smallest positive eigenvalue
of
\begin{equation*}
\tilde{P}\otimes I_{d}
+
\left(\begin{array}{cc}
\frac{\mu}{2}I_{d} & 0_{d}
\\
0_{d} & 0_{d}
\end{array}\right).
\end{equation*}
\end{proposition}

\begin{proof}
By applying the Kantorovich-Rubinstein duality for the Wasserstein metric
(see e.g. \cite{villani2008optimal}), we get
\begin{align*}
\mathcal{W}_{1}(\mu_{1},\mu_{2})
&=\sup_{\phi\in L^{1}(d\mu_{1})}
\left\{\int_{\mathbb{R}^{2d}}\phi(\xi)(\mu_{1}-\mu_{2})(d\xi): \text{$\phi$ is $1$-Lipschitz}\right\}
\\
&=\sup_{\phi\in L^{1}(d\mu_{1})}
\left\{\int_{\mathbb{R}^{2d}}(\phi(\xi)-\phi(\xi_{\ast}))(\mu_{1}-\mu_{2})(d\xi): \text{$\phi$ is $1$-Lipschitz}\right\}
\\
&\leq
\int_{\mathbb{R}^{2d}}\Vert\xi-\xi_{\ast}\Vert|\mu_{1}-\mu_{2}|(d\xi)
\\
&\leq c_{0}^{-1}\int_{\mathbb{R}^{2d}}(1+\psi V_{P}(\xi))|\mu_{1}-\mu_{2}|(d\xi)
=c_{0}^{-1}d_{\psi}(\mu_{1},\mu_{2}),
\end{align*}
where we used $1+\psi V_{P}(\xi)\geq c_{0}\Vert\xi-\xi_{\ast}\Vert$
from Lemma \ref{lem:prop:metric}.
\end{proof}

\begin{lemma}\label{lem:prop:metric}
Assume $\tilde{P}(2,2)\neq 0$. Then,
\begin{equation*}
1+\psi V_{P}(\xi)\geq c_{0}\Vert\xi-\xi_{\ast}\Vert,
\end{equation*}
for any $\xi\in\mathbb{R}^{2d}$, where
$c_{0}=\min\{\hat{c}_{0}\psi,1\}$, where
$\hat{c}_{0}$ is the smallest positive eigenvalue
of
\begin{equation*}
\tilde{P}\otimes I_{d}
+
\left(\begin{array}{cc}
\frac{\mu}{2}I_{d} & 0_{d}
\\
0_{d} & 0_{d}
\end{array}\right).
\end{equation*}
\end{lemma}

\begin{proof}
Let $\xi^T=(x^T,y^T)$.
If $\Vert\xi-\xi_{\ast}\Vert\leq 1$, then $c_{0}=1$ works.
Otherwise,
\begin{align*}
V_{P}(\xi)
&=f(x)-f(x_{\ast})+(\xi-\xi_{\ast})^{T}P(\xi-\xi_{\ast})
\\
&\geq
(\xi-\xi_{\ast})^{T}P(\xi-\xi_{\ast})
+\frac{\mu}{2}\Vert x-x_{\ast}\Vert^{2}
\\
&=(\xi-\xi_{\ast})^{T}\tilde{P}\otimes I_{d}(\xi-\xi_{\ast})
+
(\xi-\xi_{\ast})^{T}
\left(\begin{array}{cc}
\frac{\mu}{2}I_{d} & 0_{d}
\\
0_{d} & 0_{d}
\end{array}\right)(\xi-\xi_{\ast}).
\end{align*}
The proof is complete.
\end{proof}

For constrained optimization on a compact set $\mathcal{C}$, 
we have the following result.

\begin{proposition}\label{prop:metric:2}
For any $\mu_1,\mu_2$ on the product space $\mathcal{C}^{2}:=\mathcal{C}\times\mathcal{C}$,
\begin{equation*}
\mathcal{W}_{p}(\mu_1,\mu_2) 
\leq 
2^{1/p}\mathcal{D}_{\mathcal{C}^{2}} \Vert \mu_{1}-\mu_{2}\Vert_{TV}^{1/p}
\leq\mathcal{D}_{\mathcal{C}^{2}} d_{\psi}^{1/p}(\mu_{1},\mu_{2}),
\end{equation*}
where $\mathcal{D}_{\mathcal{C}^{2}}$ is the diameter of $\mathcal{C}^{2}$.
\end{proposition}

\begin{proof}
The second inequality in Proposition \ref{prop:metric:2}
follows from $d_{\psi}(\mu_{1},\mu_{2})\geq 2\Vert\mu_{1}-\mu_{2}\Vert_{TV}$.
So it suffices to prove the first inequality. 
We can compute that
\begin{align*}
\mathcal{W}_{p}^{p}(\mu_{1},\mu_{2})
&=\inf_{X_{1}\sim\mu_{1},X_{2}\sim\mu_{2}}
\mathbb{E}\left[\Vert X_{1}-X_{2}\Vert^{p}\right]
\\
&\leq\mathcal{D}_{\mathcal{C}^{2}}^{p-1}
\inf_{X_{1}\sim\mu_{1},X_{2}\sim\mu_{2}}
\mathbb{E}\left[\Vert X_{1}-X_{2}\Vert\right]
\\
&=\mathcal{D}_{\mathcal{C}^{2}}^{p-1}\mathcal{W}_{1}(\mu_{1},\mu_{2})
\\
&=\mathcal{D}_{\mathcal{C}^{2}}^{p-1}\sup_{\phi\in L^{1}(d\mu_{1})}
\left\{\int_{\mathbb{R}^{2d}}(\phi(\xi)-\phi(\xi_{\ast}))(\mu_{1}-\mu_{2})(d\xi): \text{$\phi$ is $1$-Lipschitz}\right\}
\\
&\leq\mathcal{D}_{\mathcal{C}^{2}}^{p-1}\int_{\mathbb{R}^{2d}}\Vert\xi-\xi_{\ast}\Vert|\mu_{1}-\mu_{2}|(d\xi)
\leq
2\mathcal{D}_{\mathcal{C}^{2}}^{p}\Vert\mu_{1}-\mu_{2}\Vert_{TV}.
\end{align*}
\end{proof}

\subsection{Proofs of Results in Section \ref{sec:strong:convex:AG}}

Throughout Section \ref{sec:strong:convex}, 
the noise $\varepsilon_{k}$ are assumed to satisfy
Assumption \ref{assump:noise}.
Our proof of Theorem \ref{thm:general:alpha:beta} relies
on the geometric ergodicity and convergence theory
of Markov chains. 
Geometric ergodicity and convergence
of Markov chains has been well studied
in the literature. Harris' ergodic theorem
of Markov chains essentially states
that a Markov chain is ergodic if it admits
a small set that is visited infinitely often \cite{Harris}.
Such a result often relies on finding an appropriate 
Lyapunov function \cite{Meyn1993}. The transition probabilities
converge exponentially fast towards
the unique invariant measure, and the prefactor
is controlled by the Lyapunov function \cite{Meyn1993}.
Computable bounds for geometric convergence rates
of Markov chains has been obtained in e.g. \cite{Meyn1994,hairer}.
In the following, we state the results from \cite{hairer}.
Before we proceed, let us introduce some definitions
and notations.

Let $\mathbb{X}$ be a measurable space and
$\mathcal{P}(x,\cdot)$ be a Markov transition kernel
on $\mathbb{X}$. 
For any measurable function $\varphi:\mathbb{X}\rightarrow[0,+\infty]$,
we define:
\begin{equation*}
(\mathcal{P}\varphi)(x)=\int_{\mathbb{X}}\varphi(y)\mathcal{P}(x,dy).
\end{equation*}

\begin{assumption}[Drift Condition]\label{assump:drift}
There exists a function $V:\mathbb{X}\rightarrow[0,\infty)$
and some constants $K\geq 0$ and $\gamma\in(0,1)$ so that
\begin{equation*}
(\mathcal{P}V)(x)\leq\gamma V(x)+K,
\end{equation*}
for all $x\in\mathbb{X}$.
\end{assumption}

\begin{assumption}[Minorization Condition]\label{assump:minor}
There exists some constant $\eta\in(0,1)$ and a probability
measure $\nu$ so that
\begin{equation*}
\inf_{x\in\mathbb{X}:V(x)\leq R}\mathcal{P}(x,\cdot)\geq\eta\nu(\cdot),
\end{equation*}
for some $R>2K/(1-\gamma)$.
\end{assumption}

Let us recall the definition of the weighted
total variation distance:
\begin{equation*}
d_{\psi}(\mu_{1},\mu_{2})=\int_{\mathbb{X}}(1+\psi V(x))|\mu_{1}-\mu_{2}|(dx).
\end{equation*}
It is noted in \cite{hairer} that $d_{\psi}$
has the following alternative expression.
Define the weighted supremum norm for any $\psi>0$:
\begin{equation*}
\Vert\varphi\Vert_{\psi}:=\sup_{x\in\mathbb{X}}
\frac{|\varphi(x)|}{1+\psi V(x)},
\end{equation*}
and its associated dual metric $d_{\psi}$ on probability measures:
\begin{equation*}
d_{\psi}(\mu_{1},\mu_{2})=\sup_{\varphi:\Vert\varphi\Vert_{\psi}\leq 1}
\int_{\mathbb{X}}\varphi(x)(\mu_{1}-\mu_{2})(dx).
\end{equation*}
It is also noted in \cite{hairer} that $d_{\psi}$
can also be expressed as:
\begin{equation*}
d_{\psi}(\mu_{1},\mu_{2})=\sup_{\varphi:|\Vert\varphi\Vert|_{\psi}\leq 1}
\int_{\mathbb{X}}\varphi(x)(\mu_{1}-\mu_{2})(dx),
\end{equation*}
where
\begin{equation*}
|\Vert\varphi\Vert|_{\psi}:=\sup_{x\neq y}\frac{|\varphi(x)-\varphi(y)|}{2+\psi V(x)+\psi V(y)}.
\end{equation*}

\begin{lemma}[Theorem 1.3. \cite{hairer}]\label{lem:hairer}
If the drift condition (Assumption \ref{assump:drift}) 
and minorization condition (Assumption \ref{assump:minor}) hold,
then there exists $\bar{\eta}\in(0,1)$ and $\psi>0$
so that
\begin{equation*}
d_{\psi}(\mathcal{P}\mu_{1},\mathcal{P}\mu_{2})
\leq\bar{\eta}d_{\psi}(\mu_{1},\mu_{2})
\end{equation*}
for any probability measures $\mu_{1},\mu_{2}$ on $\mathbb{X}$.
In particular, for any $\eta_{0}\in(0,\eta)$
and $\gamma_{0}\in(\gamma+2K/R,1)$ one can 
choose $\psi=\eta_{0}/K$ and $\bar{\eta}=(1-(\eta-\eta_{0}))\vee(2+R\psi\gamma_{0})/(2+R\psi)$.
\end{lemma}

\begin{lemma}[Theorem 1.2. \cite{hairer}]\label{lem:inv}
If the drift condition (Assumption \ref{assump:drift}) 
and minorization condition (Assumption \ref{assump:minor}) hold,
then $\mathcal{P}$ admits a unique invariant measure $\mu_{\ast}$,
i.e. $\mathcal{P}\mu_{\ast}=\mu_{\ast}$.
\end{lemma}

The drift condition has indeed been obtained in \cite{StrConvex}.
The AG method follows the dynamics
\begin{align}
&\xi_{k+1}=A\xi_{k}+B(\nabla f(y_{k})+\varepsilon_{k+1}),
\\
&y_{k}=C\xi_{k},
\end{align}
where 
\begin{equation*}
A:=
\left(\begin{array}{cc}
(1+\beta)I_{d} & -\beta I_{d}
\\
I_{d} & 0_{d}
\end{array}
\right),
\quad
B:=\left(
\begin{array}{c}
-\alpha I_{d}
\\
0_{d}
\end{array}\right),
\quad
C:=
\left(\begin{array}{cc}
(1+\beta)I_{d} & -\beta I_{d}
\end{array}
\right).
\end{equation*}
Define $\tilde{y}_{k}:=y_{k}-x_{\ast}$
and $\tilde{\xi}_{k}:=\xi_{k}-\xi_{\ast}$,
where $\xi_{\ast}=A\xi_{\ast}$ and $x_{\ast}=C\xi_{\ast}$.
Let us recall the Lyapunov function from \eqref{eqn:lyapunov}
\begin{equation*}
V_{P}(\xi_{k})=(\xi_{k}-\xi_{\ast})^{T}P(\xi_{k}-\xi_{\ast})+f(x_{k})-f_{\ast},
\end{equation*}
where $\xi_{\ast}=(x_{\ast},x_{\ast})$.

Next, let us prove that the drift condition holds.
The proof is mainly built on Corollary 4.2. and Lemma 4.5.
in \cite{StrConvex}. 

\begin{lemma}\label{lem:drift}
\begin{equation*}
(\mathcal{P}_{\alpha,\beta}V_{P_{\alpha,\beta}})(\xi)
\leq\gamma_{\alpha,\beta} V_{P_{\alpha,\beta}}(\xi)+K_{\alpha,\beta},
\end{equation*}
where
\begin{equation*}
\gamma_{\alpha,\beta}:=\rho_{\alpha,\beta},
\qquad
K_{\alpha,\beta}:=\left(\frac{L}{2}+\tilde{P}_{\alpha,\beta}(1,1)\right)\alpha^{2}\sigma^{2}.
\end{equation*}
\end{lemma}

\begin{proof}
By Corollary 4.2. and its proof in \cite{StrConvex} (In \cite{StrConvex}, the noise
are assumed to be independent. But a closer look at
the proof of Corollary 4.2. reveals that our Assumption \ref{assump:noise}
suffices), we have 
\begin{align}
&\mathbb{E}[V(\xi_{k+1})]-\rho \mathbb{E}[V(\xi_{k})]
\\
&=\mathbb{E}\left[
\left(\begin{array}{c}
\tilde{\xi}_{k}
\\
\nabla f(y_{k})
\end{array}
\right)^{T}
\left(
\begin{array}{cc}
A^{T}PA-\rho P & A^{T}PB
\\
B^{T}PA & B^{T}PB
\end{array}
\right)
\left(\begin{array}{c}
\tilde{\xi}_{k}
\\
\nabla f(y_{k})
\end{array}
\right)\right]+\mathbb{E}\left[\varepsilon_{k+1}^{T}B^{T}PB\varepsilon_{k+1}\right],
\nonumber
\end{align}
where
\begin{equation*}
V(\xi):=(\xi-\xi_{\ast})^{T}P(\xi-\xi_{\ast}).
\end{equation*}

A closer look at the proof of Corollary 4.2. in \cite{StrConvex}
reveals that the following equality also holds:
\begin{align}\label{eqn:comb:1}
&\mathbb{E}[V(\xi_{k+1})|\xi_{k}]-\rho V(\xi_{k})
\\
&=
\left(\begin{array}{c}
\tilde{\xi}_{k}
\\
\nabla f(y_{k})
\end{array}
\right)^{T}
\left(
\begin{array}{cc}
A^{T}PA-\rho P & A^{T}PB
\\
B^{T}PA & B^{T}PB
\end{array}
\right)
\left(\begin{array}{c}
\tilde{\xi}_{k}
\\
\nabla f(y_{k})
\end{array}
\right)+\mathbb{E}\left[\varepsilon_{k+1}^{T}B^{T}PB\varepsilon_{k+1}\right].
\nonumber
\end{align}

When $f\in\mathcal{S}_{\mu,L}$ is strongly convex,
Lemma 4.5. in \cite{StrConvex} states that
for any $\rho\in(0,1)$,
\begin{align}\label{eqn:expect:noise}
&\left(\begin{array}{c}
\tilde{\xi}_{k}
\\
\nabla f(y_{k})
\end{array}
\right)^{T}
X
\left(\begin{array}{c}
\tilde{\xi}_{k}
\\
\nabla f(y_{k})
\end{array}
\right)
\\
&\leq\rho (f(x_{k})-f_{\ast})-(f(x_{k+1})-f_{\ast})
+\frac{L\alpha^{2}}{2}\Vert\varepsilon_{k+1}\Vert^{2}
-\alpha(1-L\alpha)\nabla f(y_{k})^{T}\varepsilon_{k+1},
\nonumber
\end{align}
where $X:=\rho X_{1}+(1-\rho )X_{2}$, where
\begin{align}
&X_{1}:=\frac{1}{2}
\left(\begin{array}{ccc}
\beta^{2}\mu I_{d} & -\beta^{2}\mu I_{d} & -\beta I_{d}
\\
-\beta^{2}\mu I_{d} & \beta^{2}\mu I_{d} & \beta I_{d}
\\
-\beta I_{d} & \beta I_{d} & \alpha(2-L\alpha)I_{d}
\end{array}\right),
\\
&X_{2}:=\frac{1}{2}
\left(\begin{array}{ccc}
(1+\beta)^{2}\mu I_{d} & -\beta(1+\beta)\mu I_{d} & -(1+\beta)I_{d}
\\
-\beta(1+\beta)\mu I_{d} & \beta^{2}\mu I_{d} & \beta I_{d}
\\
-(1+\beta)I_{d} & \beta I_{d} & \alpha(2-L\alpha)I_{d}
\end{array}\right).
\end{align}
Taking expectation w.r.t. the noise $\varepsilon_{k+1}$ only
in \eqref{eqn:expect:noise}, we get
\begin{equation}\label{eqn:comb:2}
\left(\begin{array}{c}
\tilde{\xi}_{k}
\\
\nabla f(y_{k})
\end{array}
\right)^{T}
X
\left(\begin{array}{c}
\tilde{\xi}_{k}
\\
\nabla f(y_{k})
\end{array}
\right)
\leq\rho (f(x_{k})-f_{\ast})-(f(x_{k+1})-f_{\ast})
+\frac{L\alpha^{2}}{2}\sigma^{2}.
\end{equation}

With the definition 
of $\rho_{\alpha,\beta}$, $P_{\alpha,\beta}$
by Lemma \ref{lem:X}, we get
\begin{equation}\label{eqn:comb:3}
\left(
\begin{array}{cc}
A^{T}P_{\alpha,\beta}A-\rho_{\alpha,\beta}P_{\alpha,\beta} & A^{T}PB
\\
B^{T}P_{\alpha,\beta}A & B^{T}P_{\alpha,\beta}B
\end{array}
\right)
-X\preceq 0.
\end{equation}
Then, combining \eqref{eqn:comb:1} and \eqref{eqn:comb:2}, 
applying \eqref{eqn:comb:3} and the definition of $V_{P_{\alpha,\beta}}$,
we get
\begin{align*}
\mathbb{E}[V_{P_{\alpha,\beta}}(\xi_{k+1})|\xi_{k}]
&\leq\rho_{\alpha,\beta}V_{P_{\alpha,\beta}}(\xi_{k})
+\mathbb{E}\left[\varepsilon_{k+1}^{T}B^{T}P_{\alpha,\beta}B\varepsilon_{k+1}\right]
+\frac{L\alpha^{2}}{2}\sigma^{2}
\\
&=\rho_{\alpha,\beta}V_{P_{\alpha,\beta}}(\xi_{k})
+\mathbb{E}\left[\varepsilon_{k+1}^{T}\alpha^{2}\tilde{P}_{\alpha,\beta}(1,1)I_{d}\varepsilon_{k+1}\right]
+\frac{L\alpha^{2}}{2}\sigma^{2}
\\
&\leq\rho_{\alpha,\beta}V_{P_{\alpha,\beta}}(\xi_{k})
+\alpha^{2}\tilde{P}_{\alpha,\beta}(1,1)\sigma^{2}
+\frac{L\alpha^{2}}{2}\sigma^{2}
\end{align*}
It follows that
\begin{equation*}
(\mathcal{P}_{\alpha,\beta}V_{P_{\alpha,\beta}})(\xi)
\leq\rho_{\alpha,\beta}V_{P_{\alpha,\beta}}(\xi)
+\left(\frac{L}{2}+\tilde{P}_{\alpha,\beta}(1,1)\right)\alpha^{2}\sigma^{2}.
\end{equation*}
\end{proof}

In the special case $(\alpha,\beta)=(\alpha_{AG},\beta_{AG})$, 
we obtain the following result.

\begin{lemma}\label{lem:drift:AG}
Given $(\alpha,\beta)=(\alpha_{AG},\beta_{AG})$.
\begin{equation*}
(\mathcal{P}_{\alpha,\beta}V_{P_{AG}})(\xi)
\leq\gamma V_{P_{AG}}(\xi)+K,
\end{equation*}
where
\begin{equation*}
\gamma:=\rho_{AG},
\qquad
K:=\frac{\sigma^{2}}{L},
\end{equation*}
where $\rho_{AG}=1-1/\sqrt{\kappa}$.
\end{lemma}

\begin{proof}
By letting $(\alpha,\beta)=(\alpha_{AG},\beta_{AG})$
in Lemma \ref{lem:drift}, we get
\begin{equation*}
(\mathcal{P}_{\alpha,\beta}V_{P_{AG}})(\xi)
\leq\gamma V_{P_{AG}}(\xi)+K,
\end{equation*}
where
\begin{equation*}
\gamma=\rho_{AG},
\qquad
K=\left(\frac{L}{2}+\tilde{P}_{AG}(1,1)\right)\alpha_{AG}^{2}\sigma^{2},
\end{equation*}
where $\rho_{AG}=1-1/\sqrt{\kappa}$ and $\tilde{P}_{AG}(1,1)$
is the $(1,1)$-entry of $\tilde{P}_{AG}$. Notice that
\begin{equation*}
\tilde{P}_{AG}=
\left(\begin{array}{c}
\sqrt{\frac{L}{2}}
\\
\sqrt{\frac{\mu}{2}}-\sqrt{\frac{L}{2}}
\end{array}
\right)
\left(\begin{array}{cc}
\sqrt{\frac{L}{2}} & \sqrt{\frac{\mu}{2}}-\sqrt{\frac{L}{2}}
\end{array}
\right),
\end{equation*}
and hence
\begin{equation*}
P_{AG}=\tilde{P}_{AG}\otimes I_{d}
=
\left(
\begin{array}{cc}
\frac{L}{2}I_{d} & \left(\frac{\sqrt{L\mu}}{2}-\frac{L}{2}\right)I_{d}
\\
\left(\frac{\sqrt{\mu L}}{2}-\frac{L}{2}\right)I_{d} & \frac{(\sqrt{\mu}-\sqrt{L})^{2}}{2}I_{d}
\end{array}
\right),
\end{equation*}
which implies that $\tilde{P}_{AG}(1,1)=\frac{L}{2}$.
\end{proof}

Next, let us verify the minorization condition.
Assume that the noise admits a continuous probability density function,
then the Markov transition kernel $\mathcal{P}_{\alpha,\beta}$ 
also admits a continuous probability density function
for $x_{k+1}$ conditional on $x_{k}$ and $x_{k-1}$, 
which we denote by $p(\xi,x)$, that is, 
$\mathbb{P}(x_{k+1}\in dx|(x_{k}^T,x_{k-1}^T)=\xi^T)=p(\xi,x)dx$. 
Also note that when we transit from $(x_{k}^{T},x_{k-1}^{T})^{T}$
to $(x_{k+1},x_{k})$, the value of $x_{k}$ follows
a Dirac delta distribution.
We aim to show that for any Borel measurable sets $A,B$
\begin{equation*}
\inf_{(x_{k},x_{k-1})\in\mathbb{R}^{2d}:V_{P}((x_{k},x_{k-1}))\leq R}
\mathcal{P}((x_{k},x_{k-1}),(x_{k+1},x_{k})\in A\times B)
\geq\eta\nu_{2}(A\times B),
\end{equation*}
for some probability measure $\nu_{2}$.
Let us define:
\begin{equation*}
B_{R}:=\left\{x\in\mathbb{R}^{d}:\exists\, y\in\mathbb{R}^{d}, V_{P}(x,y)\leq R\right\}.
\end{equation*}
We define $\nu_{2}$ such that $\nu_{2}(A\times B)=0$
for any $B$ that does not contain $B_{R}$,
and $\nu_{2}(A\times B)=\nu_{1}(A)$ for some probability
measure $\nu_{1}$ and for any $B$ that contains $B_{R}$.Then, it suffices
for us to show that
\begin{equation*}
\inf_{\xi\in\mathbb{R}^{2d},V_{P}(\xi)\leq R}p(\xi,x)
\geq\eta\nu(x),
\end{equation*}
where $\nu(x)$ is the probability density function
for some probability measure $\nu_{1}(\cdot)$.

\begin{lemma}\label{lem:minor}
For any $\eta\in(0,1)$, there exists some $R>0$ such that
\begin{equation*}
\inf_{\xi\in\mathbb{R}^{2d},V_{P}(\xi)\leq R}p(\xi,x)
\geq\eta\nu(x).
\end{equation*}
\end{lemma}

\begin{proof}
Let us take:
\begin{equation*}
\nu(x)=p(\xi_{\ast},x)\cdot\frac{1_{\Vert x-x_{\ast}\Vert\leq M}}{\int_{\Vert x-x_{\ast}\Vert\leq M}p(\xi_{\ast},x)dx},
\end{equation*}
where $M>0$ is sufficiently large so that the denominator
in the above equation is positive.When $\Vert x-x_{\ast}\Vert>M$, 
$\inf_{\xi\in\mathbb{R}^{2d},V_{P}(\xi)\leq R}p(\xi,x)\geq 0$
automatically holds.
Thus, we only need to focus on $\Vert x-x_{\ast}\Vert\leq M$.

Note that for sufficiently large $M$, 
$\int_{\Vert x-x_{\ast}\Vert\leq M}p(\xi_{\ast},x)dx$ can get arbitrarily
close to $1$.
Fix $M$, by the continuity of $p(\xi,x)$ in both $\xi$ and $x$, 
we can find $\eta'\in(0,1)$ such that uniformly in $\Vert x-x_{\ast}\Vert\leq M$,
\begin{equation*}
\inf_{\xi\in\mathbb{R}^{2d},V_{P}(\xi)\leq R}p(\xi,x)
\geq\eta'p(\xi_{\ast},x)=\eta\nu(x),
\end{equation*}
where we can take
\begin{equation*}
\eta:=\eta'\int_{\Vert x-x_{\ast}\Vert\leq M}p(\xi_{\ast},x)dx,
\end{equation*}
which can be arbitrarily close to $1$ if we take $R>0$ to be
sufficiently small.
In particular, if we fix  $\eta\in(0,1)$, then we can take $M>0$ such that
\begin{equation*}
\int_{\Vert x-x_{\ast}\Vert\leq M}p(\xi_{\ast},x)dx\geq\sqrt{\eta},
\end{equation*}
and similarly with fixed $\eta$ and $M$, we take $R>0$ such that
uniformly in $\Vert x-x_{\ast}\Vert\leq M$,
\begin{equation*}
\inf_{\xi\in\mathbb{R}^{2d},V_{P}(\xi)\leq R}p(\xi,x)
\geq\sqrt{\eta}p(\xi_{\ast},x).
\end{equation*}
\end{proof}

Finally, we are ready to state the proof of Theorem \ref{thm:general:alpha:beta}
and Proposition \ref{prop:general}.

\begin{proof}[Proof of Theorem \ref{thm:general:alpha:beta}]
According to the proof of Lemma \ref{lem:minor},
for any fixed $\eta>0$, we can define:
\begin{equation*}
M\geq\inf\left\{m>0:\int_{\Vert x-x_{\ast}\Vert\leq m}p(\xi_{\ast},x)dx=\sqrt{\eta}\right\},
\end{equation*}
and
\begin{equation*}
R\leq\sup\left\{r>0:\inf_{\xi\in\mathbb{R}^{2d},V_{P_{\alpha,\beta}}(\xi)\leq R}p(\xi,x)
\geq\sqrt{\eta}p(\xi_{\ast},x)\,\,
\text{for every $\Vert x-x_{\ast}\Vert\leq M$}\right\}.
\end{equation*}
Then, we have
\begin{equation*}
\inf_{\xi\in\mathbb{R}^{2d},V_{P_{\alpha,\beta}}(\xi)\leq R}p(\xi,x)
\geq\eta\nu(x).
\end{equation*}
Let us recall that 
\begin{equation*}
(\mathcal{P}_{\alpha,\beta}V_{P_{\alpha,\beta}})(\xi)
\leq\gamma_{\alpha,\beta} V_{P_{\alpha,\beta}}(\xi)+K_{\alpha,\beta}.
\end{equation*}
By Lemma \ref{lem:hairer} and Lemma \ref{lem:inv}, 
\begin{equation*}
d_{\psi}(\nu_{k,\alpha,\beta},\pi_{\alpha,\beta})
\leq\bar{\eta}^{k}d_{\psi}(\nu_{0,\alpha,\beta},\pi_{\alpha,\beta})
\end{equation*}
where $\bar{\eta}=(1-(\eta-\eta_{0}))\vee(2+R\psi\gamma_{0})/(2+R\psi)$
and $\psi=\eta_{0}/K_{\alpha,\beta}$, 
where $\eta_{0}\in(0,\eta)$ and $\gamma_{0}\in(\gamma_{\alpha,\beta}+2K_{\alpha,\beta}/R,1)$.
In particular, we can choose
\begin{equation*}
\eta_{0}=\frac{\eta}{2},
\qquad
\gamma_{0}=\frac{1}{2}\gamma_{\alpha,\beta}+\frac{1}{2}
+\frac{K_{\alpha,\beta}}{R}.
\end{equation*}
Therefore,
\begin{equation*}
\bar{\eta}=\max\left\{1-\frac{\eta}{2},
1-\left(\frac{1}{2}-\frac{1}{2}\gamma_{\alpha,\beta}-\frac{K_{\alpha,\beta}}{R}\right)
\frac{R\psi}{2+R\psi}\right\},
\end{equation*}
where $\psi:=\frac{\eta}{2K_{\alpha,\beta}}$ so that
\begin{equation*}
\bar{\eta}=\max\left\{1-\frac{\eta}{2},
1-\left(\frac{1}{2}-\frac{1}{2}\gamma_{\alpha,\beta}-\frac{K_{\alpha,\beta}}{R}\right)
\frac{R\eta}{4K_{\alpha,\beta}+R\eta}\right\}.
\end{equation*}
The proof is complete.
\end{proof}

\begin{proof}[Proof of Proposition \ref{prop:general}]
Let us recall that $\gamma=\rho =1-\frac{1}{\sqrt{\kappa}}$
and $K=\frac{\sigma^{2}}{L}$.
Recall that $\gamma_{0}$ satisfies $\gamma_{0}\in(\gamma+2K/R,1)$
and let us assume that $K$ is sufficiently small so that
$K\leq\frac{R}{4\sqrt{\kappa}}$, then we can take
\begin{equation*}
\gamma_{0}=1-\frac{1}{4\sqrt{\kappa}}.
\end{equation*}
We also recall that $\psi=\eta_{0}/K$ and
\begin{equation*}
\bar{\eta}=\max\left\{1-\eta+\eta_{0},\frac{2+R\psi\gamma_{0}}{2+R\psi}\right\}
=\max\left\{1-\eta+\eta_{0},\frac{K+R\eta_{0}\gamma_{0}}{K+R\eta_{0}}\right\}.
\end{equation*}
We have discussed before that 
we can take $\eta$ to be arbitrarily close to $1$
by taking $M$ sufficiently large, and for fixed $M$
take $R$ sufficiently small.
Let us take
\begin{equation*}
\eta=1-\rho =\frac{1}{\sqrt{\kappa}},
\qquad
\eta_{0}=\frac{1}{2}\eta=\frac{1}{2\sqrt{\kappa}},
\end{equation*}
and then
\begin{equation*}
1-\eta+\eta_{0}=1-\frac{1}{2\sqrt{\kappa}}.
\end{equation*}
If we take $K<R\eta_{0}=\frac{R}{2\sqrt{\kappa}}$,
then
\begin{equation*}
\frac{K+R\eta_{0}\gamma_{0}}{K+R\eta_{0}}
\leq 1-\frac{1}{8\sqrt{\kappa}}.
\end{equation*}
Hence, we can take $K\leq\frac{R}{4\sqrt{\kappa}}$,
that is, 
\begin{equation*}
\sigma^{2}\leq\frac{RL}{4\sqrt{\kappa}},
\end{equation*}
so that 
\begin{equation*}
\bar{\eta}\leq 1-\frac{1}{8\sqrt{\kappa}}.
\end{equation*}

Finally, we want to take $R>0$ and $M>0$ such that
\begin{equation*}
\inf_{\xi\in\mathbb{R}^{2d},V_{P_{AG}}(\xi)\leq R}p(\xi,x)
\geq\eta\nu(x)=\frac{\nu(x)}{\sqrt{\kappa}}
\end{equation*}
holds for the choice of 
\begin{equation*}
\nu(x)=p(\xi_{\ast},x)\cdot\frac{1_{\Vert x-x_{\ast}\Vert\leq M}}{\int_{\Vert x-x_{\ast}\Vert\leq M}p(\xi_{\ast},x)dx}.
\end{equation*}
It is easy to see that we can take
$M$ so that
\begin{equation*}
\int_{\Vert x-x_{\ast}\Vert\leq M}p(\xi_{\ast},x)dx
\geq\frac{1}{\kappa^{1/4}},
\end{equation*}
and take $R$ such that for any $\Vert x-x_{\ast}\Vert\leq M$,
\begin{equation*}
\inf_{\xi\in\mathbb{R}^{2d},V_{P_{AG}}(\xi)\leq R}p(\xi,x)
\geq\frac{1}{\kappa^{1/4}}p(\xi_{\ast},x).
\end{equation*}

Hence, by applying Lemma \ref{lem:hairer}, we conclude that
for any two probability measures $\mu_{1},\mu_{2}$ on $\mathbb{R}^{2d}$:
\begin{equation*}
d_{\psi}(\mathcal{P}_{\alpha,\beta}^{k}\mu_{1},\mathcal{P}_{\alpha,\beta}^{k}\mu_{2})
\leq\left(1-\frac{1}{8\sqrt{\kappa}}\right)^{k}d_{\psi}(\mu_{1},\mu_{2}).
\end{equation*}
Recall that $\nu_{k,\alpha,\beta}$ denotes
the law of the iterates $\xi_{k}$.
By Lemma \ref{lem:inv}, the Markov chain $\xi_{k}$
admits a unique invariant distribution $\pi_{\alpha,\beta}$.
By letting $\mu_{1}=\nu_{0,\alpha,\beta}$
and $\mu_{2}=\pi_{\alpha,\beta}$, we conclude that
\begin{equation*}
d_{\psi}(\nu_{k,\alpha,\beta},\pi_{\alpha,\beta})
\leq\left(1-\frac{1}{8\sqrt{\kappa}}\right)^{k}
d_{\psi}(\nu_{0,\alpha,\beta},\pi_{\alpha,\beta}),
\end{equation*}
where
\begin{equation*}
\psi=\frac{\eta_{0}}{K}
=\frac{1}{2\sqrt{\kappa}K}
=\frac{L}{2\sqrt{\kappa}\sigma^{2}}.
\end{equation*}

Finally, let us prove \eqref{eqn:difference:AG}.
Given $(\alpha,\beta)=(\alpha_{AG},\beta_{AG})$, 
we have $\rho_{\alpha,\beta}=1-\frac{1}{\sqrt{\kappa}}$,
$\alpha=\frac{1}{L}$. 
It follows from Lemma \ref{lem:drift:AG} and its proof that
\begin{equation*}
\mathbb{E}[V_{P_{AG}}(\xi_{k+1})]\leq\rho_{AG}\mathbb{E}[V_{P_{AG}}(\xi_{k})]+
\frac{1}{L}\sqrt{\kappa}\sigma^{2}.
\end{equation*}
By induction on $k$, we can show that for every $k$, 
\begin{equation*}
\mathbb{E}[V_{P_{AG}}(\xi_{k+1})]\leq
V_{P_{AG}}(\xi_{0})\rho_{AG}^{k+1}
+\frac{1}{L}\sqrt{\kappa}\sigma^{2}.
\end{equation*}
By the definition of $V_{P}$, it follows that
\begin{equation*}
\mathbb{E}[f(x_{k+1})]-f(x_{\ast})\leq
V_{P_{AG}}(\xi_{0})\rho_{AG}^{k+1}
+\frac{1}{L}\sqrt{\kappa}\sigma^{2}
=
V_{P_{AG}}(\xi_{0})\rho_{AG}^{k+1}+\frac{1}{L}\sqrt{\kappa}\sigma^{2}.
\end{equation*}
Thus, we get
\begin{equation*}
\mathbb{E}[f(x_{k})]-f(x_{\ast})
\leq
V_{P_{AG}}(\xi_{0})\left(1-\frac{1}{\sqrt{\kappa}}\right)^{k}
+\frac{1}{L}\sqrt{\kappa}\sigma^{2}.
\end{equation*}
The proof is complete.
\end{proof}

\begin{remark}
In Proposition \ref{prop:general}, the amount of noise that can be tolerated is limited. Nevertheless, in applications where the gradient is estimated from noisy measurements, such results would be applicable if the noise level is mild \cite{birand2013measurements}.
\end{remark}
\begin{proof}[Proof of Corollary \ref{cor:Gaussian}]
If the noise $\varepsilon_{k}$ are i.i.d. Gaussian $\mathcal{N}(0,\Sigma)$,
then conditional
on $x_{k}=x_{k-1}=x_{\ast}$ in the AG method, with stepsize $\alpha=1/L$,
$x_{k+1}$ is distributed as $\mathcal{N}(x_{\ast},L^{-2}\Sigma)$
with $\Sigma\preceq L^{2}I_{d}$.
Therefore, for $\gamma>0$ sufficiently small,
\begin{equation*}
\mathbb{E}\left[e^{\gamma\Vert x_{k+1}-x_{\ast}\Vert^{2}}\Big|x_{k}=x_{k-1}=x_{\ast}\right]
=\frac{1}{\sqrt{\det{(I_{d}-2\gamma L^{-2}\Sigma)}}}.
\end{equation*}
By Chebychev's inequality, letting $\gamma=1/2$, 
for any $m\geq 0$, we get
\begin{equation*}
\mathbb{P}\left(\Vert x_{k+1}-x_{\ast}\Vert\geq m|x_{k}=x_{k-1}=x_{\ast}\right)
\leq
\frac{e^{-\frac{1}{2}m^{2}}}{\sqrt{\det(I_{d}-L^{-2}\Sigma)}}.
\end{equation*}
Hence, we can take
\begin{equation*}
M=\left(-2\log\left(\left(1-\frac{1}{\kappa^{1/4}}\right)\sqrt{\det(I_{d}-L^{-2}\Sigma)}\right)\right)^{1/2}.
\end{equation*}
Conditional on $(x_{k}^{T},x_{k-1}^{T})^{T}=\xi=(\xi_{(1)}^{T},\xi_{(2)}^{T})^{T}$,
where $V_{P}(\xi)\leq r$ for some $r>0$,
then, $x_{k+1}$ is Gaussian distributed:
\begin{equation*}
x_{k+1}|(x_{k},x_{k-1})=(\xi_{(1)},\xi_{(2)})
\sim\mathcal{N}\left(\mu_{\xi},L^{-2}\Sigma\right),
\end{equation*}
where
\begin{equation}\label{eqn:mu:xi}
\mu_{\xi}=\frac{2\sqrt{\kappa}}{\sqrt{\kappa}+1}\xi_{(1)}
-\frac{\sqrt{\kappa}-1}{\sqrt{\kappa}+1}\xi_{(2)}
-L^{-1}\nabla f\left(\frac{2\sqrt{\kappa}}{\sqrt{\kappa}+1}\xi_{(1)}
-\frac{\sqrt{\kappa}-1}{\sqrt{\kappa}+1}\xi_{(2)}\right).
\end{equation}
Thus, uniformly in $\Vert x-x_{\ast}\Vert\leq M$,
\begin{equation*}
\frac{p(\xi,x)}{p(\xi_{\ast},x)}
=e^{-\frac{1}{2}(x-\mu_{\xi})^{T}L^{2}\Sigma^{-1}(x-\mu_{\xi})
+\frac{1}{2}(x-x_{\ast})^{T}L^{2}\Sigma^{-1}(x-x_{\ast})}.
\end{equation*}
Note that $V_{P_{AG}}(\xi)\leq r$ implies that
\begin{equation*}
\left(
\begin{array}{cc}
\xi_{(1)}-x_{\ast}
\\
\xi_{(2)}-x_{\ast}
\end{array}\right)^{T}
P_{AG}
\left(
\begin{array}{cc}
\xi_{(1)}-x_{\ast}
\\
\xi_{(2)}-x_{\ast}
\end{array}\right)\leq r.
\end{equation*}
By the definition of $P_{AG}$, we get
\begin{equation*}
\left(
\begin{array}{cc}
\xi_{(1)}-x_{\ast}
\\
\xi_{(2)}-x_{\ast}
\end{array}\right)^{T}
\left(
\begin{array}{cc}
\sqrt{\frac{L}{2}}I_{d}
\\
\left(\sqrt{\frac{\mu}{2}}-\sqrt{\frac{L}{2}}\right)I_{d}
\end{array}\right)
\left(
\begin{array}{cc}
\sqrt{\frac{L}{2}}I_{d}
\\
\left(\sqrt{\frac{\mu}{2}}-\sqrt{\frac{L}{2}}\right)I_{d}
\end{array}\right)^{T}
\left(
\begin{array}{cc}
\xi_{(1)}-x_{\ast}
\\
\xi_{(2)}-x_{\ast}
\end{array}\right)\leq r,
\end{equation*}
so that
\begin{equation*}
\frac{L}{2}\Vert\xi_{(1)}-x_{\ast}\Vert^{2}
+\frac{(\sqrt{\mu}-\sqrt{L})^{2}}{2}\Vert\xi_{(2)}-x_{\ast}\Vert^{2}\leq r,
\end{equation*}
which implies that
\begin{equation*}
\Vert\xi_{(1)}-x_{\ast}\Vert
\leq\frac{\sqrt{2r}}{\sqrt{L}},
\qquad
\Vert\xi_{(2)}-x_{\ast}\Vert
\leq\frac{\sqrt{2r}}{\sqrt{L}-\sqrt{\mu}}.
\end{equation*}
Moreover,
\begin{align*}
\mu_{\xi}-x_{\ast}
&=\frac{2\sqrt{\kappa}}{\sqrt{\kappa}+1}\xi_{(1)}
-\frac{\sqrt{\kappa}-1}{\sqrt{\kappa}+1}\xi_{(2)}
-L^{-1}\nabla f\left(\frac{2\sqrt{\kappa}}{\sqrt{\kappa}+1}\xi_{(1)}
-\frac{\sqrt{\kappa}-1}{\sqrt{\kappa}+1}\xi_{(2)}\right)
\\
&\qquad
-\left(\frac{2\sqrt{\kappa}}{\sqrt{\kappa}+1}x_{\ast}
-\frac{\sqrt{\kappa}-1}{\sqrt{\kappa}+1}x_{\ast}
-L^{-1}\nabla f\left(\frac{2\sqrt{\kappa}}{\sqrt{\kappa}+1}x_{\ast}
-\frac{\sqrt{\kappa}-1}{\sqrt{\kappa}+1}x_{\ast}\right)\right)
\\
&=\frac{2\sqrt{\kappa}}{\sqrt{\kappa}+1}(\xi_{(1)}-x_{\ast})
-\frac{\sqrt{\kappa}-1}{\sqrt{\kappa}+1}(\xi_{(2)}-x_{\ast})
\\
&\qquad
-L^{-1}\left(\nabla f\left(\frac{2\sqrt{\kappa}}{\sqrt{\kappa}+1}\xi_{(1)}
-\frac{\sqrt{\kappa}-1}{\sqrt{\kappa}+1}\xi_{(2)}\right)
-\nabla f\left(\frac{2\sqrt{\kappa}}{\sqrt{\kappa}+1}x_{\ast}
-\frac{\sqrt{\kappa}-1}{\sqrt{\kappa}+1}x_{\ast}\right)\right).
\end{align*}
Since $\nabla f$ is $L$-Lipschitz,
\begin{align}
\Vert\mu_{\xi}-x_{\ast}\Vert
&\leq
(1+L^{-1}L)\frac{2\sqrt{\kappa}}{\sqrt{\kappa}+1}\Vert\xi_{(1)}-x_{\ast}\Vert
+(1+L^{-1}L)\frac{\sqrt{\kappa}-1}{\sqrt{\kappa}+1}\Vert\xi_{(2)}-x_{\ast}\Vert
\nonumber
\\
&\leq
2\frac{2\sqrt{\kappa}}{\sqrt{\kappa}+1}\frac{\sqrt{2r}}{\sqrt{L}}
+2\frac{\sqrt{\kappa}-1}{\sqrt{\kappa}+1}\frac{\sqrt{2r}}{\sqrt{L}-\sqrt{\mu}}
\nonumber
\\
&\leq
2\frac{2\sqrt{\kappa}}{\sqrt{\kappa}+1}\frac{\sqrt{2r}}{\sqrt{L}-\sqrt{\mu}}
+2\frac{\sqrt{\kappa}-1}{\sqrt{\kappa}+1}\frac{\sqrt{2r}}{\sqrt{L}-\sqrt{\mu}}
\nonumber
\\
&=2\frac{3\sqrt{\kappa}-1}{\sqrt{\kappa}+1}\frac{\sqrt{2r}}{\sqrt{L}-\sqrt{\mu}}.\label{ineq:r}
\end{align}
Thus, uniformly in $\Vert x-x_{\ast}\Vert\leq M$,
\begin{align*}
\frac{p(\xi,x)}{p(\xi_{\ast},x)}
&=\exp\left\{-\frac{1}{2}(x-\mu_{\xi})^{T}L^{2}\Sigma^{-1}(x-\mu_{\xi})
+\frac{1}{2}(x-x_{\ast})^{T}L^{2}\Sigma^{-1}(x-x_{\ast})\right\}
\\
&\geq
\exp\left\{-\frac{1}{2}\Vert\mu_{\xi}-x_{\ast}\Vert 
L^{2}\Vert\Sigma^{-1}\Vert
(\Vert x-\mu_{\xi}\Vert+\Vert x-x_{\ast}\Vert)\right\}
\\
&\geq
\exp\left\{-\frac{1}{2}\Vert\mu_{\xi}-x_{\ast}\Vert 
L^{2}\Vert\Sigma^{-1}\Vert
(\Vert\mu_{\xi}-x_{\ast}\Vert+2\Vert x-x_{\ast}\Vert)\right\}
\\
&\geq
\exp\left\{-\frac{1}{2} 
L^{2}\Vert\Sigma^{-1}\Vert
(\Vert\mu_{\xi}-x_{\ast}\Vert^{2}+2M\Vert\mu_{\xi}-x_{\ast}\Vert)\right\}
\geq\frac{1}{\kappa^{1/4}},
\end{align*}
if we have
\begin{equation}\label{ineq:r:2}
\Vert\mu_{\xi}-x_{\ast}\Vert
\leq
-M+\sqrt{M^{2}+\frac{\log(\kappa)}{2L^{2}\Vert\Sigma^{-1}\Vert}}.
\end{equation}
Combining \eqref{ineq:r} and \eqref{ineq:r:2}, 
we can take
\begin{align*}
R&=\frac{1}{8}\left(-M+\sqrt{M^{2}+\frac{\log(\kappa)}{2L^{2}\Vert\Sigma^{-1}\Vert}}\right)^{2}
\frac{(\sqrt{\kappa}+1)^{2}(\sqrt{L}-\sqrt{\mu})^{2}}{(3\sqrt{\kappa}-1)^{3}}
\\
&=\left(-M+\sqrt{M^{2}+\frac{\log(L/\mu)}{2L^{2}\Vert\Sigma^{-1}\Vert}}\right)^{2}
\frac{(L-\mu)^{2}}{8(3\sqrt{L}-\sqrt{\mu})^{3}}.
\end{align*}
For the remaining of the proof, without loss of generality assume that $\mu=\Theta(1)$ and $L = \Theta(\kappa)$.\footnote{Given two scalar-valued functions $f$ and $g$, we say $f=\Theta(g)$, if the ratio $f(x)/g(x)$ lies in an interval $[c_1,c_2]$  for every $x$ and some$c_1,c_2>0$.} It is straightforward to see from the Taylor expansion of $M$ that $M=O(\kappa^{-1/8})$ and
\begin{align*}
R&=\frac{\left(\frac{\log(L/\mu)}{2L^{2}\Vert\Sigma^{-1}\Vert}\right)^{2}}
{\left(M+\sqrt{M^{2}+\frac{\log(L/\mu)}{2L^{2}\Vert\Sigma^{-1}\Vert}}\right)^{2}}
\frac{(L-\mu)^{2}}{8(3\sqrt{L}-\sqrt{\mu})^{3}}
\\
&=O\left(\frac{1}{M^{2}}\left(\frac{\log(L/\mu)}{2L^{2}\Vert\Sigma^{-1}\Vert}\right)^{2}\frac{(L-\mu)^{2}}{8(3\sqrt{L}-\sqrt{\mu})^{3}}\right)
\\
&=O\left(\kappa^{-13/4}\log^{2}(\kappa)\right).
\end{align*}
\end{proof}

\subsection{Proofs of Results in Section \ref{sec:projected:AG}} 
Consider the constrained 
optimization problem 
$$ \min_{x\in\mathcal{C}}  f(x), $$
where $\mathcal{C}\subset \mathbb{R}^d$ is compact. 
The projected AG method consists of the iterations
\begin{align}
&\tilde{x}_{k+1}=\prox\left(\tilde{y}_{k}-\alpha(\nabla f(\tilde{y}_{k}) +\varepsilon_{k+1})\right),
\\
&\tilde{y}_{k}=(1+\beta)\tilde{x}_{k}-\beta \tilde{x}_{k-1},
\end{align}
where $\varepsilon_k$ is the random gradient error satisfying Assumption \ref{assump:noise}, $\alpha,\beta>0$ are the stepsize and momentum parameter and the projection onto the convex compact set $C$ with diameter $\mathcal{D}_{\mathcal{C}}$ can be written as
$$ 
\prox (x) := \arg \min_{y\in\R^d}\left( \frac{1}{2\alpha}\|x-y\|^2 + h(y)\right)
$$
where the function $h:\R^d\to \R \cup \{+\infty\}$ is the indicator function, defined to be zero if $y\in \mathcal{C}$ and infinity otherwise.
Let us recall that we assumed that the random gradient error $\varepsilon_{k}$
admits a continuous density so that
conditional on $\tilde{\xi}_{k}=(\tilde{x}_{k}^T,\tilde{x}_{k-1}^T)^T$, 
$\tilde{x}_{k+1}$
also admits a continuous density, i.e.
\begin{equation*}
\mathbb{P}(\tilde{x}_{k+1}\in d\tilde{x}|\tilde{\xi}_{k}=\tilde{\xi})
=\tilde{p}(\tilde{\xi},\tilde{x})d\tilde{x},
\end{equation*}
where $\tilde{p}(\tilde{\xi},\tilde{x})>0$ is continuous in both $\tilde{\xi}$ and $\tilde{x}$.

For the function $f(x)$, the gradient mapping $g:\R^d \to \R$ which replaces the gradient for constrained optimization problems is defined as
$$g(y) = \frac{1}{\alpha}\left(y - \prox(y- \alpha \nabla f(y)\right), \quad \alpha>0.$$
Due to the noise in the gradients, we also define the perturbed gradient mapping, $g_\varepsilon(y):\R^d\to\R$ as 
$$g_\varepsilon(y) = \frac{1}{\alpha}\left( 
y - \prox\big(y- \alpha (\nabla f(y)+\varepsilon)\big)\right), \quad \alpha>0,\quad\varepsilon\in\R^{d}.
$$
Due to the non-expansiveness property of the projection operator, we have (see e.g. \cite[Lemma 2.4]{combettes2005signal})
\beq\label{ineq-proj-contracts} \Delta_\varepsilon(y) := g_\varepsilon(y) - g(y), \quad \| \Delta_\varepsilon(y) \|^2 \leq \|\varepsilon\|^2, \quad \text{for every $y\in\mathbb{R}^{d}$}. 
\eeq
Following a similar approach to \cite{hu2017dissipativity,fazlyab2017dynamical}, 
we reformulate the projected AG iterations as a linear dynamical system as 
\begin{eqnarray*}
\tx_{k+1}&=& (1+\beta)\tx_{k}-\beta \tx_{k-1} -\alpha g_{\epsilon_{k+1}}(\ty_{k})\,,\\
\ty_{k}&=&(1+\beta)\tx_{k}-\beta \tx_{k-1}\,,
\end{eqnarray*}
which is equivalent to 
\begin{align}\label{eq-proj-ag-iters-1}
&\txi_{k+1} = A \txi_k + B \tu_k, 
\\
&\ty_k = C\xi_k, \quad \tx_k = E \txi_k,
\\
&{\tu}_k = g(\ty_k) +  \Delta_{\varepsilon_{k+1}}(\ty_k),   \label{eq-proj-ag-iters-3} 
\end{align}
with $\txi_{k} = [\tx_{k}^T ~~ \tx_{k-1}^T]^T$, and
 \begin{align} \label{def-system-mat-proj-AG}
      &A = \begin{pmatrix} 
        (1+\beta)I_d & -\beta I_d \\
        I_d & 0_d    
          \end{pmatrix}, 
          \quad
      B = \begin{pmatrix} -\alpha I_d \\ 0_d
           \end{pmatrix},
      \\
      &C = \begin{pmatrix} (1+\beta) I_d & -\beta I_d
      \end{pmatrix},
      \quad
      E = \begin{pmatrix} I_d & 0_d
      \end{pmatrix}.
      \nonumber
   \end{align}
  
We see that ${\txi}_{k}$ forms a time-homogeneous Markov chain. To this chain, we can associate a Markov kernel $\tilde{\mathcal{P}}_{\alpha,\beta}$, following a similar approach to the Markov kernel $\mathcal{P}_{\alpha,\beta}$ we defined for AG. We have the following result.

\begin{lemma}\label{lem:key:projected}
\begin{equation*}
(\tilde{\mathcal{P}}_{\alpha,\beta} V_{P_{\alpha,\beta}})(\txi)
\leq\rho_{\alpha,\beta} V_{P_{\alpha,\beta}}(\txi)+\tK,
\end{equation*}
where
\begin{equation*}
\tK:=\alpha \sigma (2\diam  \| P_{\alpha,\beta}\| + G_M) + \alpha^2\sigma^2 \left(\|P_{\alpha,\beta}\| + \frac{L}{2}\right),
\end{equation*}
if there exists a matrix $P_{\alpha,\beta}\in\R^{2d\times 2d}$ such that 
\begin{equation}\label{ineq:key}
-\rho_{\alpha,\beta}X_{1}-(1-\rho_{\alpha,\beta})X_{2}
+X_{3}
\preceq 0,
\end{equation}
where
\begin{align*}
&X_{1}=\frac{1}{2}
\left(
\begin{array}{ccc}
\beta^{2}\mu I_{d} & -\beta^{2}\mu I_{d} & -\beta I_{d}
\\
-\beta^{2}\mu I_{d} & \beta^{2}\mu I_{d} & \beta I_{d}
\\
-\beta I_{d} & \beta I_{d} & \alpha(2-L\alpha) I_{d}
\end{array}
\right),
\\
&X_{2}=\frac{1}{2}
\left(
\begin{array}{ccc}
(1+\beta)^{2}\mu I_{d} & -\beta(1+\beta)\mu I_{d} & -(1+\beta)I_{d}
\\
-\beta(1+\beta)\mu I_{d} & \beta^{2}\mu I_{d} & \beta I_{d}
\\
-(1+\beta)I_{d} & \beta I_{d} & \alpha(2-L\alpha)I_{d}
\end{array}
\right),
\end{align*}
and
\begin{equation*}
X_{3}=\begin{pmatrix}
A^{T}P_{\alpha,\beta}A-\trabsq P_{\alpha,\beta} & A^{T}P_{\alpha,\beta}B
\\
B^{T}P_{\alpha,\beta}A & B^{T}P_{\alpha,\beta}B
\end{pmatrix},
\end{equation*}
where $G_M := \max_{x\in\mathcal{C}} \|\nabla f(x)\|$.

In particular, with $\rho=1-\frac{1}{\sqrt{\kappa}}$, $\beta=\frac{\sqrt{\kappa}-1}{\sqrt{\kappa}+1}$, $\alpha=\frac{1}{L}$ where $\kappa= \frac{L}{\mu}$. Then \eqref{ineq:key} holds with 
the matrix $$P=\frac{\mu}{2}\begin{pmatrix}  (1-\sqrt{\kappa})I_d & \sqrt{\kappa}I_d
\end{pmatrix}^T \begin{pmatrix} 
(1-\sqrt{\kappa})I_d & \sqrt{\kappa}I_d
\end{pmatrix}.$$ 
\end{lemma}

\begin{proof}
We follow the proof technique of \cite{fazlyab2017dynamical} for deterministic proximal AG which is based on \cite[Lemma 2.4]{nesterov2004introductory} and adapt this proof technique to accelerated stochastic projected gradient. Defining the error at step $k$
$$ 
\tilde{e}_k := [(\txi_k - \txi_*)^T ~~ (g(\ty_k) - g(\ty_*))^T]^T, 
$$ 
where $\txi_*:= [x_*^T ~ x_*^T]^T$ and $g(\ty_*) = 0$ due to the first order optimality conditions where $\ty_*:=\tx_*$ is the unique minimum of $f$ over $C$. 
Let $\mathcal{F}_k$ be the natural filtration for the iterations of the algorithm until and including step $k$ so that $x_k, y_k$ and $\tilde{e}_{k}$ are $\mathcal{F}_k$-measurable. 
Similar to the analysis of AG, we estimate 
\begin{align}
&\E\left[f\left(\tx_{k+1}\right) - f\left(\tx_k\right) \bigg| \calF_k\right]
\\
&= \E\left[f\left(\ty_k -\alpha g_{\varepsilon_{k+1}}\left(\ty_k\right)\right) 
- f\left(\tx_k\right) \bigg| \calF_k\right] \\
&= \E\left[f\left(\ty_k -\alpha g\left(\ty_k\right)-\alpha \Delta_{\varepsilon_{k+1}}\left(\ty_k\right)\right) - f\left(\tx_k\right)\bigg| \calF_k\right] \\
&\leq  \E\bigg[f\left(\ty_k -\alpha g\left(\ty_k\right)\right) 
+ \nabla f\left(\ty_k -\alpha g\left(\ty_k\right)\right)^T \alpha \Delta_{\varepsilon_{k+1}}\left(\ty_k\right) \\
&\qquad +\frac{\alpha^2 L}{2} \|\Delta_{\varepsilon_{k+1}}(\ty_k)\|^2 - f\left(\tx_k\right) \bigg| \calF_k\bigg] \\
& \leq  f\left(\ty_k -\alpha g\left(\ty_k\right)\right)- f\left(\tx_k\right) 
+ \E\left[ \alpha G_M \|\Delta_ {\varepsilon_{k+1}}(\ty_k)\|+  \frac{\alpha^2 L}{2}\|\varepsilon_{k+1}\|^2 \bigg| \calF_k\right] \\
& \leq  f\left(\ty_k -\alpha g\left(\ty_k\right)\right)- f\left(\tx_k\right)  + \alpha G_M \sigma + \frac{\alpha^2 L}{2} \sigma^2 , \label{ineq-exp-opt-diff}
\end{align}
where in the first inequality we used the fact that the gradient of $f$ is $L$-smooth which implies that 
    $$f(y) - f(z) \leq\nabla f(z)^T (y-z) +\frac{L}{2}\| y- z\|^2, \quad  \text{for every $y,z\in\mathbb{R}^{d}$}$$
(see e.g. \cite{Bubeck2014}) and second inequality follows from Jensen's inequality.Finally, the last step is a consequence of \eqref{ineq-proj-contracts} and Assumption \ref{assump:noise} on the noise. It follows from a similar computation that 
\beq 
\E\left[ f(\tx_{k+1}) - f(\tx_*) \bigg| \calF_k\right] \leq f\big(\ty_k -\alpha g(\ty_k)\big)- f(\tx_*)  + \alpha G_M \sigma + \frac{\alpha^2 L}{2} \sigma^2. \label{ineq-exp-opt}
\eeq
We note that the matrices $X_1$ and $X_2$ can be written as 
\begin{align} 
&X_1 = \frac{-1}{2}\begin{pmatrix} -\mu(C-E)^T(C-E) & (C-E)^T \\
C- E           & (L\alpha^2 - 2\alpha)I_d 
\end{pmatrix},
\\
&X_2 = \frac{-1}{2}\begin{pmatrix} -\mu C^T C & C^T \\
   C        & (L\alpha^2 - 2\alpha)I_d
\end{pmatrix}, 
\end{align}
where $A, B, C, E$ are defined by \eqref{def-system-mat-proj-AG}.
Using \cite[eqn. (36)--(37)]{fazlyab2017dynamical}
and Lemma \ref{lem:two:ineq},
we have
\beq 
    f\big(\ty_k -\alpha g(\ty_k)\big)- f(\tx_k) \leq -\tilde{e}_k^T X_1 \tilde{e}_k ,\label{ineq: upperbound } \\
    f\big(\ty_k -\alpha g(\ty_k)\big)- f(\tx_*) \leq -\tilde{e}_k^T X_2 \tilde{e}_k .\label{ineq: upperbound2 }
\eeq
Plugging these into \eqref{ineq-exp-opt-diff} and \eqref{ineq-exp-opt}, we obtain 
\beq 
    \E\left[ f(\tx_{k+1}) - f(\tx_k) \bigg| \calF_k\right] &\leq& -\tilde{e}_k^T X_1 \tilde{e}_k + \alpha G_M \sigma + \frac{\alpha^2 L}{2} \sigma^2 ,\label{ineq-lmi-zero}\\ 
    \E\left[ f(\tx_{k+1}) - f(\tx_*) \bigg| \calF_k\right] &\leq& -\tilde{e}_k^T X_2 \tilde{e}_k + \alpha G_M \sigma + \frac{\sigma^2 L}{2}\sigma^2 .\label{ineq-lmi-one}
\eeq  
It also follows from \eqref{eq-proj-ag-iters-1}-- \eqref{eq-proj-ag-iters-3} and the facts that $A\txi_* = \txi_*$ and $B\tu_* = 0$ that
  \beq  
  \txi_{k+1}-\txi_* =  A \left(\txi_k-\txi_*\right) + B \left(\tu_k - \tu_*\right) + B \Delta_{\varepsilon_{k+1}}(\ty_k) 
                          = \zeta_k  + B \Delta_{\varepsilon_{k+1}}(\ty_k),
  \eeq
where  $$\zeta_k := A \left(\txi_k-\txi_*\right) + B \left(\tu_k - \tu_*\right).$$

For any symmetric positive semi-definite matrix $P_{\alpha,\beta}\in\R^{2d\times 2d}$,
we define the quadratic function 
\begin{equation*}
Q_{P_{\alpha,\beta}}(\txi) = \txi^T P_{\alpha,\beta} \txi.
\end{equation*}
We can estimate that
\begin{align*}
&\E \left[ {Q}_{P_{\alpha,\beta}}\left(\txi_{k+1}\right)  \big| \calF_k \right] 
\\
&= \E \left[ \left(\txi_{k+1}-\txi_*\right)^T P_{\alpha,\beta} \left(\txi_{k+1} - \txi_*\right)  \big| \calF_k \right] \nonumber \\
&=   \zeta_k^T P_{\alpha,\beta} \zeta_k^T 
+\E \left[ 2(\txi_{k+1}-\txi_*)^T P_{\alpha,\beta} B \Delta_{\varepsilon_{k+1}}(\ty_k) +  
  B^T  \Delta_{\varepsilon_{k+1}}(\ty_k)^T P_{\alpha,\beta}  B \Delta_{\varepsilon_{k+1}}(\ty_k)   \big| \calF_k \right] \nonumber \\
&\leq  \tilde{e}_{k}^T \begin{pmatrix} 
A^T P_{\alpha,\beta} A  & A^T P_{\alpha,\beta} B \\
B^T P_{\alpha,\beta} A             &  B^T P_{\alpha,\beta} B
\end{pmatrix} \tilde{e}_{k}  +  \E \bigg[ 2\alpha \diam \cdot\|P_{\alpha,\beta}\| \cdot\| \varepsilon_{k+1}\| + \alpha^2 \|P_{\alpha,\beta}\| \cdot \|\varepsilon_{k+1}\|^2 | \calF_k\bigg] \nonumber \\
&= \tilde{e}_{k}^T \begin{pmatrix} 
A^T P_{\alpha,\beta} A  & A^T P_{\alpha,\beta} B \nonumber \\
B^T P_{\alpha,\beta} A             &  B^T P_{\alpha,\beta} B
\end{pmatrix} \tilde{e}_{k}  + 2\diam\alpha \sigma \| P_{\alpha,\beta}\| + \alpha^2 \sigma^2 \|P_{\alpha,\beta}\|.
\end{align*}
Therefore,  
\beq \E \left[ {Q}_{P_{\alpha,\beta}}\left(\txi_{k+1}\right) -  {Q}_{P_{\alpha,\beta}}\left(\txi_{k}\right)  \Big| \calF_k \right] = \tilde{e}_k^T X_3 \tilde{e}_k + 2\diam\alpha \sigma \| P_{\alpha,\beta}\|+ \alpha^2 \sigma^2 \|P_{\alpha,\beta}\|.
\eeq
Considering the Lyapunov function $V_{P_{\alpha,\beta}}(\txi_k) = f(\tx_k) - f(\tx_*) + \txi_k^T P_{\alpha,\beta} \txi_k$, we have 
\begin{align} 
&V_{P_{\alpha,\beta}}\left(\txi_{k+1}\right) - {\trabsq} V_{P_{\alpha,\beta}}\left(\txi_k\right) 
\\
&= {\trabsq}\left( f\left(\txi_{k+1}\right)-f\left(\txi_*\right)\right) 
+ (1-\trabsq) \left( f\left({\txi}_{k+1}\right) - f\left({\txi}_*\right)\right) 
\\
&\qquad\qquad\qquad
+  Q_{P_{\alpha,\beta}}\left(\txi_{k+1}-\txi_*\right) 
- Q_{P_{\alpha,\beta}}\left(\txi_k-\txi_*\right). 
\end{align}
Taking conditional expectations and inserting \eqref{ineq-lmi-zero}--\eqref{ineq-lmi-one}, 
\begin{align}
&\E \left[V_{P_{\alpha,\beta}}\left({\txi}_{k+1}\right) \Big| \calF_k \right] 
\\
&\leq \trabsq V_{P_{\alpha,\beta}}\left(\txi_k\right) + \tilde{e}_{k}^T  \bigg( -\trabsq X_1 -\left(1-\trabsq\right) X_2 + X_3\bigg) \tilde{e}_{k} \\
&\qquad\qquad
+ 2\mathcal{D}_{\mathcal{C}}\alpha \sigma \| P_{\alpha,\beta}\|+ \alpha^2\sigma^2 \left(\|P_{\alpha,\beta}\| + \frac{L}{2}\right) \\
&\leq \trabsq V_{P_{\alpha,\beta}}\left(\txi_k\right) +\alpha \sigma (2\diam \| P_{\alpha,\beta}\| + G_M) + \alpha^2\sigma^2 \left(\|P_{\alpha,\beta}\| + \frac{L}{2}\right),
\end{align}
which completes the proof.
\end{proof}


\begin{lemma}[\cite{fazlyab2017dynamical}]\label{lem:two:ineq}
Using the notations as in the proof of Lemma \ref{lem:key:projected}, we have
the following two inequalities:
\beq 
    f\big(\ty_k -\alpha g(\ty_k)\big)- f(\tx_k) \leq -\tilde{e}_k^T X_1 \tilde{e}_k , 
    \\
    f\big(\ty_k -\alpha g(\ty_k)\big)- f(\tx_*) \leq -\tilde{e}_k^T X_2 \tilde{e}_k .
\eeq
\end{lemma}

\begin{proof}
Recall that f satisfies following inequalities, 
\begin{eqnarray}
f(z)-f(y) \leq \nabla f(y)^T(z-y) +\frac{L}{2}\Vert y-z\Vert^2 ,\\ 
f(y)-f(x) \leq \nabla f(y)^T(y-x) - \frac{\mu}{2} \Vert y-x\Vert^2 .
\end{eqnarray}
Choosing $z=\tilde{y}_k-\alpha g(\tilde{y}_k), \; y=\tilde{y}_k$ and $x=\tilde{x}_k$ yields, 
\begin{equation}\label{ineq: upperbound3}
f(y_k-\alpha g(y_k))-f(x_{k})
\leq \nabla f(y_k)^T \big( y_k-x_k-\alpha g(y_k) \big)
+\frac{L}{2}\Vert \alpha g(y_k )\Vert^2- \frac{\mu}{2}\Vert y_k-x_k\Vert^2 .
\end{equation}
 Additionally let $\partial h(x) := \{ v \in \mathbb{R}^d : h(x)-h(y)\leq v^T(x-y) \forall y \in \mathbb{R}^d\}$ then by optimality condition, $0 \in \partial (\mathcal{P}_C (w))-\frac{1}{\alpha}(\mathcal{P}_C (w)-w)$ (e.g. \cite{FirstOrderMethods} theorem 6.39). In particular there exists a $T_{h}(w) \in \partial h(x)$ such that $g(w)=\nabla f(w)+T_{h}(w)$. Choose $w=y_k$ and note that $y_k=(1+\beta)x_k -\beta x_{k-1}$ and $C$ is a convex set thus $y_k \in C$. So if $T_h(y_k) \in \partial h(y_k)$  then either $0\leq T_h(y_{k})^T(y_k-x)$ or $-\infty \leq T_h(y_k)^T (y_k-x)$ therefore $0\leq T_h(y_k)^T(y_k-x)$ implying that $\nabla f(y)^T (y-z) \leq g(y)^T (y-x)$ for all $x\in \mathbb{R}^d$. Combining this result with \eqref{ineq: upperbound3} we obtain,
\begin{align*}
&f(y_k-\alpha g(y_k))-f(x_{k})
\\
&\leq  \nabla f(y_k)^T \big( y_k-x_k-\alpha g(y_k) \big)
+\frac{L}{2}\alpha^2 \Vert g(y_k)\Vert^2- \frac{\mu}{2}\beta^2\Vert x_k- x_{k-1}\Vert^2 
f(y_k-\alpha g(y_k))-f(x_{k}) 
\\
&\leq  \beta g(y_k)^T(x_k-x_{k-1}) +\left(\frac{L}{2}\alpha^2-\alpha\right) \Vert g(y_k)\Vert^2 
\\ 
&\qquad\qquad
- \frac{\mu}{2}\beta^2\left( \Vert x_{k}-x_*\Vert^2 - 2 (x_{k}-x_*)^T(x_{k-1}-x_*)+ \Vert x_{k-1}-x_*\Vert^2 \right).
\end{align*}
This proves \eqref{ineq: upperbound }. 
Finally, \eqref{ineq: upperbound2 } can also be obtained if we take $x=x_*$ and follow similar steps.
\end{proof} 
%

\begin{lemma}
Given $\alpha=\frac{1}{L}$, $\beta=\frac{\sqrt{\kappa}-1}{\sqrt{\kappa}+1}$, where
$\kappa=L/\mu$, we have
\begin{equation*}
(\tilde{\mathcal{P}}_{\alpha,\beta}V_{P_{\alpha,\beta}})(\tilde{\xi})
\leq\tilde{\gamma} V_{P_{\alpha,\beta}}(\tilde{\xi})+\tilde{K},
\end{equation*}
where
\begin{equation*}
\tilde{\gamma}:=1-\frac{1}{\sqrt{\kappa}},
\qquad
\tilde{K}:=\frac{\sigma}{L} \left(\diam  \mu((1-\sqrt{\kappa})^{2}+\kappa) + G_M\right) + \frac{\sigma^2}{L^{2}} \left(\frac{\mu}{2}((1-\sqrt{\kappa})^{2}+\kappa) + \frac{L}{2}\right).
\end{equation*}
\end{lemma}

\begin{proof}
Note that 
\begin{equation*}
(\tilde{\mathcal{P}}_{\alpha,\beta} V_{P_{\alpha,\beta}})(\txi)
\leq\trabsq V_{P_{\alpha,\beta}}(\txi)+\tK,
\end{equation*}
where
\begin{equation*}
\tK:=\alpha \sigma (2\diam  \| P_{\alpha,\beta}\| + G_M) + \alpha^2\sigma^2 \left(\|P_{\alpha,\beta}\| + \frac{L}{2}\right),
\end{equation*}
and with $\alpha=\frac{1}{L}$, $\beta=\frac{\sqrt{\kappa}-1}{\sqrt{\kappa}+1}$,
we have
\begin{equation*}
P_{\alpha,\beta}=\frac{\mu}{2}\begin{pmatrix}  (1-\sqrt{\kappa})I_d & \sqrt{\kappa}I_d
\end{pmatrix}^T \begin{pmatrix} 
(1-\sqrt{\kappa})I_d & \sqrt{\kappa}I_d
\end{pmatrix},
\end{equation*}
so that
\begin{equation*}
\Vert P_{\alpha,\beta}\Vert
\leq
\frac{\mu}{2}\left\Vert\begin{pmatrix}  (1-\sqrt{\kappa})I_d & \sqrt{\kappa}I_d
\end{pmatrix}^T\right\Vert 
\cdot
\left\Vert\begin{pmatrix} 
(1-\sqrt{\kappa})I_d & \sqrt{\kappa}I_d
\end{pmatrix}\right\Vert
=\frac{\mu}{2}((1-\sqrt{\kappa})^{2}+\kappa).
\end{equation*}
Hence,
\begin{equation*}
\tilde{K}_{\alpha,\beta}
\leq
\frac{\sigma}{L} \left(\diam  \mu((1-\sqrt{\kappa})^{2}+\kappa) + G_M\right) + \frac{\sigma^2}{L^{2}} \left(\frac{\mu}{2}((1-\sqrt{\kappa})^{2}+\kappa) + \frac{L}{2}\right).
\end{equation*}
\end{proof}

\begin{proof}[Proof of Theorem \ref{thm:general:alpha:beta:projected}]
The proof is similar to the proof of Theorem \ref{thm:general:alpha:beta}
and the proof of \eqref{eqn:difference:AG}. 
We obtain
\begin{equation*}
\mathbb{E}[f(\tilde{x}_{k})]-f(\tilde{x}_{\ast})
\leq
V_{P_{\alpha,\beta}}(\tilde{\xi}_{0})\tilde{\gamma}_{\alpha,\beta}^{k}+\frac{\tilde{K}_{\alpha,\beta}}{1-\tilde{\gamma}_{\alpha,\beta}}.
\end{equation*}
The conclusion then follows from the defintiion of $\tilde{\gamma}_{\alpha,\beta}$
and $\tilde{K}_{\alpha,\beta}$.
\end{proof}

\begin{proof}[Proof of Proposition \ref{prop:general:projected}]
The proof is similar as the proof of Proposition \ref{prop:general}. 
We can take $\tilde{K}\leq\frac{R}{4\sqrt{\kappa}}$, that is,
\begin{equation*}
\frac{\sigma}{L} \left(\diam  \mu((1-\sqrt{\kappa})^{2}+\kappa) + G_M\right) + \frac{\sigma^2}{L^{2}} \left(\frac{\mu}{2}((1-\sqrt{\kappa})^{2}+\kappa) + \frac{L}{2}\right)\leq\frac{R}{4\sqrt{\kappa}},
\end{equation*}
which implies
\begin{equation*}
\sigma\leq
\frac{-b_{1}}{2a_{1}}+\frac{1}{2a_{1}}\sqrt{b_{1}^{2}+a_{1}\frac{R}{\sqrt{\kappa}}},
\end{equation*}
where
\begin{equation*}
a_{1}=\frac{1}{L^{2}} \left(\frac{\mu}{2}((1-\sqrt{\kappa})^{2}+\kappa) + \frac{L}{2}\right),
\qquad
b_{1}=\frac{1}{L} \left(\diam  \mu((1-\sqrt{\kappa})^{2}+\kappa) + G_M\right).
\end{equation*}
As in the proof of Proposition \ref{prop:general}, 
we can take
\begin{equation*}
\tilde{\psi}=\frac{1}{2\sqrt{\kappa}\tilde{K}}.
\end{equation*}

Finally, the proof of \eqref{eqn:difference:AG:projected}
is similar as the proof of \eqref{eqn:difference:projected}.
We obtain
\begin{equation*}
\mathbb{E}[f(\tilde{x}_{k})]-f(\tilde{x}_{\ast})\leq
V_{P_{AG}}(\tilde{\xi}_{0})\tilde{\gamma}^{k}
+\frac{\tilde{K}}{1-\tilde{\gamma}}.
\end{equation*}
The conclusion then follows from the definition of $\tilde{K}$
and $\tilde{\gamma}$.
\end{proof}

\section{Numerical Illustrations}

In this section, we illustrate some of our theoretical results over some simple functions with numerical experiments. On the left panel of Figure \ref{fig: Comparison n=1}, we compare ASG for the quadratic objective $f(x) = x^2/2$ in dimension one with additive i.i.d. Gaussian noise on the gradients for different noise levels $\sigma\in \{0.01,0.1,1,2\}$. The plots show performance with respect to expected suboptimality using $10^4$ sample paths. As expected, the performance deteriorates when $\sigma$ increases. The fact that the performance stabilizes after a certain number of iterations supports the claim that a stationary distribution exists, a claim that was proved in Theorem \ref{thm:alpha:beta}. 
In the middle panel, we repeat the experiment in dimension $d=10$ over the quadratic objective $f(x) =\frac{1}{2} x^T Q x $, 
where $Q$ is a diagonal matrix with diagonal entries $Q_{ii}=1/i$. We observe similar patterns. 
\begin{figure}[h!]
    \centering
    \includegraphics[width=0.33\linewidth]{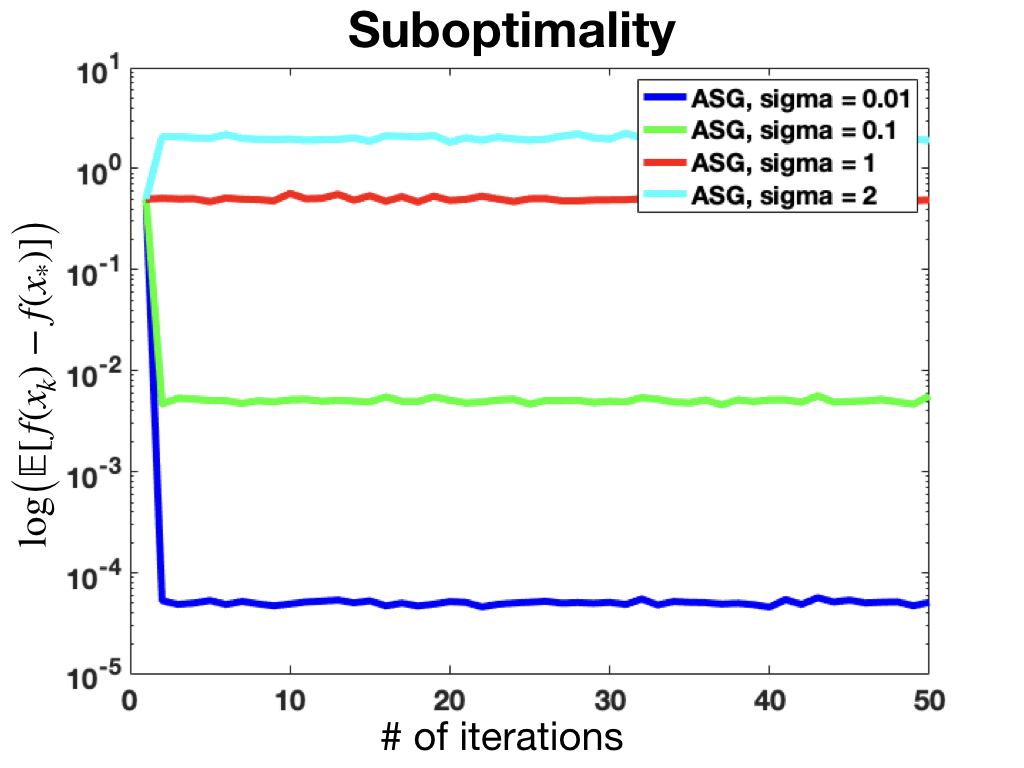}
    \includegraphics[width=0.33\linewidth]{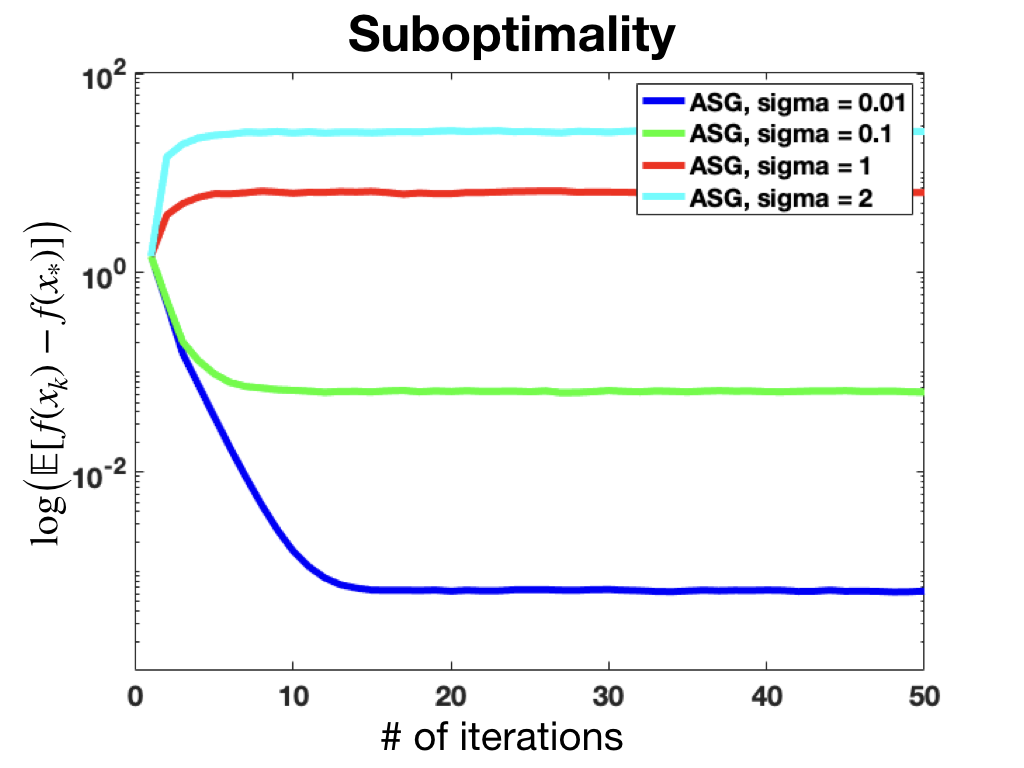}
   \includegraphics[width=0.33\linewidth]{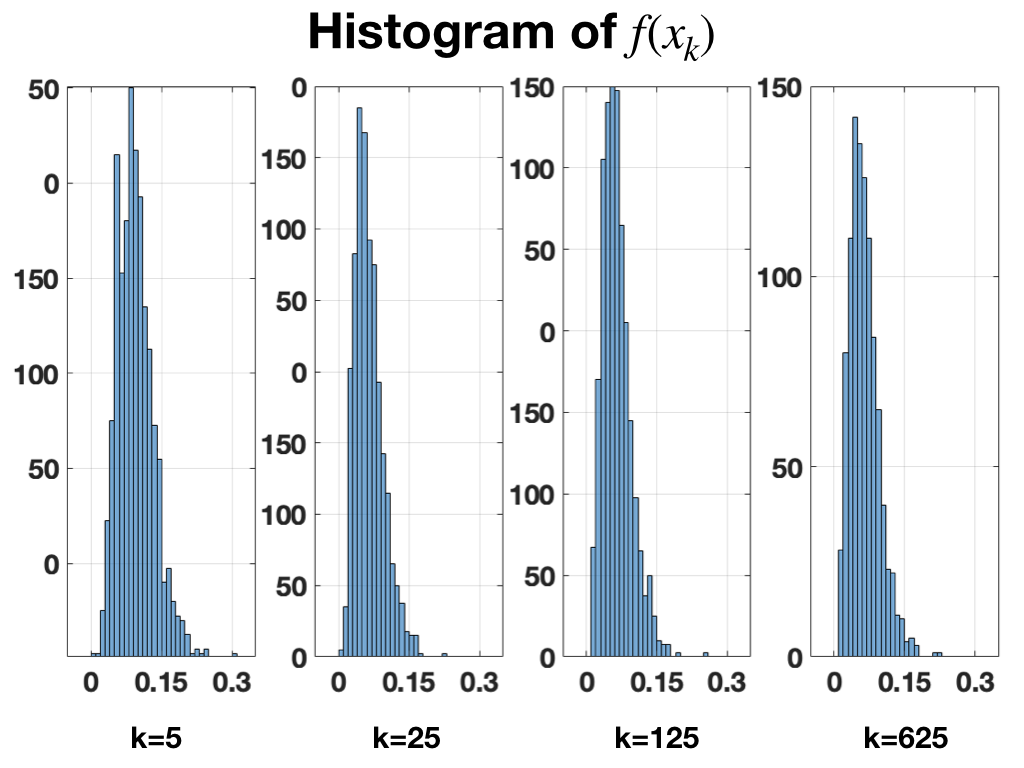}
    \caption{Performance comparison of ASG for different noise levels $\sigma$ on quadratic functions. \emph{Left panel}: $f(x)=\frac{1}{2}x^2$ in dimension one. \emph{Middle panel}: $f(x) = \frac{1}{2}x^T Qx $ in dimension $d=10$. \emph{Right panel}: Histogram of $f(x_k)$ for different values of $k$ where $f(x)=\frac{1}{2} x^TQx$ in dimension $d=10$.}
    \label{fig: Comparison n=1}
\end{figure}

Finally, on the right panel of Figure \ref{fig: Comparison n=1}, we estimate the distribution of $f(x_k)$ for $k \in \{5,25,125,625\}$. For this purpose, we plot the histograms of $f(x_k)$ over $10^4$ sample paths for every fixed $k$. We observe that the histograms for $k=125$ and $625$ are similar, illustrating the fact that ASG admits a stationary distribution.

\end{document}